\documentclass[a4paper,fleqn]{cas-sc}

\usepackage[authoryear,longnamesfirst]{natbib}

\usepackage[utf8]{inputenc}
\usepackage[T1]{fontenc}
\usepackage{amsmath,amssymb,mathtools}
\usepackage{enumitem}
\usepackage{algorithm}
\usepackage{algorithmic}
\usepackage{graphicx}
\usepackage{subcaption}
\graphicspath{{figs/}}
\usepackage{booktabs}
\usepackage{multirow}
\usepackage{bm}
\usepackage{placeins}

\newtheorem{theorem}{Theorem}[section]
\newtheorem{lemma}[theorem]{Lemma}
\newtheorem{proposition}[theorem]{Proposition}
\newtheorem{corollary}[theorem]{Corollary}
\newdefinition{definition}{Definition}[section]
\newdefinition{assumption}{Assumption}
\newdefinition{remark}{Remark}[theorem]
\newproof{proof}{Proof}

\newcommand{\R}{\mathbb{R}}
\newcommand{\N}{\mathbb{N}}
\newcommand{\calF}{\mathcal{F}}
\newcommand{\calL}{\mathcal{L}}
\newcommand{\ip}[2]{\langle #1,\, #2 \rangle}
\newcommand{\norm}[1]{\| #1 \|}
\newcommand{\dist}{\mathrm{dist}}
\newcommand{\argmin}{\operatorname{arg\,min}}
\newcommand{\argmax}{\operatorname{arg\,max}}
\newcommand{\lb}{\mathrm{lb}}

\hypersetup{hypertexnames=false}
\begin{document}
\let\WriteBookmarks\relax
\def\floatpagepagefraction{1}
\def\textpagefraction{.001}

\shorttitle{Two-Phase Adaptive Balanced Penalty for Controllable Pareto Front Learning}
\shortauthors{N.V. Hoang, D.D. Le and T.N. Thang}

\title[mode=title]{A Two-Phase Adaptive Balanced Penalty Method for Controllable
  Pareto Front Learning under Split Feasibility Conditions}

\author[1]{Nguyen Viet Hoang}[orcid=0009-0005-5535-3182]
\ead{Hoang.NV227019@sis.hust.edu.vn}

\credit{Methodology, Software, Validation, Writing -- Original Draft}

\author[2]{Dung D. Le}
\ead{dung.ld@vinuni.edu.vn}

\credit{Supervision, Writing -- Review \& Editing}

\author[1]{Tran Ngoc Thang}
\cormark[1]
\ead{thang.tranngoc@hust.edu.vn}

\credit{Methodology, Conceptualization, Formal Analysis, Supervision,
  Writing -- Review \& Editing}

\affiliation[1]{organization={Faculty of Applied Mathematics and
    Informatics, Hanoi University of Science and Technology},
  city={Hanoi},
  country={Vietnam}}

\affiliation[2]{organization={College of Engineering and Computer Science,
    VinUniversity},
  city={Hanoi},
  country={Vietnam}}

\cortext[1]{Corresponding author}

\begin{abstract}
We address the open problem of training hypernetworks for Controllable
Pareto Front Learning (CPFL) under split feasibility conditions with
rigorous theoretical guarantees.
We reformulate the constrained Pareto problem as a \emph{Bi-Level
Scalarized Split Problem} (BSSP) and propose the \emph{Adaptive
Balanced Penalty} (ABP) algorithm, whose three gradient
components---optimality, set feasibility, and image feasibility---are
blended through an adaptive indicator driven by a computable lower
bound.
Using a novel convex surrogate technique, we prove full-sequence
convergence under standard convexity and Robbins--Monro step-size
assumptions.
The ABP penalty structure is then translated into a two-phase,
feasibility-first training strategy for Hyper-MLP and HyperTrans
architectures (ABP-HyperNet).
To evaluate constrained CPFL, we introduce the Expected Feasible
Hypervolume (EFHV), which jointly captures solution quality
and constraint satisfaction.
Experiments on five multi-objective benchmarks validate the ABP solver
against ground truth, while three multi-task learning datasets
demonstrate that ABP-HyperNet achieves up to $2.3\times$ higher EFHV
than unconstrained baselines by raising feasibility from
36--49\% to 87--100\%.
\end{abstract}

\begin{highlights}
\item Bi-level scalarized split problem (BSSP) formulation for constrained CPFL.
\item Adaptive balanced penalty (ABP) algorithm with proven full-sequence convergence.
\item Two-phase feasibility-first training strategy for ABP-HyperMLP and ABP-HyperTrans.
\item Expected feasible hypervolume (EFHV) metric, combining feasibility rate with feasible-subset hypervolume, shows up to $2.3\times$ improvement over baselines on three MTL datasets.
\end{highlights}

\begin{keywords}
Multi-objective optimization \sep
Controllable Pareto front learning \sep
Split feasibility problem \sep
Hypernetwork \sep
Transformer \sep
Penalty method
\end{keywords}

\begin{graphicalabstract}
\includegraphics[width=0.9\textwidth]{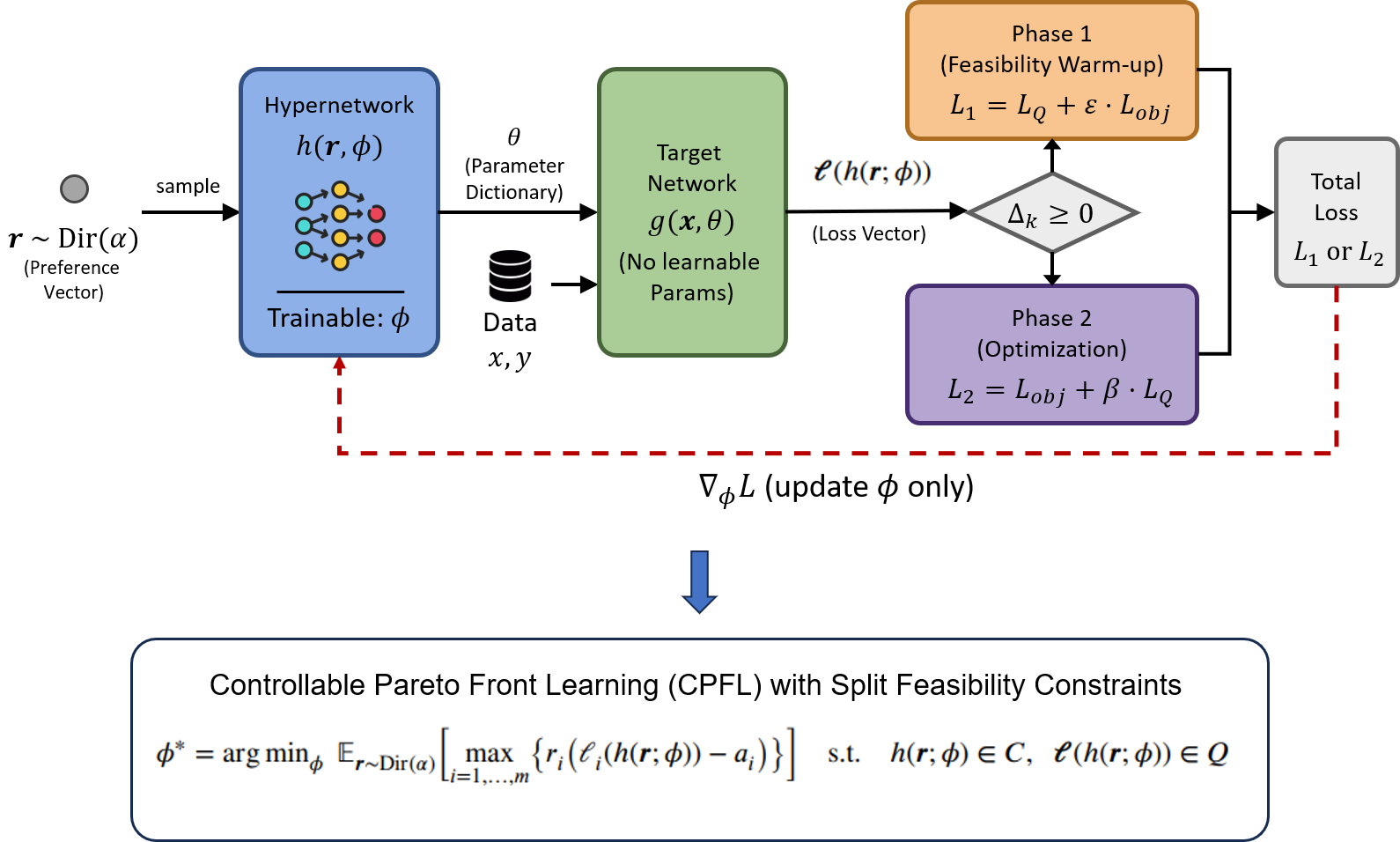}

\bigskip\noindent
\textbf{Graphical Abstract:} \textit{Controllable Pareto Front Learning} (CPFL) \textit{under Split Feasibility Conditions}
uses a hypernetwork to map preference vectors to Pareto optimal
solutions that strictly reside within a decision maker's specified
target region~$Q$.
We reformulate this constrained task as a \textit{Bi-Level Scalarized
Split Problem} (BSSP) and propose an \textit{Adaptive Balanced Penalty}
algorithm that systematically blends three gradient
components---optimality, set feasibility, and image feasibility.
The resulting principled penalty structure is translated into a
two-phase, feasibility-first training strategy for Hyper-MLP and
HyperTrans architectures, enabling reliable convergence to optimal
trade-offs while consistently satisfying box and sphere constraints.
\end{graphicalabstract}

\maketitle

\section{Introduction}\label{sec:intro}

Multi-objective optimization (MOO) addresses problems with several
conflicting objectives that must be optimized simultaneously over a
shared decision space.
Its applications span chemistry~\citep{cao2019}, finance~\citep{vuong2023},
and, notably, deep multi-task learning (MTL)~\citep{sener2018}.
In the MOO setting, the notion of optimality is captured by the Pareto
front: the set of all non-dominated trade-off points in the objective
space.

\emph{Controllable Pareto Front Learning} (CPFL) trains a single
model---typically a hypernetwork---that, given an arbitrary preference
vector at inference time, outputs a Pareto-optimal solution corresponding
to that preference.
Traditional CPFL methods approximate the \emph{entire} Pareto
front~\citep{navon2021,lin2019}, using evolutionary
algorithms~\citep{jangir2021} or Pareto set learning~\citep{lin2022}.
\citet{tuan2023} introduced a multi-layer-perceptron-based hypernetwork
(Hyper-MLP) for this task, and \citet{tuan2024} proposed a
transformer-based variant (HyperTrans) that leverages self-attention to
model pairwise trade-offs between objectives.
In the multi-task learning (MTL) setting, CPFL-based approaches differ
from classical gradient-balancing strategies~\citep{sener2018} in that a
single hypernetwork produces task-specific weights for any given
preference, enabling real-time trade-off control at inference.
Crucially, all existing CPFL methods---including those
of~\citep{tuan2023,tuan2024}---learn the \emph{entire unconstrained}
Pareto front;
no mechanism is provided to restrict solutions to a decision-maker-specified
region of the objective space.

The need for incorporating constraints into learning algorithms extends
far beyond multi-objective optimization.
In supervised learning, constraints enforce domain knowledge---such as
physical laws in physics-informed neural
networks~\citep{raissi2019}---or fairness requirements across
demographic groups~\citep{agarwal2018}, while in reinforcement
learning, Constrained Markov Decision Processes guarantee safety
during exploration~\citep{altman1999,achiam2017}.
More broadly, penalty and augmented Lagrangian techniques have become
standard tools for embedding functional constraints into stochastic
gradient training~\citep{bertsekas1999,lan2013}.

In the CPFL setting specifically, the relevant constraint is an
\emph{objective-space} region~$Q$ specifying acceptable trade-off
outcomes---for example, a budget constraint, a quality threshold,
or a regulatory bound.
This motivates \emph{constrained} Controllable Pareto Front Learning,
which seeks to learn only the portion of the Pareto front inside~$Q$.
The mathematical framework for this type of constraint comes from
\emph{Split Feasibility Problems} (SFP), introduced by
\citet{censor1994}, which seek a point in a closed convex set whose
image under a given operator lies in another closed convex set.
SFP has found applications in signal processing, image
reconstruction~\citep{byrne2003}, and radiation
therapy~\citep{censor2005,brooke2021}; extensions to variational
inequality settings were studied in~\citep{censor2012}, and the CQ
algorithm of \citet{byrne2002} remains the most widely used solver.
\citet{tuan2024} established the theoretical connection between SFP
and Pareto front learning by formulating the Split Multi-Objective
Optimization Problem (SMOP) and its outcome-space relaxation.
However, the practical training procedure of~\citep{tuan2024} does
not enforce the constraint $\calF(x)\in Q$: the constraint region~$Q$
appears only as a heuristic regularization term in the loss function,
and the resulting hypernetwork is trained to approximate the full,
unconstrained Pareto front.
To date, \emph{no existing method formulates or solves constrained
CPFL as a well-defined optimization problem with formal feasibility
guarantees}.

In the broader context of constrained optimization, penalty-based
methods have a long classical history~\citep{bertsekas1999,rockafellar1970},
and their adaptation to the stochastic gradient setting---including
convergence rate analysis under Robbins--Monro
schedules---was established by \citet{lan2013}.
In this work, we address the above gap in two complementary steps.
First, we introduce the \emph{Bi-Level Scalarized Split Problem}
(BSSP), a rigorous formulation that combines Chebyshev scalarization
with split feasibility on the extended downward hull~$Q^+$, and
we develop the \emph{Adaptive Balanced Penalty} (ABP) algorithm
(Algorithm~\ref{alg:main}) with proven full-sequence convergence.
Second, we translate the penalty structure identified in this
analysis---not the convergence proof itself---into a two-phase
training strategy for hypernetwork-based CPFL.
This strategy inherits the feasibility-first philosophy of the ABP
algorithm and operates as a
\emph{convergence-theory-inspired heuristic} in the stochastic,
non-convex neural network setting; its effectiveness is validated
empirically in Sections~\ref{sec:exp-mop}--\ref{sec:exp-mtl}.
The key insight is that the three gradient components of the ABP
algorithm (optimality, set feasibility, image feasibility) map
directly onto the two training phases: a feasibility-first warm-up
phase where the hypernetwork is driven into the constraint
region~$Q$, and a trade-off optimization phase where Chebyshev
quality is maximized under an adaptive feasibility restoring force.
Figure~\ref{fig:concept} provides a visual overview of this
contrast.

\begin{figure}[pos=h]
  \centering
  \begin{subfigure}[b]{0.45\textwidth}
    \centering
    \includegraphics[width=\textwidth]{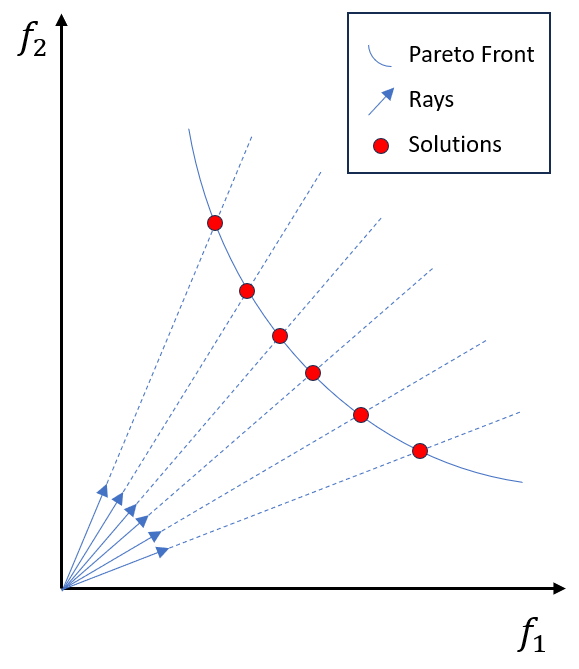}
    \caption{Unconstrained CPFL}
    \label{fig:concept_a}
  \end{subfigure}
  \hfill
  \begin{subfigure}[b]{0.45\textwidth}
    \centering
    \includegraphics[width=\textwidth]{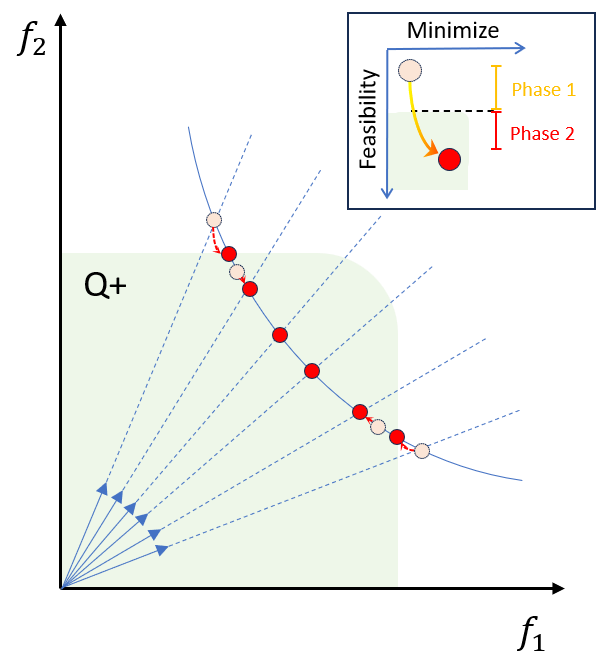}
    \caption{Constrained CPFL (proposed)}
    \label{fig:concept_b}
  \end{subfigure}
  \caption{Conceptual overview of Controllable Pareto Front Learning.
  (a)~Existing CPFL methods~\citep{navon2021,tuan2023,tuan2024}
  approximate the entire unconstrained Pareto front; solutions may
  lie anywhere in the objective space.
  (b)~Our BSSP two-phase training strategy restricts solutions to
  a decision-maker-specified region~$Q$, systematically driving them
  into~$Q$ while optimizing the trade-off.}
  \label{fig:concept}
\end{figure}

\paragraph{Contributions.}
The main contributions of this paper are threefold:

\begin{enumerate}[label=(\roman*),leftmargin=*,itemsep=3pt]
\item \textbf{Bi-level scalarized split problem and algorithm.}
  We introduce the BSSP formulation that combines Chebyshev scalarization
  with split feasibility on the extended downward hull~$Q^+$, eliminating
  the auxiliary variable required by the formulation of~\citep{tuan2024}.
  An Adaptive Balanced Penalty algorithm is proposed that blends three
  gradient components---optimality, set feasibility, and image
  feasibility---using an adaptive indicator driven by a computable
  lower bound, and whose structure directly informs the two-phase
  training strategy.

\item \textbf{Convergence guarantee.}
  We establish full-sequence convergence of the ABP algorithm under
  standard convexity and Robbins--Monro step-size assumptions, together
  with a zero-gap condition for which we provide verifiable sufficient
  conditions.
  The central technical device is a convex surrogate construction:
  at each iteration, a linearized half-space converts the non-convex
  image-feasibility function into a global subgradient inequality,
  yielding the quasi-Fej\'{e}r property that drives the convergence
  analysis.
  This result supplies the theoretical justification for the two-phase
  training strategy described below.

\item \textbf{Two-phase feasibility-first training strategy for
  hypernetwork-based CPFL.}
  Guided by the penalty structure identified in the BSSP convergence
  analysis, we design a practical two-phase training strategy for
  hypernetwork-based constrained Pareto front learning.
  In Phase~1 (feasibility warm-up), the hypernetwork is trained
  exclusively to satisfy the constraint region~$Q$, driving all
  generated solutions into the feasible set before any trade-off
  optimization begins.
  In Phase~2 (trade-off optimization), the Chebyshev objective is
  maximized under an adaptive feasibility-restoring force that
  prevents constraint violations accumulated during optimality
  pursuit.
  We instantiate this strategy on two hypernetwork
  architectures---Hyper-MLP and HyperTrans---and evaluate it on five
  MOP benchmarks and three MTL datasets.
  The constrained configurations raise feasibility rates from
  36--49\% to 87--100\% and achieve up to $2.3\times$ higher EFHV
  than the unconstrained baselines of~\citep{tuan2023,tuan2024},
  with only a modest and controlled reduction in global hypervolume.
  While the convergence guarantee of Theorem~\ref{thm:main} applies
  to the deterministic convex BSSP, the training strategy inherits
  the same feasibility-first philosophy and delivers consistent
  empirical gains across all tested settings.
\end{enumerate}

\noindent\textit{Organization.}
The paper is organized as follows.
Section~\ref{sec:prelim} collects mathematical preliminaries.
Section~\ref{sec:problem} develops the problem formulation from
the split feasibility framework to the BSSP\@.
Section~\ref{sec:convergence} presents the Adaptive Balanced Penalty
algorithm and states the main convergence results, with full
proofs deferred to Appendix~\ref{app:proofs}.
Section~\ref{sec:cpfl} describes the CPFL formulation and
hypernetwork architectures for multi-task learning.
Sections~\ref{sec:exp-mop} and~\ref{sec:exp-mtl} present
experimental results on MOP benchmarks and MTL datasets,
respectively.
Section~\ref{sec:conclusion} concludes.

\section{Preliminaries}\label{sec:prelim}

We collect definitions and tools used throughout the paper.

\begin{definition}[Multi-objective optimization]\label{def:mop}
Consider the multi-objective optimization problem
\begin{equation}\label{eq:MOP}\tag{MOP}
  \min_{x\in C}\;\calF(x),
  \qquad
  \calF(x)=\bigl(f_1(x),\dots,f_m(x)\bigr),
\end{equation}
where $C\subset\R^n$ is a nonempty closed convex set and each
$f_i\colon\R^n\to\R$ is continuously differentiable.
We denote the \textit{outcome set} by $Y:=\calF(C)=\{\calF(x):x\in C\}$.
\end{definition}

\begin{definition}[Pareto optimality]\label{def:pareto}
A point $x^*\in C$ is \emph{weakly Pareto optimal} for~\eqref{eq:MOP}
if there is no $x\in C$ with $\calF(x)\prec\calF(x^*)$
(i.e., $f_i(x)<f_i(x^*)$ for all~$i$).
The set of weakly Pareto optimal solutions is denoted $X_{WE}$,
and its image $Y_{WE}:=\calF(X_{WE})$ is the \emph{weakly Pareto
front}.
A point $x^*\in C$ is \emph{Pareto stationary} if
$\max_{i=1,\ldots,m}\nabla f_i(x^*)^\top d\ge 0$ for all
$d\in\R^n$.
When each $f_i$ is convex, every Pareto stationary point is weakly
Pareto optimal~\citep{ehrgott2005}.
\end{definition}

\begin{definition}[Chebyshev scalarization]\label{def:cheby}
Let $z^*=(z_1^*,\dots,z_m^*)$ be the \emph{ideal point} defined by
$z_i^*:=\inf_{x\in C}f_i(x)$, assumed finite for every~$i$.
Given a preference vector
$\bm{r}=(r_1,\dots,r_m)\in\R^m_{>0}$ with $\sum_{i=1}^m r_i=1$,
the \emph{weighted Chebyshev scalarization} is
\begin{equation}\label{eq:cheb}
  s\bigl(\calF(x),\bm{r}\bigr)
  :=\max_{i=1,\dots,m}\bigl\{r_i\bigl(f_i(x)-z_i^*\bigr)\bigr\}.
\end{equation}
When each $f_i$ is convex, $s(\calF(\cdot),\bm{r})$ is convex as
the pointwise maximum of convex functions.
A classical result~\citep[Theorem~3.4.5]{miettinen1999} states that
every minimizer of~\eqref{eq:cheb} over~$C$ is weakly
Pareto optimal for~\eqref{eq:MOP}.
\end{definition}

\begin{definition}[Extended downward hull]\label{def:downhull}
Let $Q\subset\R^m$ be a nonempty closed convex set representing a
desired region in the objective space.
The \emph{extended downward hull} is
\begin{equation}\label{eq:Qplus}
  Q^+:=Q-\R^m_+=\{y-u:y\in Q,\;u\in\R^m_+\}.
\end{equation}
The set $Q^+$ enlarges $Q$ by ``free disposal'': if a point belongs
to $Q^+$, then so does any point that is componentwise no larger.
This construction plays a role analogous to the extended outcome set
$Y^+=Y+\R^m_+$ used in~\citep{tuan2024}.
\end{definition}

\begin{lemma}\label{lem:Qplus}
If $Q\subset\R^m$ is nonempty, closed and convex, then $Q^+$ is
nonempty, closed and convex.
\end{lemma}

\begin{proof}
Nonemptiness and convexity follow from $Q\subset Q^+$ and the
fact that $Q^+$ is the Minkowski sum of the convex sets $Q$ and
$-\R^m_+$.  Closedness holds because $-\R^m_+$ is a polyhedral
convex cone, and the Minkowski sum of a closed convex set with a
polyhedral convex set is closed; see, e.g.,
\citep[Corollary~9.1.1]{rockafellar1970}, \citep[Theorem~3.12]{rockafellar2009}, or
\citep[Theorem~6.5]{bauschke2017}.
\end{proof}

\begin{proposition}[Free disposal]\label{prop:free}
If $z\in Q^+$ and $z'\le z$ componentwise, then $z'\in Q^+$.
\end{proposition}

\begin{proof}
$z=y-u$ with $y\in Q$, $u\ge 0$.  Then $z'=y-(u+z-z')$ with
$u+z-z'\ge 0$.
\end{proof}

We now derive an explicit formula for $P_{Q^+}$ that is central to
the algorithm.

\begin{proposition}[Projection onto $Q^+$]\label{prop:proj}
For every $\tilde z\in\R^m$,
\begin{equation}\label{eq:proj}
  P_{Q^+}(\tilde z)=\min\bigl(\tilde z,\,y^*\bigr),
  \qquad
  y^*\in\argmin_{y\in Q}\;\bigl\|(\tilde z-y)_+\bigr\|^2,
\end{equation}
where $\min$ and $(\cdot)_+$ act componentwise.
\end{proposition}

\begin{proof}
The projection $z^*:=P_{Q^+}(\tilde z)$ exists and is unique
(Lemma~\ref{lem:Qplus}).

\textit{Step~1.}
Every $z\in Q^+$ has the form $z=y-u$ with $y\in Q$,
$u\in\R^m_+$.  For fixed~$y$, the separable minimization
$\min_{u\ge 0}\norm{y-u-\tilde z}^2$ is solved componentwise by
$u_j^*=\max(y_j-\tilde z_j,0)=(y_j-\tilde z_j)_+$
for each~$j$, giving residual $\norm{(\tilde z-y)_+}^2$.
Hence $\inf_{z\in Q^+}\norm{z-\tilde z}^2
=\inf_{y\in Q}\norm{(\tilde z-y)_+}^2$.

\textit{Step~2.}
Write $z^*=y_0-u_0$ with $y_0\in Q$, $u_0\ge 0$.
Set $\hat u:=(y_0-\tilde z)_+$ and $\hat z:=y_0-\hat u\in Q^+$.
By Step~1, $\hat u$ is the minimizer of
$\min_{u\ge 0}\norm{y_0-u-\tilde z}^2$, so
$\norm{\hat z-\tilde z}^2\le\norm{y_0-u_0-\tilde z}^2
=\norm{z^*-\tilde z}^2$.
On the other hand, $z^*$ is the nearest point in $Q^+$ to
$\tilde z$, so $\norm{z^*-\tilde z}^2\le\norm{\hat z-\tilde z}^2$.
Combining: $\norm{\hat z-\tilde z}^2=\norm{z^*-\tilde z}^2$.
Uniqueness of the projection forces $\hat z=z^*$.  Thus
$y^*:=y_0$ attains the infimum in Step~1.

\textit{Step~3.}
$z^*=y^*-(y^*-\tilde z)_+=\min(y^*,\tilde z)$.
\end{proof}

\begin{corollary}\label{cor:res}
$\rho:=\tilde z-P_{Q^+}(\tilde z)\ge 0$ componentwise for every
$\tilde z\in\R^m$.
\end{corollary}

\begin{proof}
Set $p:=P_{Q^+}(\tilde z)$.  The characterization of the
metric projection onto a closed convex set gives
$\ip{\rho}{w-p}\le 0$ for all $w\in Q^+$.
By Proposition~\ref{prop:free}, $p-te_j\in Q^+$ for every $t>0$
and every standard basis vector~$e_j$, so $-t\rho_j\le 0$,
whence $\rho_j\ge 0$.
\end{proof}

\begin{definition}[Metric projection and subdifferential]
\label{def:proj-subdiff}
For a nonempty closed convex set $S\subset\R^n$, the
\emph{metric projection} $P_S(x):=\argmin_{s\in S}\norm{x-s}$
is well-defined and satisfies the variational characterization
$\ip{x-P_S(x)}{s-P_S(x)}\le 0$ for all $s\in S$.
The \emph{squared distance function}
$d_S(x):=\tfrac{1}{2}\dist^2(x,S)$ is convex and differentiable
with $\nabla d_S(x)=x-P_S(x)$.
For a convex function $\varphi\colon\R^n\to\R$, the
\emph{subdifferential} at~$x$ is
$\partial\varphi(x):=\{g\in\R^n:\varphi(y)\ge\varphi(x)+
\ip{g}{y-x}\;\forall\,y\}$.
If $\varphi=\max_{i}\varphi_i$ with each $\varphi_i$ differentiable,
then $\nabla\varphi_{i^*}(x)\in\partial\varphi(x)$ for any
$i^*\in\argmax_i\varphi_i(x)$.
\end{definition}

\section{Problem formulation}\label{sec:problem}

In this section, we develop the Bi-Level Scalarized Split Problem
(BSSP) through a systematic chain of reformulations, following and
extending the framework of~\citep{tuan2024}.

\subsection{From split feasibility to split multi-objective
optimization}\label{subsec:sfp-smop}

The \emph{Split Feasibility Problem} (SFP), introduced by Censor and
Elfving~\citep{censor1994}, seeks
\begin{equation}\label{eq:SFP}\tag{SFP}
  \text{Find } x^*\in C\;:\;\calF(x^*)\in Q,
\end{equation}
where $C\subset\R^n$ and $Q\subset\R^m$ are nonempty closed convex
sets and $\calF\colon\R^n\to\R^m$ is a smooth mapping.
The most widely used solver is the CQ algorithm of
\citet{byrne2002}, which iterates
\begin{equation}\label{eq:CQ}
  x^{k+1}=P_C\!\bigl(x^k-\gamma_k\,J_\calF(x^k)^T
  (\calF(x^k)-P_Q(\calF(x^k)))\bigr),
\end{equation}
where $P_C$ and $P_Q$ denote metric projections onto $C$ and $Q$,
respectively, and $\gamma_k>0$ is a step size.

The classical SFP~\eqref{eq:SFP} requires the constraint set~$C$
to be convex.
However, when $C$ is taken to be the weakly Pareto optimal solution
set~$X_{WE}$ of~\eqref{eq:MOP}, $C$ is in general \emph{non-convex},
even when $\calF$ is linear~\citep{kim2013}.
This motivates the \emph{Split Multi-Objective Optimization Problem}
(SMOP)~\citep{tuan2024}:
\begin{equation}\label{eq:SMOP}\tag{SMOP}
  \text{Find } x^*\in X_{WE}\;:\;\calF(x^*)\in Q.
\end{equation}
Problem~\eqref{eq:SMOP} seeks a Pareto-optimal solution whose
image in the objective space lies within the decision-maker's
preferred region~$Q$.
To select among multiple solutions of~\eqref{eq:SMOP},
\citet{tuan2024} considered optimizing a scalarization
function over the feasible set of~\eqref{eq:SMOP}:
\begin{equation}\label{eq:SP}\tag{SP}
  \min_{x}\;S(\calF(x))
  \quad\text{s.t.}\quad
  x\in X_{WE},\quad\calF(x)\in Q,
\end{equation}
where $S\colon\R^m\to\R$ is a monotonically increasing,
pseudoconvex function.

Since $X_{WE}$ is non-convex, directly solving~\eqref{eq:SP} is
intractable.
The \emph{outcome-space approach}~\citep{tuan2024,kim2013} transforms
the problem using the extended outcome set
$Y^+:=Y+\R^m_+=\{y+u:y\in Y,\;u\in\R^m_+\}$,
which enjoys several favourable properties:
(i) $Y^+$ is a closed convex set;
(ii) $\partial Y^+=Y^+_{WE}$ (the weakly efficient boundary equals the topological boundary);
(iii) $Y^+$ is a reverse-normal set.
These properties lead to the equivalent outcome-space problems:
\begin{align}
  &\min_y\;S(y) \quad\text{s.t.}\quad
  y\in Y^+_{WE}\cap Q,
  \tag{OSP$^+$}\label{eq:OSPplus}\\[4pt]
  &\min_y\;S(y) \quad\text{s.t.}\quad
  y\in Y^+\cap Q,
  \tag{$\overline{\text{OSP}}$}\label{eq:OSPbar}
\end{align}
whose equivalence is established in~\citep[Propositions~4.3
and~4.4]{tuan2024} when $Q$ is a normal set.
The explicit-form problem introduces the pair $(x,y)$:
\begin{equation}\label{eq:ESP}\tag{ESP}
  \min_{(x,y)}\;S(\calF(x))
  \quad\text{s.t.}\quad
  x\in C,\quad y\in Q,\quad\calF(x)\le y,
\end{equation}
which is a pseudoconvex program when each $f_i$ is convex
\citep[Proposition~4.5]{tuan2024}.

\subsection{From ESP to Bi-Level Scalarized Split Problem}\label{subsec:bssp}

Problem~\eqref{eq:ESP} is solvable by standard gradient methods
but requires the auxiliary variable~$y$, and the coupling constraint
$\calF(x)\le y$ can slow convergence in high dimensions.
We propose a more direct formulation by eliminating~$y$ through
the extended downward hull $Q^+$ (Definition~\ref{def:downhull}).

\begin{proposition}\label{prop:equiv}
The image-feasibility condition of~\eqref{eq:ESP}---namely,
$\exists\,y\in Q$ with $\calF(x)\le y$---is equivalent to
$\calF(x)\in Q^+$.
\end{proposition}

\begin{proof}
If $\calF(x)\le y$ with $y\in Q$, then
$\calF(x)=y-\bigl(y-\calF(x)\bigr)\in Q-\R^m_+=Q^+$.
Conversely, if $\calF(x)\in Q^+$, write
$\calF(x)=y_0-u_0$ with $y_0\in Q$, $u_0\ge 0$; then
$\calF(x)\le y_0\in Q$.
\end{proof}

Replacing $S$ by the Chebyshev scalarization~\eqref{eq:cheb} and
using Proposition~\ref{prop:equiv}, we arrive at the
\emph{Bi-Level Scalarized Split Problem}:
\begin{equation}\label{eq:BSSP}\tag{BSSP}
  \min_{x\in\R^n}\;s\bigl(\calF(x),\bm{r}\bigr)
  \quad\text{subject to}\quad
  x\in C,\quad\calF(x)\in Q^+.
\end{equation}

\begin{remark}[Relation to~\citep{tuan2024}]
Problem~\eqref{eq:BSSP} is equivalent to~\eqref{eq:ESP} with $S$
replaced by the Chebyshev function, but avoids the auxiliary
variable~$y$ by encoding the split feasibility constraint directly
through~$Q^+$.
This simplification enables a clean three-component penalty
decomposition (Section~\ref{sec:algo}) and a transparent
correspondence with the hypernetwork training loss
(Section~\ref{sec:cpfl}).
\end{remark}

We introduce three component functions that will drive the algorithm:
\begin{align}
  \varphi(x) &:= s\bigl(\calF(x),\bm{r}\bigr)
    =\max_{i=1,\dots,m}r_i\bigl(f_i(x)-z_i^*\bigr),
    \label{eq:phi}\\[4pt]
  H(x) &:= \tfrac{1}{2}\dist^2(x,C),
    \label{eq:H}\\[4pt]
  G(x) &:= \tfrac{1}{2}\dist^2\!\bigl(\calF(x),Q^+\bigr).
    \label{eq:G}
\end{align}
The feasible set of~\eqref{eq:BSSP} is
$\mathcal{S}:=\{x\in\R^n:H(x)=0,\;G(x)=0\}$,
the optimal value is
$\varphi^*:=\inf\{\varphi(x):x\in \mathcal{S}\}$, and the solution set is
$\Omega:=\{x\in \mathcal{S}:\varphi(x)=\varphi^*\}$.

A central ingredient of the algorithm is a computable
\emph{lower bound} for~$\varphi^*$.  Define
\begin{equation}\label{eq:philb}
  \varphi_{\lb}:=\inf_{x\in C}\,\varphi(x).
\end{equation}
Since $\mathcal{S}\subset C$, every feasible point of~\eqref{eq:BSSP}
belongs to~$C$, so
\begin{equation}\label{eq:lb-valid}
  \varphi_{\lb}\le\varphi^*.
\end{equation}
Problem~\eqref{eq:philb} is a convex program (minimization of the
pointwise maximum of convex functions over a convex set) and can be
solved---or approximated to arbitrary precision---by standard
methods~\citep{bertsekas1999,nesterov2004}.  We write
\begin{equation}\label{eq:sigma}
  \sigma:=\varphi^*-\varphi_{\lb}\ge 0
\end{equation}
for the \emph{bound gap}.  The algorithm uses only~$\varphi_{\lb}$
(never~$\varphi^*$ itself).

\section{Algorithm and convergence analysis}\label{sec:convergence}

\subsection{Algorithm and assumptions}\label{sec:algo}

At iteration~$k$, define
\begin{equation}\label{eq:pk}
  p^k:=P_{Q^+}\!\bigl(\calF(x^k)\bigr),
  \qquad
  \rho^k:=\calF(x^k)-p^k\ge 0,
\end{equation}
and the direction vectors
\begin{align}
  z^k &:= x^k-P_C(x^k),
    \label{eq:zk}\\
  v^k &:= J_{\calF}(x^k)^T\,\rho^k,
    \label{eq:vk}\\
  w^k &:= r_{i^*}\nabla f_{i^*}(x^k),\qquad
    i^*\in\argmax_i\;r_i\bigl(f_i(x^k)-z_i^*\bigr).
    \label{eq:wk}
\end{align}
When the $\argmax$ in~\eqref{eq:wk} is not a singleton, any
choice of~$i^*$ yields a valid subgradient
$w^k\in\partial\varphi(x^k)$
(see Lemma~\ref{lem:phi} below), so the algorithm is well-defined
regardless of the tie-breaking rule.

The three direction vectors have clear roles:
$w^k$ drives \emph{optimality} (descent of the Chebyshev
scalarization),
$z^k$ enforces \emph{set feasibility} ($x\in C$), and
$v^k$ enforces \emph{image feasibility} ($\calF(x)\in Q^+$).

Set $\Delta_k:=\varphi(x^k)-\varphi_{\lb}$.  The \emph{adaptive
composite direction} is
\begin{equation}\label{eq:dk}
  d^k:=\alpha_k\,\mathbf{1}_{\{\Delta_k\ge 0\}}\,w^k
  +\beta_k\,z^k+\gamma_k\,v^k.
\end{equation}
Since $\varphi_{\lb}\le\varphi^*\le\varphi(x)$ for $x\in \mathcal{S}$, the
indicator equals~$1$ at every feasible iterate.

\begin{algorithm}[H]
\caption{Adaptive balanced penalty (ABP) method}\label{alg:main}
\begin{algorithmic}[1]
\REQUIRE $x^0\in\R^n$;\; $\bm{r}\in\R^m_{>0}$;\;
  $\varphi_{\lb}\le\varphi^*$;\; $\mu>0$;\;
  $\{\alpha_k\},\{\beta_k\},\{\gamma_k\},\{\lambda_k\}$.
\FOR{$k=0,1,2,\dots$}
  \STATE Compute $p^k,\,\rho^k$ via~\eqref{eq:pk};\;
    $z^k,\,v^k,\,w^k$ via~\eqref{eq:zk}--\eqref{eq:wk};
  \STATE $\Delta_k\leftarrow\varphi(x^k)-\varphi_{\lb}$;
  \STATE $d^k$ via~\eqref{eq:dk};\quad
    $\eta_k\leftarrow\max(\mu,\,\norm{d^k})$;
  \STATE $x^{k+1}\leftarrow x^k
    -\dfrac{\lambda_k}{\eta_k}\,d^k$;
\ENDFOR
\end{algorithmic}
\end{algorithm}

\begin{assumption}\label{ass:all}\mbox{}
\begin{enumerate}[label=\textup{(A\arabic*)},ref=A\arabic*,
  leftmargin=*,itemsep=4pt]
\item\label{A1}
  $\calF=(f_1,\dots,f_m)\in C^1(\R^n;\R^m)$.
\item\label{A2}
  $C\subset\R^n$ is nonempty, closed and convex;\;
  $Q\subset\R^m$ is nonempty, closed and convex;\;
  $z_i^*>-\infty$ for each~$i$.
\item\label{A3}
  Each $f_i$ is convex.
\item\label{A4}
  $\Omega\neq\emptyset$, i.e., the solution set of~\eqref{eq:BSSP}
  is nonempty.
\item\label{A5}
  $\{\lambda_k\}\subset(0,\infty)$,\;
  $\sum\lambda_k=\infty$,\;
  $\sum\lambda_k^2<\infty$.
\item\label{A6}
  $0<\underline\alpha\le\alpha_k\le\bar\alpha$,\;
  $0<\underline\beta\le\beta_k\le\bar\beta$,\;
  $0<\underline\gamma\le\gamma_k\le\bar\gamma$
  for all $k$.
\item\label{A7}
  \textup{(Zero-gap condition).}
  $\varphi_{\lb}=\varphi^*$\; (equivalently, $\sigma=0$).
\end{enumerate}
\end{assumption}

Assumptions~\eqref{A1}--\eqref{A3} are standard regularity and
convexity conditions.
Assumption~\eqref{A5} is the classical Robbins--Monro step-size
condition~\citep{robbins1951}, satisfied by $\lambda_k=1/(k+1)$.
Assumption~\eqref{A6} allows the penalty weights to be fixed
constants.
Assumption~\eqref{A4} holds whenever $C$ is bounded and
$\mathcal{S}\neq\emptyset$.

\paragraph{Weakened zero-gap condition.}
When the zero-gap condition~\eqref{A7} cannot be verified exactly, we replace
it with the following two weaker assumptions.

\begin{assumption}[Bounded gap \textup{(A7')}]\label{ass:gap}
  $\sigma := \varphi^* - \varphi_{\lb} \le \varepsilon_0$ for some
  known constant $\varepsilon_0 \ge 0$.
\end{assumption}

\begin{assumption}[Boundedness \textup{(A8)}]\label{ass:bdd}
  The sequence $\{x^k\}$ generated by the ABP algorithm is
  bounded: $\sup_k \|x^k\| \le B < \infty$.
\end{assumption}

\begin{remark}\label{rem:a8-status}
When $\varepsilon_0 = 0$, Assumption~(A7') reduces to~\eqref{A7} and
Assumption~(A8) is a \emph{consequence} of the quasi-Fej\'{e}r property
(Proposition~\ref{prop:fejer}) rather than an independent requirement.
When $C \subset \mathbb{R}^n$ is compact---as is the case for CVX1 and
CVX2, where $C$ is a closed box---Assumption~(A8) holds automatically
for any initialization $x^0 \in C$ and need not be imposed separately.
In general, (A8) can be enforced by projecting iterates onto a
sufficiently large ball, or verified a~posteriori for a specific instance.
\end{remark}

\begin{remark}[Assumption~\eqref{A3} and non-convex problems]
\label{rem:nonconvex}
Convexity of each $f_i$ is essential for Theorem~\ref{thm:main}:
it ensures $\varphi$ is convex and that the surrogate
(Lemma~\ref{lem:delta-cvx}) is a valid global minorant.
The ABP algorithm is applied as an effective heuristic for
non-convex objectives (e.g., the ZDT benchmarks of
Section~\ref{sec:exp-mop}).
\end{remark}

\subsection{Convergence results}\label{subsec:convergence-results}

\subsubsection{Key gradient inequalities}\label{subsubsec:grad-ineq}

We collect properties of $\varphi$, $H$, $G$ and construct a
convex surrogate for the non-convex function~$G$.

\begin{lemma}\label{lem:H}
Under~\eqref{A2}, $H$ is convex and $C^1$ with
$\nabla H(x)=x-P_C(x)$.~\textit{(Proof in Appendix~\ref{app:tech-lemmas}.)}
\end{lemma}

\noindent\textit{This means $z^k = \nabla H(x^k)$, so the set-feasibility step in the ABP algorithm is precisely a gradient descent step on~$H$.}

\begin{lemma}\label{lem:phi}
Under~\eqref{A1} and~\eqref{A3}, $\varphi$ is convex and
$w^k\in\partial\varphi(x^k)$ for every choice of~$i^*$
in~\eqref{eq:wk}.~\textit{(Proof in Appendix~\ref{app:tech-lemmas}.)}
\end{lemma}

\noindent\textit{This confirms that the ABP algorithm is well-defined regardless of which active index~$i^*$ is selected at each iteration.}

Since $G(x)=\tfrac{1}{2}\dist^2(\calF(x),Q^+)$ is generally
non-convex, we construct at each iteration a convex surrogate
that matches~$G$ in value and gradient at~$x^k$ and
\emph{minorizes}~$G$ globally.

Define the \emph{linearized half-space}
\begin{equation}\label{eq:Qk}
  Q^+_k:=\{y\in\R^m:\ip{\rho^k}{y-p^k}\le 0\},
\end{equation}
and the \emph{convex surrogate}
$\delta_k(x):=\tfrac{1}{2}\dist^2(\calF(x),Q^+_k)$.
When $\rho^k\neq 0$, the half-space projection
formula \citep[Example~28.16]{bauschke2017} gives
\begin{equation}\label{eq:delta-explicit}
  \delta_k(x)
  =\frac{\bigl[\ip{\rho^k}{\calF(x)-p^k}\bigr]_+^2}
        {2\norm{\rho^k}^2};
\end{equation}
when $\rho^k=0$, $\delta_k\equiv 0$.

\begin{lemma}\label{lem:Qsub}
$Q^+\subseteq Q^+_k$ for every~$k$.~\textit{(Proof in Appendix~\ref{app:tech-lemmas}.)}
\end{lemma}

\noindent\textit{The half-space $Q_k^+$ contains $Q^+$, ensuring the surrogate $\delta_k$ globally minorizes~$G$.}

\begin{lemma}\label{lem:delta-cvx}
Under~\eqref{A1} and~\eqref{A3}, $\delta_k$ is convex.~\textit{(Proof in Appendix~\ref{app:tech-lemmas}.)}
\end{lemma}

\noindent\textit{Convexity of $\delta_k$ is the key property that makes the global subgradient inequality in Lemma~\ref{lem:ineq-G} possible.}

\begin{lemma}\label{lem:delta-agree}
For every~$k$: $\delta_k(x^k)=G(x^k)$, and
$\nabla\delta_k(x^k)=v^k$ (with $v^k=0$ when $\rho^k=0$).~\textit{(Proof in Appendix~\ref{app:tech-lemmas}.)}
\end{lemma}

\noindent\textit{The surrogate $\delta_k$ agrees with $G$ in both value and gradient at~$x^k$, so no approximation error is introduced at the current iterate.}

The next three lemmas show that each direction vector has a useful
inner-product lower bound with respect to any $x^*\in\Omega$.

\begin{lemma}\label{lem:ineq-phi}
Under~\eqref{A1}, \eqref{A3}, \eqref{A7}:
$\ip{w^k}{x^k-x^*}\ge\Delta_k$ for all~$k$ and $x^*\in\Omega$.~\textit{(Proof in Appendix~\ref{app:tech-lemmas}.)}
\end{lemma}

\begin{lemma}\label{lem:ineq-H}
Under~\eqref{A2}:
$\ip{z^k}{x^k-x^*}\ge H_k\ge 0$ for all~$k$ and $x^*\in\Omega$.~\textit{(Proof in Appendix~\ref{app:tech-lemmas}.)}
\end{lemma}

\begin{lemma}\label{lem:ineq-G}
Under~\eqref{A1}--\eqref{A3}, $v^k$ satisfies a global subgradient
inequality for~$G$:
\begin{equation}\label{eq:ineq-G}
  G(y)\ge G_k+\ip{v^k}{y-x^k}
  \qquad\forall\,y\in\R^n,\;k\ge 0.
\end{equation}
In particular, $\ip{v^k}{x^k-x^*}\ge G_k\ge 0$ for all
$x^*\in\Omega$.~\textit{(Proof in Appendix~\ref{app:tech-lemmas}.)}
\end{lemma}

\noindent\textit{These three inequalities are the load-bearing elements of the Lyapunov analysis:
each direction vector $w^k$, $z^k$, $v^k$ produces an inner product with $x^k - x^*$ that is
quantitatively lower-bounded by the corresponding sub-optimality measure
$\Delta_k$, $H_k$, $G_k$ respectively.}

\subsubsection{The Lyapunov function and boundedness}\label{subsubsec:lyapunov}

Throughout, $H_k:=H(x^k)$, $G_k:=G(x^k)$,
$\Delta_k:=\varphi(x^k)-\varphi_{\lb}$.
All estimates in this subsection hold for every $x^*\in\Omega$.

The \emph{adaptive merit function} is
\begin{equation}\label{eq:Phi}
  \Phi_k:=\alpha_k\,\Delta_k^+ +\beta_k\,H_k+\gamma_k\,G_k,
  \qquad\Delta_k^+:=\max(\Delta_k,0)\ge 0.
\end{equation}

\begin{lemma}[Lyapunov inequality]\label{lem:lyap}
Under Assumption~\ref{ass:all}, for every $x^*\in\Omega$ and
$k\ge 0$,
\begin{equation}\label{eq:lyap}
  \norm{x^{k+1}-x^*}^2
  \le\norm{x^k-x^*}^2
  -\frac{2\lambda_k}{\eta_k}\,\Phi_k
  +\lambda_k^2.
\end{equation}~\textit{(Proof in Appendix~\ref{app:lyapunov}.)}
\end{lemma}

\begin{lemma}[Robbins--Siegmund {\citep[Lemma~1]{robbins1971}}]
\label{lem:RS}
If $\{a_k\},\{b_k\}\subset[0,\infty)$ satisfy
$a_{k+1}\le a_k+b_k$ for all $k$ and $\sum b_k<\infty$, then
$\{a_k\}$ converges.
\end{lemma}

\begin{proposition}[Quasi-Fej\'er property]\label{prop:fejer}
Under Assumption~\ref{ass:all}, the sequence
$\{\norm{x^k-x^*}^2\}$ converges for every $x^*\in\Omega$;
in particular, $\{x^k\}$ is bounded.~\textit{(Proof in Appendix~\ref{app:lyapunov}.)}
\end{proposition}

\begin{remark}
The Lyapunov inequality~\eqref{eq:lyap} requires only
$\eta_k\ge\|d^k\|$, which is built into
the ABP algorithm.
Boundedness of $\{x^k\}$ (Proposition~\ref{prop:fejer}) is
therefore established prior to---and independently of---any bound
on $\{d^k\}$.
\end{remark}

\begin{lemma}\label{lem:dir}
Under Assumption~\ref{ass:all}, $\norm{d^k}\le\bar M<\infty$ for
all~$k$, and $\mu\le\eta_k\le\bar\eta:=\max(\mu,\bar M)$.~\textit{(Proof in Appendix~\ref{app:lyapunov}.)}
\end{lemma}

\subsubsection{Convergence theorems}\label{subsubsec:conv-theorems}

The proof of global convergence proceeds in three stages.
First, the Lyapunov inequality (Lemma~\ref{lem:lyap}) and the
Robbins--Monro summability of $\{\lambda_k^2\}$ yield
$\sum \lambda_k \Phi_k < \infty$, which forces
$\liminf_{k\to\infty} \Phi_k = 0$.
Second, any subsequence along which $\Phi_k \to 0$ is shown to
drive all three penalty components ($\Delta_k^+$, $H_k$, $G_k$)
to zero simultaneously.
Third, a convergent sub-subsequence is extracted using boundedness
(Proposition~\ref{prop:fejer}), and uniqueness of the limit follows
from the quasi-Fej\'er property.

\begin{lemma}\label{lem:sum}
$\sum_{k=0}^\infty\lambda_k\Phi_k<\infty$.~\textit{(Proof in Appendix~\ref{app:global-conv}.)}
\end{lemma}

\begin{lemma}\label{lem:liminf}
$\liminf_{k\to\infty}\Phi_k=0$.~\textit{(Proof in Appendix~\ref{app:global-conv}.)}
\end{lemma}

\begin{proposition}\label{prop:simult}
There exists a subsequence $\{k_j\}$ with
$\Delta_{k_j}^+\to 0$, $H_{k_j}\to 0$, $G_{k_j}\to 0$, and
$\Delta_{k_j}\to 0$.~\textit{(Proof in Appendix~\ref{app:global-conv}.)}
\end{proposition}

\begin{proposition}\label{prop:subseq}
There exist a subsequence $\{k_j\}$ and $x^\infty\in\Omega$ with
$x^{k_j}\to x^\infty$.~\textit{(Proof in Appendix~\ref{app:global-conv}.)}
\end{proposition}

\begin{theorem}[Global convergence]\label{thm:main}
Under Assumption~\ref{ass:all}, $\{x^k\}$ generated by
the ABP algorithm (Algorithm~\ref{alg:main}) converges to some $x^\infty\in\Omega$.~\textit{(Proof in Appendix~\ref{app:global-conv}.)}
\end{theorem}

\begin{remark}\label{rem:exact-vs-approx}
Theorem~\ref{thm:main} relies on the zero-gap condition~\eqref{A7}.
The remainder of this section extends the analysis to the case
$\sigma > 0$, deriving quantitative approximate-feasibility and
approximate-optimality bounds that vanish as $\varepsilon_0 \to 0$.
\end{remark}

When only the weaker Assumption~(A7') holds ($\sigma \le \varepsilon_0$),
the Lyapunov inequality acquires a drift term proportional to
$\bar\alpha\sigma\lambda_k/\mu$, which prevents the Robbins--Siegmund
argument from establishing boundedness; hence Assumption~(A8) is
now required.
An ergodic (Ces\`{a}ro) bound replaces the exact summability, yielding
$\liminf_k\Phi_k \le C_*\varepsilon_0$, and the three feasibility/optimality
bounds follow by decomposing $\Phi_k$ into its nonnegative components.

When~\eqref{A7} holds exactly, Theorem~\ref{thm:main} guarantees
full-sequence convergence to a point in~$\Omega$.  The following
results establish what can be concluded when only the weaker
Assumption~(A7') is available, replacing the Robbins--Siegmund
quasi-Fej\'{e}r argument with an ergodic (Ces\`{a}ro) bound.

Recall the notation:
$\varphi(x)=\max_i r_i(f_i(x)-z_i^*)$,\;
$\varphi_{\lb}:=\inf_{x\in C}\varphi(x)$,\;
$\sigma:=\varphi^*-\varphi_{\lb}\ge 0$;
$H_k:=\tfrac12\dist^2(x^k,C)$,\;
$G_k:=\tfrac12\dist^2(\calF(x^k),Q^+)$,\;
$\Delta_k:=\varphi(x^k)-\varphi_{\lb}$;
$\Phi_k:=\alpha_k\Delta_k^++\beta_k H_k+\gamma_k G_k\ge 0$,
\; $\Delta_k^+:=\max(\Delta_k,0)$.

Under~(A8), Lemma~\ref{lem:dir} (boundedness of directions) holds
with $K:=\overline{B}(0,B+1)$ playing the role of the compact
set, giving the uniform bound $\norm{d^k}\le\bar{M}<\infty$ and hence
$\mu\le\eta_k\le\bar\eta:=\max(\mu,\bar M)$ for all~$k$.

\begin{lemma}[Modified optimality inner product]\label{lem:ineq-mod}
Let $x^*\in\Omega$.  Under \textup{\eqref{A1}, \eqref{A3}, (A7')}:
\begin{equation}\label{eq:ineq-phi-mod}
  \ip{w^k}{x^k-x^*} \;\ge\; \Delta_k - \sigma
  \qquad\forall\,k\ge 0.
\end{equation}~\textit{(Proof in Appendix~\ref{app:approx-conv}.)}
\end{lemma}

\begin{lemma}[Modified Lyapunov inequality]\label{lem:lyap-approx}
Under \textup{\eqref{A1}--\eqref{A6}, (A7'), (A8)}, for every $x^*\in\Omega$,
\begin{equation}\label{eq:lyap-mod}
  \norm{x^{k+1}-x^*}^2
  \;\le\;
  \norm{x^k-x^*}^2
  - \frac{2\lambda_k}{\eta_k}\,\Phi_k
  + \frac{2\bar\alpha\sigma}{\mu}\,\lambda_k
  + \lambda_k^2
  \qquad\forall\,k\ge 0.
\end{equation}~\textit{(Proof in Appendix~\ref{app:approx-conv}.)}
\end{lemma}

\begin{lemma}[Ergodic bound]\label{lem:ergodic}
Under \textup{\eqref{A1}--\eqref{A6}, (A7'), (A8)}, define
$C_* := \bar\alpha\bar\eta/\mu$ and
$\Lambda_N:=\sum_{k=0}^N\lambda_k$.  Then
\begin{equation}\label{eq:cesaro}
  \limsup_{N\to\infty}\;
  \frac{1}{\Lambda_N}\sum_{k=0}^N\lambda_k\Phi_k
  \;\le\; C_*\varepsilon_0,
\end{equation}
and in particular
\begin{equation}\label{eq:ergodic}
  \liminf_{k\to\infty}\,\Phi_k \;\le\; C_*\varepsilon_0.
\end{equation}~\textit{(Proof in Appendix~\ref{app:approx-conv}.)}
\end{lemma}

\begin{proposition}[Approximate simultaneous bounds]\label{prop:simult-approx}
Under \textup{\eqref{A1}--\eqref{A6}, (A7'), (A8)}, there exist a subsequence
$\{k_j\}_{j\ge 1}\subset\N$ and a constant $\ell_*\in[0,C_*\varepsilon_0]$
such that $\Phi_{k_j}\to\ell_*$ and, along this subsequence,
\begin{align}
  H_{k_j} &\;\le\; \frac{\Phi_{k_j}}{\underline\beta}
    \;\xrightarrow{j\to\infty}\;
    \frac{\ell_*}{\underline\beta}
    \;\le\; \frac{C_*\varepsilon_0}{\underline\beta},
    \label{eq:Hbound}\\
  G_{k_j} &\;\le\; \frac{\Phi_{k_j}}{\underline\gamma}
    \;\xrightarrow{j\to\infty}\;
    \frac{\ell_*}{\underline\gamma}
    \;\le\; \frac{C_*\varepsilon_0}{\underline\gamma},
    \label{eq:Gbound}\\
  \Delta_{k_j}^+ &\;\le\; \frac{\Phi_{k_j}}{\underline\alpha}
    \;\xrightarrow{j\to\infty}\;
    \frac{\ell_*}{\underline\alpha}
    \;\le\; \frac{C_*\varepsilon_0}{\underline\alpha}.
    \label{eq:Dbound}
\end{align}
In particular, when $\varepsilon_0=0$ we have $\ell_*=0$ and
$H_{k_j}\to 0$, $G_{k_j}\to 0$, $\Delta_{k_j}^+\to 0$.~\textit{(Proof in Appendix~\ref{app:approx-conv}.)}
\end{proposition}

\begin{theorem}[Approximate convergence]\label{thm:approx}
Under \textup{\eqref{A1}--\eqref{A6}, (A7'), (A8)}, let $C_*=\bar\alpha\bar\eta/\mu$
and let $\{k_j\}$, $\ell_*$ be as in
Proposition~\ref{prop:simult-approx}.  Every cluster point~$\hat x$
of $\{x^{k_j}\}$ satisfies:
\begin{enumerate}[label=\textup{(\roman*)},leftmargin=*,itemsep=4pt]
\item\label{thm:feas}
  \textup{(Approximate feasibility in $C$.)}
  \[
    H(\hat x)
    :=\tfrac12\dist^2(\hat x,C)
    \;\le\;\frac{\ell_*}{\underline\beta}
    \;\le\;\frac{C_*\varepsilon_0}{\underline\beta},
    \qquad\text{i.e., }
    \dist(\hat x,C)
    \;\le\;\sqrt{\frac{2C_*\varepsilon_0}{\underline\beta}}.
  \]
\item\label{thm:sfeas}
  \textup{(Approximate split feasibility.)}
  \[
    G(\hat x)
    :=\tfrac12\dist^2(\calF(\hat x),Q^+)
    \;\le\;\frac{\ell_*}{\underline\gamma}
    \;\le\;\frac{C_*\varepsilon_0}{\underline\gamma},
    \qquad\text{i.e., }
    \dist(\calF(\hat x),Q^+)
    \;\le\;\sqrt{\frac{2C_*\varepsilon_0}{\underline\gamma}}.
  \]
\item\label{thm:opt}
  \textup{(Approximate optimality.)}
  \[
    \varphi(\hat x)
    \;\le\;\varphi_{\lb}+\frac{\ell_*}{\underline\alpha}
    \;\le\;\varphi^*+\frac{C_*\varepsilon_0}{\underline\alpha}.
  \]
\end{enumerate}
When $\varepsilon_0=0$: $\ell_*=0$, all three bounds vanish, and
$\hat x\in\Omega$ \textup{(exact optimality and feasibility)}.~\textit{(Proof in Appendix~\ref{app:approx-conv}.)}
\end{theorem}

\begin{remark}[Summary of convergence conclusions]\label{rem:summary-approx}
Table~\ref{tab:approx-results} contrasts the conclusions for $\varepsilon_0=0$
and $\varepsilon_0>0$.

\begin{table}[h]
\centering
\renewcommand{\arraystretch}{1.3}
\begin{tabular}{lll}
\hline
Property & $\varepsilon_0=0$ & $\varepsilon_0>0$ \\
\hline
$\dist(\hat x,C)$ & $=0$ (exact) & $\le\sqrt{2C_*\varepsilon_0/\underline\beta}$ \\
$\dist(\calF(\hat x),Q^+)$ & $=0$ (exact) & $\le\sqrt{2C_*\varepsilon_0/\underline\gamma}$ \\
$\varphi(\hat x)-\varphi^*$ & $=0$ (exact) & $\le C_*\varepsilon_0/\underline\alpha$ \\
Convergence type & Full sequence (Thm.~\ref{thm:main}) & Subsequential \\
Boundedness (A8) & Consequence of analysis & Must be assumed \\
\hline
\end{tabular}
\caption{Convergence conclusions as a function of~$\varepsilon_0$.}
\label{tab:approx-results}
\end{table}
\end{remark}

\begin{remark}[Sharpness and parameter sensitivity]\label{rem:sharp}
The constant $C_*=\bar\alpha\bar\eta/\mu$ governs the approximation
quality in all three bounds.  It grows with the upper bound~$\bar\alpha$
on the optimality weight sequence and with $\bar\eta=\max(\mu,\bar M)$,
and decreases as the normalizer floor~$\mu$ increases.  Since the
feasibility errors scale as $O(\sqrt{\varepsilon_0})$ and the
optimality error scales as $O(\varepsilon_0)$, reducing the bound
gap~$\sigma$ (e.g., by tightening the lower-bound computation) yields
a quadratically faster improvement in feasibility quality and a linear
improvement in objective quality.
\end{remark}

\begin{remark}[Recovery of exact convergence at $\varepsilon_0=0$]
\label{rem:recovery}
When $\varepsilon_0=0$ (i.e., the zero-gap condition~\eqref{A7} holds),
all drift contributions vanish, Assumption~(A8) is
superfluous, and the ergodic bound tightens to $\liminf_k\Phi_k=0$.
Theorem~\ref{thm:approx} then yields $\hat x\in\Omega$ exactly.
Invoking the full-sequence convergence argument
(Proposition~\ref{prop:fejer} and Theorem~\ref{thm:main}) further shows
that the \emph{entire} sequence $\{x^k\}$ converges to~$\hat x$.
\end{remark}

\section{Controllable Pareto front learning for multi-task learning}
\label{sec:cpfl}

We now translate the ABP penalty framework into a practical
training procedure for hypernetwork-based CPFL in the multi-task
learning (MTL) setting.
Figure~\ref{fig:pipeline} illustrates the overall pipeline.

%
\begin{figure}
  \centering
  \includegraphics[width=\textwidth]{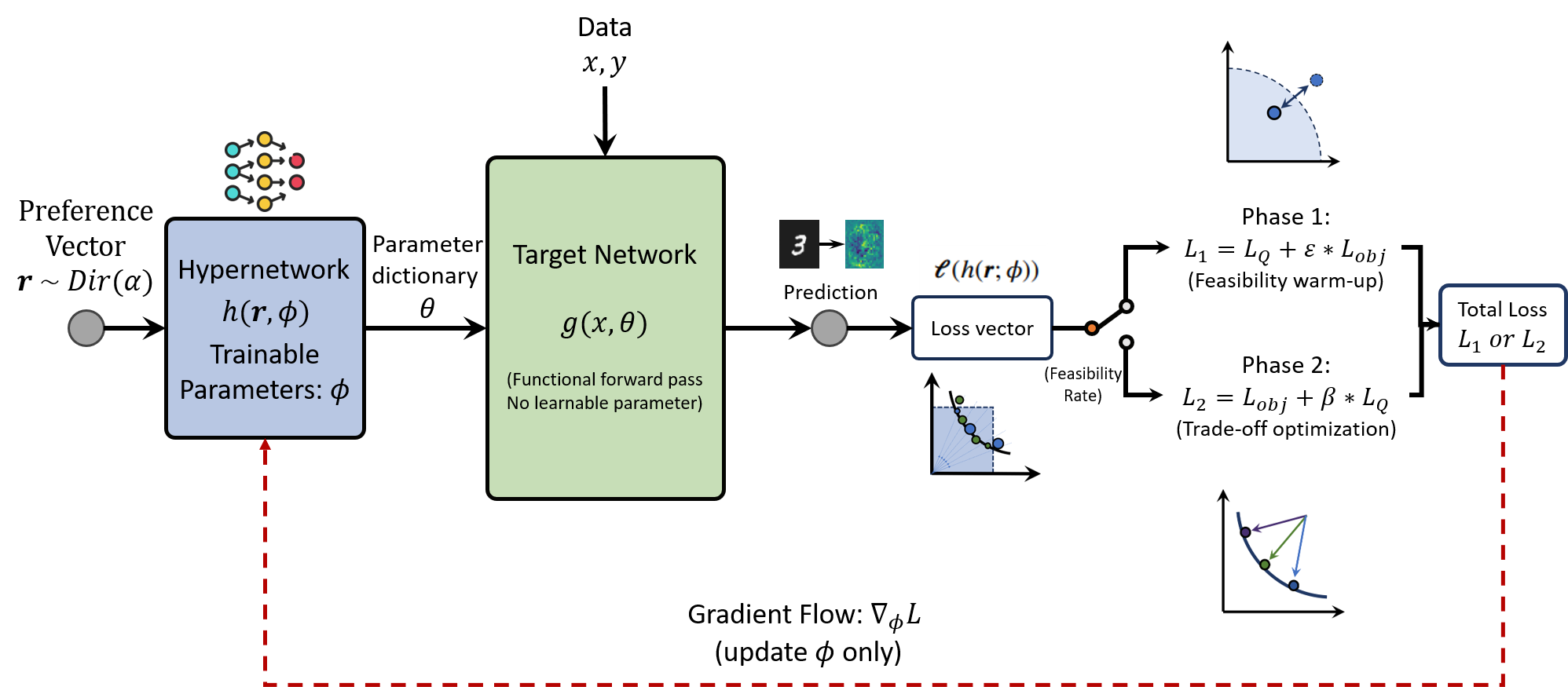}
  \caption{Training pipeline for hypernetwork-based CPFL
  under split feasibility conditions.}
  \label{fig:pipeline}
\end{figure}

\subsection{Multi-task learning as constrained multi-objective optimization}
\label{subsec:mtl-moo}

Consider $m$ learning tasks with losses
$\ell_1(\theta),\ldots,\ell_m(\theta)$, where $\theta\in C$ denotes
the parameters of a shared \emph{target network} and
$C\subseteq\R^{n_\theta}$ is the admissible parameter space.
The MTL problem is naturally a multi-objective program:
\[
  \min_{\theta\in C}\;\bigl(\ell_1(\theta),\ldots,\ell_m(\theta)\bigr).
\]
In the CPFL paradigm, a \emph{hypernetwork}
$h(\bm{r};\phi)\colon\R^m\to\R^{n_\theta}$ takes a preference
vector $\bm{r}\in\Delta^{m-1}$ (the probability simplex) as input
and produces the target network parameters
$\theta_{\bm{r}}=h(\bm{r};\phi)$.
A single set of hypernetwork parameters~$\phi$ thus parametrizes
the entire Pareto front.

Combining CPFL with the split feasibility framework
of~\citep{tuan2024}, the training problem is:
\begin{equation}\label{eq:cpfl-loss}
  \phi^*=\argmin_\phi\;\mathbb{E}_{\bm{r}\sim\mathrm{Dir}(\alpha)}
  \Bigl[\max_{i=1,\ldots,m}\bigl\{
  r_i\bigl(\ell_i(h(\bm{r};\phi))-a_i\bigr)\bigr\}\Bigr]
  \quad\text{s.t.}\quad
  h(\bm{r};\phi)\in C,\;\;
  \bm{\ell}(h(\bm{r};\phi))\in Q,
\end{equation}
where $\bm{\ell}=(\ell_1,\ldots,\ell_m)$,
$a_i$ is the ideal loss for task~$i$, and $Q$ is the
decision-maker's preferred region in the loss space.
In the notation of~\eqref{eq:BSSP}, the variable is
$x:=h(\bm{r};\phi)\in\R^{n_\theta}$,
the objective mapping is $\calF:=\bm{\ell}$,
the source constraint is $C\subseteq\R^{n_\theta}$, and the
target constraint is $Q\subset\R^m$.
In practice, $C=\R^{n_\theta}$ (unconstrained weight space), so the
set-feasibility penalty
$H(\theta)=\tfrac{1}{2}\dist^2(\theta,C)$ vanishes identically;
we retain~$C$ in the formulation to preserve the structural
correspondence with the ABP algorithm.

Different MTL problems use different target network architectures
(see Section~\ref{sec:exp-mtl}); the hypernetwork architectures
described below are agnostic to the choice of target network.

\begin{table}[t]
\centering
\caption{Structural correspondence between the ABP algorithm (Algorithm~\ref{alg:main}) and
the hypernetwork training loss. The penalty terms and their roles
map directly between the two settings; the weight schedules are
discussed in Section~\ref{subsec:twophase}.}
\label{tab:correspondence}
\begin{tabular*}{\tblwidth}{@{\extracolsep{\fill}}lll@{}}
\toprule
\textbf{ABP algorithm} & \textbf{Hypernetwork training} & \textbf{Role} \\
\midrule
$x\in\R^n$ & $\theta=h(\bm{r};\phi)\in\R^{n_\theta}$ & Decision variable \\
$C\subset\R^n$ & $C\subseteq\R^{n_\theta}$ (weight space) & Source constraint \\
$\calF(x)$ & $\bm{\ell}(h(\bm{r};\phi))$ & Objective vector \\
$s(\calF(x),\bm{r})$ & $\calL_{\mathrm{obj}}$ & Chebyshev scalarization \\
$G(x)=\tfrac{1}{2}\dist^2(\calF(x),Q^+)$ & $\calL_Q$ & Image-feasibility penalty \\
$H(x)=\tfrac{1}{2}\dist^2(x,C)$ & $\tfrac{1}{2}\dist^2(\theta,C)$ & Set-feasibility penalty \\
\bottomrule
\end{tabular*}
\end{table}

\subsection{Hypernetwork architectures}\label{subsec:hypernetworks}

Both hypernetwork architectures generate the full parameter dictionary
$\bm{\theta}=\{\theta_j\}$ of a purely functional \emph{target
network}~$g(\cdot;\bm{\theta})$ that holds no learnable parameters of
its own; all weights are injected at run-time so that gradients flow
from the task losses back to the hypernetwork parameters~$\phi$.
For the MTL experiments, the target network is a Multi-LeNet
CNN~\citep{sener2018,navon2021} (Fig.~\ref{fig:lenet-target}),
comprising $|\bm{\theta}|=31{,}910$ parameters
(two conv--pool blocks, one hidden layer, two task-specific heads).
The hypernetwork size is $|\phi|\approx 8.66$\,M for both
architectures, following~\citep{tuan2024}.

\begin{figure}
  \centering
  \includegraphics[width=0.82\textwidth]{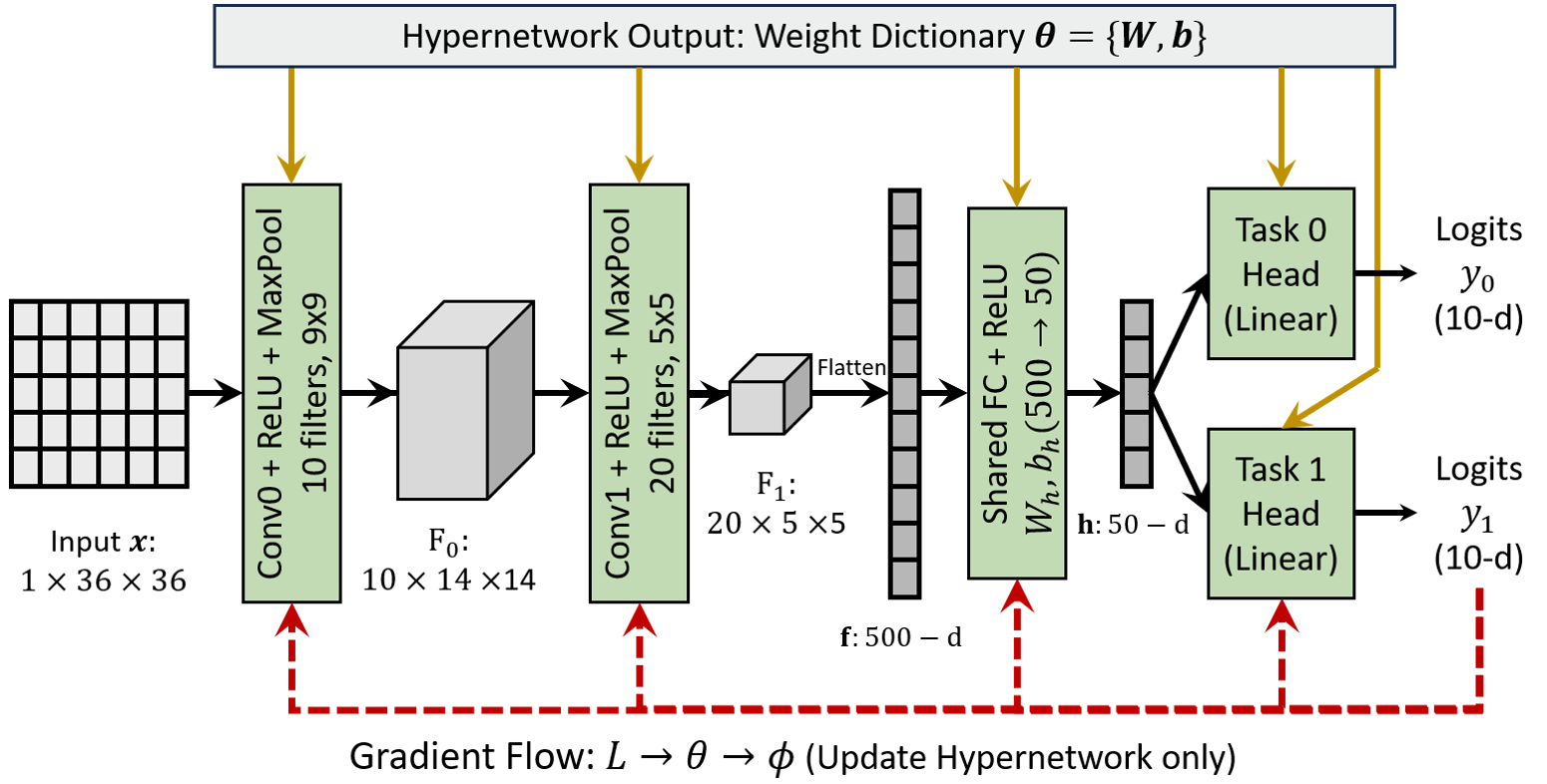}
  \rule{0.82\linewidth}{0.4pt}\\[4pt]
  \caption{Multi-LeNet target network for Multi-MNIST, Multi-Fashion,
  and Fashion\texttt{+}MNIST.
  All weights $\bm{\theta}$ are generated by the hypernetwork;
  no learnable parameters reside in the target network itself.
  The two task heads produce logits for the left-image task
  (Task~0) and right-image task (Task~1), respectively.}
  \label{fig:lenet-target}
\end{figure}

\subsubsection{Hyper-MLP}\label{subsec:hyper-mlp}

The MLP-based hypernetwork~\citep{navon2021} processes the preference
vector as a single input and maps it through a shared trunk to
per-parameter linear heads.
Given $\bm{r}\in\R^m$, the shared representation is:
\begin{equation}\label{eq:hyper-mlp}
  \bm{h}_{\mathrm{MLP}}(\bm{r})
  =\mathrm{ReLU}\!\bigl(\bm{W}_3\,\mathrm{ReLU}(\bm{W}_2\,
  \mathrm{ReLU}(\bm{W}_1\bm{r}+\bm{b}_1)+\bm{b}_2)+\bm{b}_3\bigr)
  \in\R^{d},
\end{equation}
where $d=256$ is the hidden dimension and
$\bm{W}_1\in\R^{d\times m}$,
$\bm{W}_2,\bm{W}_3\in\R^{d\times d}$
are learnable weight matrices with corresponding biases.
Each target parameter~$\theta_j$ is produced by a dedicated linear head:
\begin{equation}\label{eq:mlp-heads}
  \theta_j=\bm{A}_j\,\bm{h}_{\mathrm{MLP}}(\bm{r})+\bm{c}_j,
  \qquad
  \bm{A}_j\in\R^{n_j\times d},\;\bm{c}_j\in\R^{n_j}.
\end{equation}

\begin{figure}
  \centering
  \includegraphics[width=0.55\textwidth]{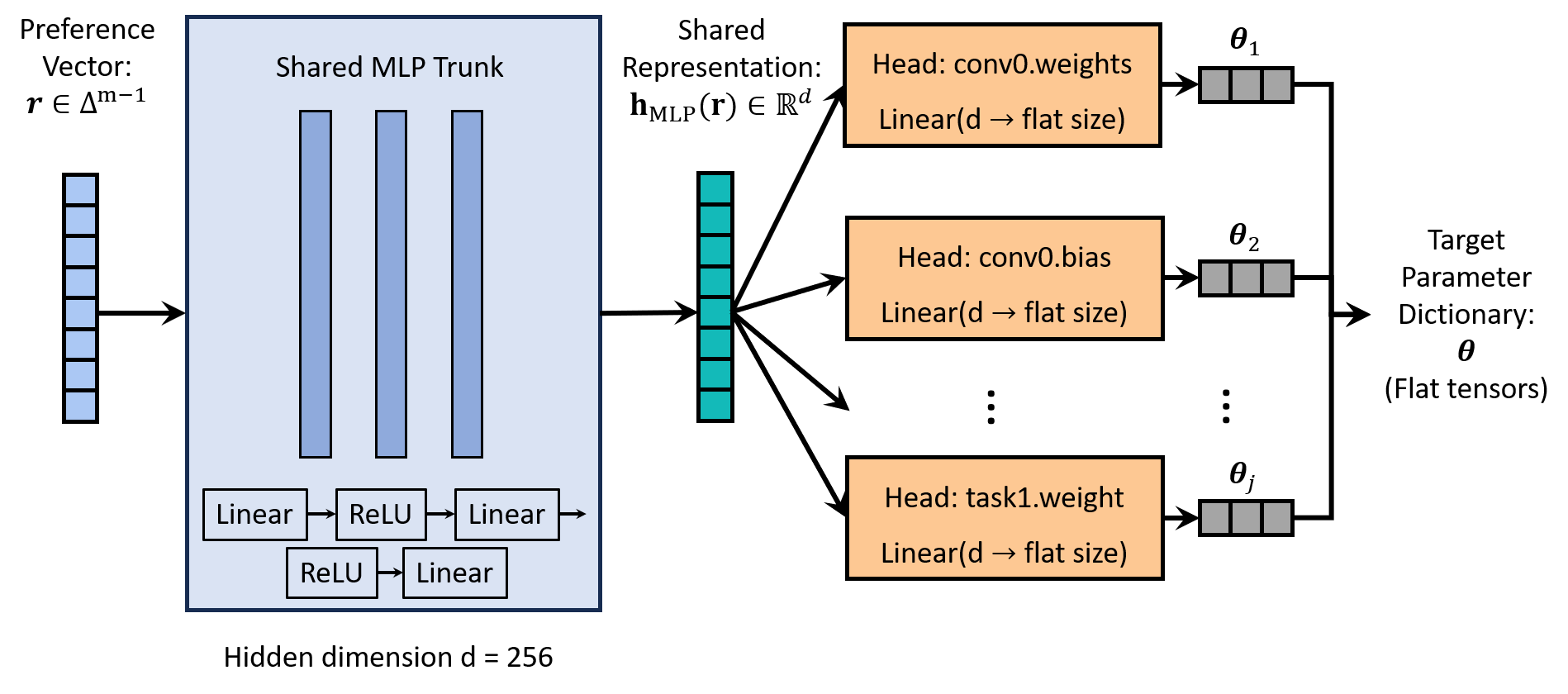}
  \caption{Hyper-MLP architecture.
  A three-layer shared MLP trunk maps $\bm{r}$ to
  $\bm{h}_{\mathrm{MLP}}(\bm{r})\in\R^d$;
  separate linear heads project this representation onto each
  target parameter tensor~$\theta_j$.}
  \label{fig:hyper-mlp}
\end{figure}

\subsubsection{HyperTrans}\label{subsec:hyper-trans}

The key limitation of Hyper-MLP is that $\bm{r}$ is processed as a
monolithic input, obscuring per-objective contributions.
This matters for the BSSP penalty structure because the image
feasibility residual~$\rho^k$ has a per-objective decomposition.
The Transformer-based hypernetwork~\citep{tuan2024} addresses this by
treating each component $r_i$ as a separate \emph{token}, enabling
explicit pairwise interaction through self-attention---a design
motivated by the universal approximation properties of
Transformers~\citep{yun2019}.

\textit{Per-component embedding.}
Each $r_i\in\R$ is mapped to a $d$-dimensional embedding:
\begin{equation}\label{eq:embedding}
  \bm{e}_i=\mathrm{act}(\bm{U}_i\,r_i+\bm{d}_i)\in\R^d,
  \qquad i=1,\ldots,m,
\end{equation}
where $\bm{U}_i\in\R^{d\times 1}$, $\bm{d}_i\in\R^d$.
The embeddings are stacked into
$\bm{E}=[\bm{e}_1,\ldots,\bm{e}_m]^T\in\R^{m\times d}$.

\textit{Transformer block.}
A single Transformer block with multi-head self-attention (MHSA)
and a feed-forward network (FFN), each with residual connections,
processes the sequence:
\begin{align}
  \bm{E}' &= \bm{E}+\mathrm{MHSA}(\bm{E},\bm{E},\bm{E}),
  \label{eq:mhsa}\\
  \bm{E}''&= \bm{E}'+\mathrm{FFN}(\bm{E}').
  \label{eq:ffn}
\end{align}
The attention scores capture the pairwise trade-off structure,
providing HyperTrans with an inductive bias for modelling the
coupling between objectives---precisely what the image-feasibility
direction $v^k=J_\calF(x^k)^T\rho^k$ aggregates in the ABP algorithm.

\textit{Aggregation.}
The processed tokens are aggregated by mean-pooling:
\begin{equation}\label{eq:meanpool}
  \bm{h}_{\mathrm{Trans}}(\bm{r})
  =\frac{1}{m}\sum_{i=1}^m\bm{E}''_i\in\R^d,
\end{equation}
followed by the same per-parameter linear heads as
in~\eqref{eq:mlp-heads}.

\begin{figure}
  \centering
  \includegraphics[width=0.72\textwidth]{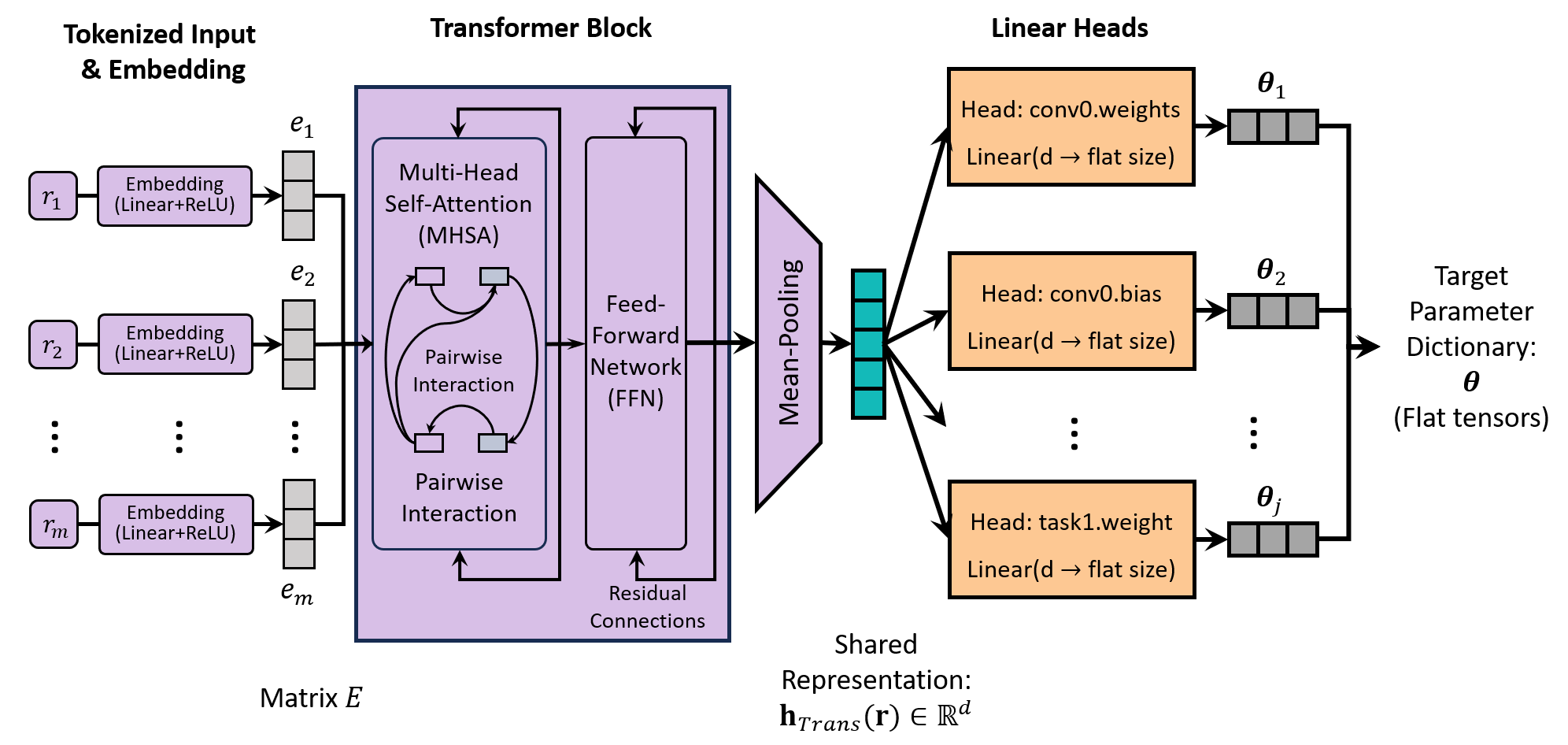}
  \caption{HyperTrans architecture: each $r_i$ is embedded into a
  $d$-dimensional token; a single Transformer block captures
  pairwise interactions; mean-pooling aggregates the tokens into
  a shared state projected via linear heads to produce~$\bm{\theta}$.}
  \label{fig:hyper-trans-arch}
\end{figure}

Table~\ref{tab:correspondence} makes the structural correspondence
between the ABP algorithm and the hypernetwork training loss explicit.

\subsection{Two-phase feasibility-first training}
\label{subsec:twophase}

\paragraph{Scope of the training strategy.}
Theorem~\ref{thm:main} guarantees full-sequence convergence of
the ABP algorithm to an optimal BSSP solution under
Assumption~\ref{ass:all}, which includes convexity of each $f_i$,
Robbins--Monro step-sizes, and bounded penalty weights. The hypernetwork
training described below operates outside this regime: objectives are
non-convex, gradients are stochastic mini-batches, and the phase
transition introduces changes in the penalty weights not covered by the
assumptions of Theorem~\ref{thm:main}. Accordingly, the two-phase strategy
should be understood as a \emph{convergence-theory-inspired heuristic}:
it embodies the feasibility-first principle derived from
the ABP algorithm, but its convergence properties in the neural
network setting remain an open question (see Section~\ref{sec:conclusion}).

The penalty structure of the ABP algorithm---where the
feasibility weights $\beta_k,\gamma_k$ are non-decreasing while the
optimality weight~$\alpha_k$ is non-increasing---suggests a natural
analogue for hypernetwork training: first drive the generated solutions
into the feasible region~$Q$ (Phase~1), then optimize the
Chebyshev trade-off within~$Q$ (Phase~2).

\textit{Phase~1 (feasibility warm-up).}
The training loss emphasizes the constraint penalty:
\begin{equation}\label{eq:phase1}
  \calL^{(1)}=\calL_Q+\varepsilon\,\calL_{\mathrm{obj}},
\end{equation}
where $\calL_{\mathrm{obj}}
=\max_i\{r_i(\ell_i(h(\bm{r};\phi))-a_i)\}$
is the Chebyshev scalarization loss and $\calL_Q$ penalizes the
squared violation of the constraint region~$Q$.
The dominant $\calL_Q$ term drives the hypernetwork outputs
into the feasible region, while the small coefficient $\varepsilon$
maintains diversity across preference vectors, preventing collapse to
a single feasible point.
This phase corresponds to the early iterations of the ABP algorithm,
where the large $\gamma_k$ on the image-feasibility direction~$v^k$
pulls the iterates towards~$Q^+$.

\textit{Phase~2 (trade-off optimization).}
Once feasibility is sufficiently established, the loss reverses the emphasis:
\begin{equation}\label{eq:phase2}
  \calL^{(2)}=\calL_{\mathrm{obj}}+\beta\,\calL_Q,
\end{equation}
where $\beta>0$ is an adaptive penalty coefficient.
This creates a dynamic equilibrium: when solutions drift outside~$Q$,
$\beta$ is increased to restore feasibility; when feasibility is
comfortable, $\beta$ is relaxed to give the Chebyshev objective
more freedom.
This mirrors the late iterations of the ABP algorithm,
where the optimality direction~$w^k$ (weighted by $\alpha_k$)
becomes the primary driver while the penalty terms prevent excessive
constraint violation.

\textit{Gradient flow.}
For each sampled preference $\bm{r}$, the gradient path is:
$\bm{r}\xrightarrow{h(\cdot;\phi)}\theta_{\bm{r}}
\xrightarrow{\text{target net}}\hat{y}
\xrightarrow{\text{losses}}\bm{\ell}
\xrightarrow{\text{scalarization}}\calL
\xrightarrow{\nabla_\phi}\phi$.
Only the hypernetwork parameters~$\phi$ are updated; the target
network has no learnable parameters of its own.
In both phases, the set-feasibility penalty $H\equiv 0$ is absent
from the training loss, as noted in
Section~\ref{subsec:mtl-moo}.

\subsection{Evaluation metrics}\label{subsec:metrics}

We employ three complementary metrics; the first two are standard,
while the third is introduced in this work.

\textit{Mean Euclidean Distance (MED)}~\citep{tuan2024}.
When ground-truth constrained Pareto solutions are available, the
MED measures approximation accuracy:
\begin{equation}\label{eq:MED}
  \mathrm{MED}(\mathcal{F}^*,\hat{\mathcal{F}})
  =\frac{1}{|\mathcal{F}^*|}
  \sum_{i=1}^{|\mathcal{F}^*|}
  \bigl\|f_i^*-\hat{f}_i\bigr\|_2.
\end{equation}
A lower MED indicates that predicted solutions lie closer to the
true constrained Pareto front.

\textit{Hypervolume (HV)}~\citep{zitzler1999}.
Given a finite set $\mathcal{P}=\{p_j\in\R^m\}$ of non-dominated
points and a reference point $\bar{y}\in\R^m_+$, the hypervolume is
\begin{equation}\label{eq:HV}
  \mathrm{HV}(\mathcal{P})
  =\mathrm{VOL}\!\Bigl(
  \bigcup_{p\in\mathcal{P},\,p\prec\bar{y}}
  \prod_{i=1}^m [p_i,\bar{y}_i]
  \Bigr).
\end{equation}
A higher HV indicates a better-spread and better-converged Pareto
front approximation.

\textit{Expected Feasible Hypervolume (EFHV).}
In constrained multi-objective settings, reporting HV and feasibility
rate separately makes it difficult to rank methods that trade off one
for the other.
Feasibility-weighted hypervolume concepts have appeared in
constrained Bayesian optimization, where the Constrained Expected
Hypervolume Improvement (CEHVI) multiplies the Expected Hypervolume
Improvement by the probability of
feasibility~\citep{emmerich2006,gelbart2014}.
However, CEHVI is a \emph{point-wise acquisition function}
designed to guide sequential search, not a set-level evaluation
metric.
In evolutionary constrained multi-objective optimization, the
standard practice is to compute HV on the feasible subset
alone~\citep{zitzler1999}, which ignores the feasibility rate entirely.
Existing CPFL methods~\citep{navon2021,tuan2024,lin2022}
report only unconstrained HV or MED and do not provide
a combined quality--feasibility metric.

To address this gap, we define the \emph{Expected Feasible Hypervolume}
as a set-level evaluation metric that multiplicatively combines both
dimensions.
Let $\mathcal{P}=\{p_j\}_{j=1}^K$ be the solutions produced for $K$
preference rays, and let $\pi=|\{j:p_j\in Q\}|/K$ denote the
ray-level feasibility rate.
The EFHV is defined as
\begin{equation}\label{eq:EFHV}
  \mathrm{EFHV}(\mathcal{P})
  = \pi \cdot \mathrm{HV}\!\bigl(\mathcal{P}_{\mathrm{feas}}\bigr),
\end{equation}
where $\mathcal{P}_{\mathrm{feas}}=\{p_j\in Q\}$ is the feasible
subset.
The multiplicative form ensures that a method must excel on
\emph{both} dimensions to achieve a high score: even if the feasible
subset has high HV, a low feasibility rate~$\pi$ suppresses the
overall EFHV proportionally.

\section{Internal validation: multi-objective optimization benchmarks}
\label{sec:exp-mop}

\subsection{Benchmark problems}\label{subsec:mop-benchmarks}

We validate the ABP algorithm (Algorithm~\ref{alg:main}) and its
hypernetwork-based adaptation (ABP-HyperTrans) on five benchmark
problems from the literature, summarized in Table~\ref{tab:mop-bench}.
The Pareto fronts are normalized to $[0,1]^m$ following the
convention of~\citep{tuan2024}.
The constraint set~$Q$ is chosen as a sphere
$\overline{B}(c,R):=\{y\in\R^m:\norm{y-c}\le R\}$ in all cases.
Problem names (CVX1--CVX3, ZDT1--ZDT2) and their decision sets
$C$ follow the cited sources~\citep{tuan2024,binh1997,thang2020,zitzler2000};
the prefix ``CVX'' designates problems whose \emph{objective functions}
are convex and does not imply convexity of the decision set~$C$.
The target regions~$Q$ listed in Table~\ref{tab:mop-bench} are introduced
in this paper.

\begin{table}[t]
\centering
\caption{MOP benchmark specifications.
  $n$: number of decision variables; $m$: number of objectives;
  $C$: decision space; $Q$: target region in objective space;
  $\overline{B}(c,R)$ denotes the closed ball of radius~$R$
  centered at~$c$.}
\label{tab:mop-bench}
\begin{tabular*}{\tblwidth}{@{\extracolsep{\fill}}llcccl@{}}
\toprule
Problem & Source & $n$ & $m$ & $Q$ & Type \\
\midrule
CVX1 & \citep{tuan2024}        & 1  & 2 & $\overline{B}((0.4,0.4),\,0.2)$ & Convex \\
CVX2 & \citep{binh1997}         & 2  & 2 & $\overline{B}((0.4,0.4),\,0.2)$ & Convex \\
CVX3 & \citep{thang2020}        & 3  & 3 & $\overline{B}((0.5,0.5,0.5),\,0.2)$ & Convex \\
ZDT1 & \citep{zitzler2000}      & 30 & 2 & $\overline{B}((0.4,0.4),\,0.2)$ & Non-convex \\
ZDT2 & \citep{zitzler2000}      & 30 & 2 & $\overline{B}((0.4,0.5),\,0.4)$ & Non-convex \\
\bottomrule
\end{tabular*}
\end{table}

\subsubsection{Convex benchmarks}

\textbf{CVX1}~\citep{tuan2024}.
A two-objective problem in one variable:
\begin{equation}\label{eq:cvx1}\tag{CVX1}
  \min_{x}\;\bigl\{x,\;(x-1)^2\bigr\}
  \quad\text{s.t.}\quad 0\le x\le 1.
\end{equation}

\textbf{CVX2}~\citep{binh1997}.
A two-objective problem in two variables (Binh--Korn problem):
\begin{equation}\label{eq:cvx2}\tag{CVX2}
  \min_{x_1,x_2}\;\{f_1,\,f_2\}
  \quad\text{s.t.}\quad x_i\in[0,5],\;i=1,2,
\end{equation}
where
\[
  f_1 = \frac{x_1^2+x_2^2}{50},
  \qquad
  f_2 = \frac{(x_1-5)^2+(x_2-5)^2}{50}.
\]

\textbf{CVX3}~\citep{thang2020}.
A three-objective problem on the unit sphere in~$\R^3$.
While the objective functions are convex (as indicated by the prefix),
the spherical decision set~$C$ is non-convex; consequently, the
convergence guarantees of Theorems~\ref{thm:main} and~\ref{thm:approx}
do not apply, and the ABP algorithm is applied to this benchmark
as an effective heuristic.
\begin{equation}\label{eq:cvx3}\tag{CVX3}
  \min_{x_1,x_2,x_3}\;\{f_1,\,f_2,\,f_3\}
  \quad\text{s.t.}\quad
  x_1^2+x_2^2+x_3^2=1,\;
  x_i\in[0,1],\;i=1,2,3,
\end{equation}
where
\begin{align*}
  f_1 &= \frac{x_1^2+x_2^2+x_3^2+x_2-12x_3+12}{14},\\
  f_2 &= \frac{x_1^2+x_2^2+x_3^2+8x_1-44.8x_2+8x_3+44}{57},\\
  f_3 &= \frac{x_1^2+x_2^2+x_3^2-44.8x_1+8x_2+8x_3+43.7}{56}.
\end{align*}

\subsubsection{Non-convex benchmarks}

\textbf{ZDT1}~\citep{zitzler2000}.
A classical bi-objective benchmark with a convex Pareto front:
\begin{equation}\label{eq:zdt1}\tag{ZDT1}
  f_1(\mathbf{x}) = x_1,
  \qquad
  f_2(\mathbf{x}) = g(\mathbf{x})\Bigl[1-\sqrt{f_1(\mathbf{x})/g(\mathbf{x})}\,\Bigr],
\end{equation}
where $g(\mathbf{x})=1+\tfrac{9}{n-1}\sum_{i=2}^n x_i$ and
$0\le x_i\le 1$ for $i=1,\ldots,n$ (with $n=30$).
The analytical Pareto front satisfies
$f_2=1-\sqrt{f_1}$ for $f_1\in[0,1]$.

\textbf{ZDT2}~\citep{zitzler2000}.
A classical bi-objective benchmark with a non-convex Pareto front:
\begin{equation}\label{eq:zdt2}\tag{ZDT2}
  f_1(\mathbf{x}) = x_1,
  \qquad
  f_2(\mathbf{x}) = g(\mathbf{x})\Bigl[1-\bigl(f_1(\mathbf{x})/g(\mathbf{x})\bigr)^2\Bigr],
\end{equation}
where $g(\mathbf{x})$ is as in~\eqref{eq:zdt1} and $n=30$.
The analytical Pareto front satisfies $f_2=1-f_1^2$ for
$f_1\in[0,1]$.

Since all five MOP benchmarks admit known ground-truth constrained
Pareto solutions, the primary validation metric is the MED
(Section~\ref{subsec:metrics}); HV and EFHV are also reported.

\subsection{Experimental setup}\label{subsec:mop-setup}

\textit{Two-phase execution.}
For each benchmark, the algorithm proceeds in two phases:
Phase~1 uses the CQ method~\eqref{eq:CQ} to find an initial
feasible point $x_{\mathrm{feas}}\in \mathcal{S}$;
Phase~2 runs the ABP algorithm from~$x_{\mathrm{feas}}$
for $K$ uniformly spaced preference vectors on the simplex.

\textit{Ground truth.}
For CVX1--2, the Chebyshev-optimal points under the constraint
$\calF(x)\in Q^+$ are computed by an exact solver (SLSQP with
analytical Jacobian).
For CVX3, grid search on the unit sphere combined with $Q^+$
filtering and Chebyshev optimization is used.
For ZDT1 and ZDT2, the analytical Pareto front is sampled at
5000 uniformly spaced points and filtered by $Q^+$; the resulting
constrained front serves as the reference for MED computation.

\paragraph{Step-size schedule.}
All runs of the ABP algorithm use the step-size schedule
$\lambda_k = 1/(k+1)^\nu$ for a problem-dependent exponent
$\nu \in (1/2,\,1]$, which satisfies the Robbins--Monro
condition~\eqref{A5} for any such choice of~$\nu$.

\paragraph{Verification of Assumptions~(A7)--(A8).}
For CVX1 and CVX2 the decision set $C$ is a compact box, so
Assumption~(A8) holds automatically.
The zero-gap condition~\eqref{A7} is verified numerically
for all 50 preference rays via SLSQP; the gap $\sigma$ equals zero
on 39/50 rays (CVX1) and 44/50 rays (CVX2), with $\sigma_{\max}\le 0.081$.
Full details are given in Appendix~\ref{app:sigma-gap}.

\subsection{Results}\label{subsec:mop-results}

The purpose of this section is \emph{internal validation}: we verify
that (1)~the ABP algorithm (Algorithm~\ref{alg:main}) faithfully
approximates the known ground-truth constrained Pareto solutions,
and (2)~the hypernetwork-based adaptation (ABP-HyperTrans) matches
or improves upon the numerical solver while offering orders-of-magnitude
faster inference.
Table~\ref{tab:mop-results} reports MED ($\downarrow$) and
inference time for all five benchmarks at 50 preference rays.

\begin{table}[t]
\centering
\caption{Internal validation on MOP benchmarks (50 preference rays).
  MED ($\downarrow$) measures distance to the ground-truth constrained
  Pareto front.
  ABP Solver reports a single deterministic run;
  ABP-HyperTrans reports mean\,$\pm$\,std over 10 seeds.
  \textbf{Bold}: best MED per problem.}
\label{tab:mop-results}
\begin{tabular*}{\tblwidth}{@{\extracolsep{\fill}}l cc cc l@{}}
\toprule
& \multicolumn{2}{c}{ABP Solver} &
  \multicolumn{2}{c}{ABP-HyperTrans} & \\
\cmidrule(lr){2-3}\cmidrule(lr){4-5}
Problem &
  MED & Time (ms) &
  MED (mean$\pm$std) & Time (ms) & Theory \\
\midrule
CVX1 & 0.005126 & 1451
  & $\mathbf{0.004742} \pm 0.001884$ & 16.4
  & \checkmark \\
CVX2 & 0.002773 & 228
  & $\mathbf{0.001895} \pm 0.000246$ & 16.6
  & \checkmark \\
CVX3 & 0.007237 & 616
  & $\mathbf{0.007156} \pm 0.000614$ & 28.0
  & heuristic \\
ZDT1 & 0.009889 & 1802
  & $\mathbf{0.006334} \pm 0.001973$ & 17.0
  & heuristic \\
ZDT2 & 0.010078 & 12835
  & $\mathbf{0.008518} \pm 0.001063$ & 16.7
  & heuristic \\
\bottomrule
\end{tabular*}
\end{table}

\textbf{Correctness of the ABP algorithm.}
The ABP Solver achieves low MED on all five benchmarks,
confirming that the proposed penalty structure successfully
balances optimality and feasibility.
On CVX1 and CVX2, convergence is guaranteed by
Theorem~\ref{thm:main} on the majority of preference rays
(Appendix~\ref{app:sigma-gap}); on the remaining rays,
Theorem~\ref{thm:approx} provides approximate convergence bounds.
On CVX3 and the non-convex ZDT problems, the algorithm operates as an
effective heuristic (Remark~\ref{rem:nonconvex}).

\textbf{ABP-HyperTrans matches or improves the solver.}
ABP-HyperTrans achieves lower mean MED than the ABP Solver on
all five benchmarks, demonstrating that the hypernetwork-based
adaptation faithfully---and in some cases more accurately---approximates
the ground-truth constrained Pareto front.
The improvement is most pronounced on ZDT1 and ZDT2,
where the self-attention mechanism captures the higher-dimensional
preference-to-solution mapping more effectively.

\textbf{Computational efficiency.}
The inference time of ABP-HyperTrans (16--28\,ms) is two to three
orders of magnitude faster than the ABP Solver
(228--12{,}835\,ms), confirming the practical advantage of
the learned model for real-time preference control.

Figures~\ref{fig:mop-cvx} and~\ref{fig:mop-zdt} display the
approximated constrained Pareto fronts produced by ABP-HyperTrans
on four benchmarks with 50 preference rays.
In each panel, the shaded region represents the extended downward
hull~$Q^+$, gold dots mark the ground-truth constrained solutions,
and green markers show the solutions returned by ABP-HyperTrans.
The MED and HV values annotated in each panel correspond to a
single representative seed; quantitative comparisons across all seeds
should be drawn from Table~\ref{tab:mop-results}.

\begin{figure}
  \centering
  \includegraphics[width=\textwidth]{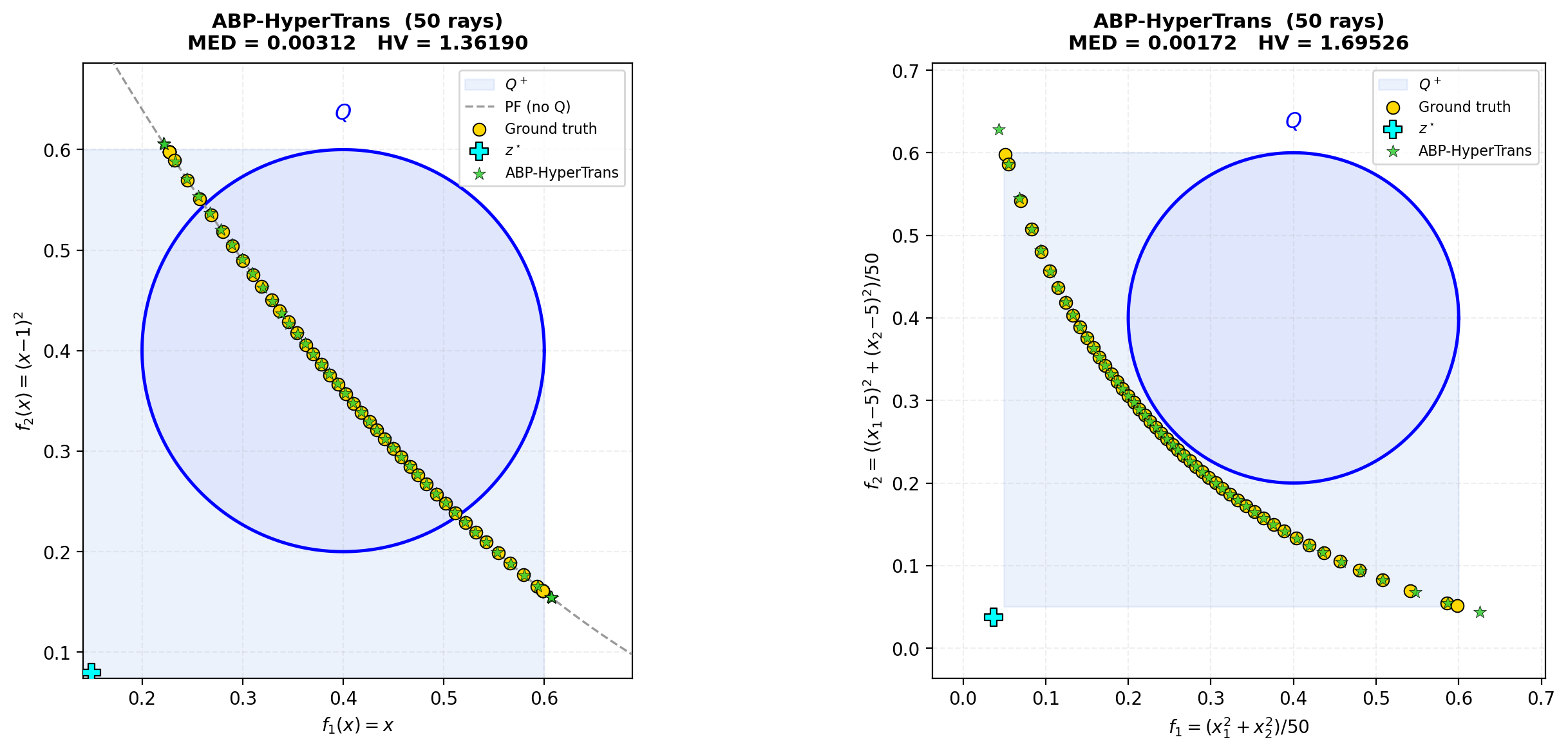}
  \caption{Constrained Pareto front approximation on two convex benchmarks
  with 50 rays.
  Left: CVX1 ($\dim x=1$).
  Right: CVX2 ($\dim x=2$, Binh \& Korn).
  ABP-HyperTrans places its output points tightly along the
  ground-truth constrained front inside~$Q^+$.}
  \label{fig:mop-cvx}
\end{figure}

In Fig.~\ref{fig:mop-cvx}, ABP-HyperTrans achieves high-fidelity
approximation on both convex benchmarks.
On CVX1 (left), the one-dimensional decision variable yields a
clean mapping from preference rays to the constrained front,
and virtually all solutions lie inside~$Q^+$.
On CVX2 (right), the two-variable objective map produces a
well-structured arc and ABP-HyperTrans tracks the ground-truth
with low MED, confirming reliable satisfaction of the spherical
constraint.

\begin{figure}
  \centering
  \includegraphics[width=\textwidth]{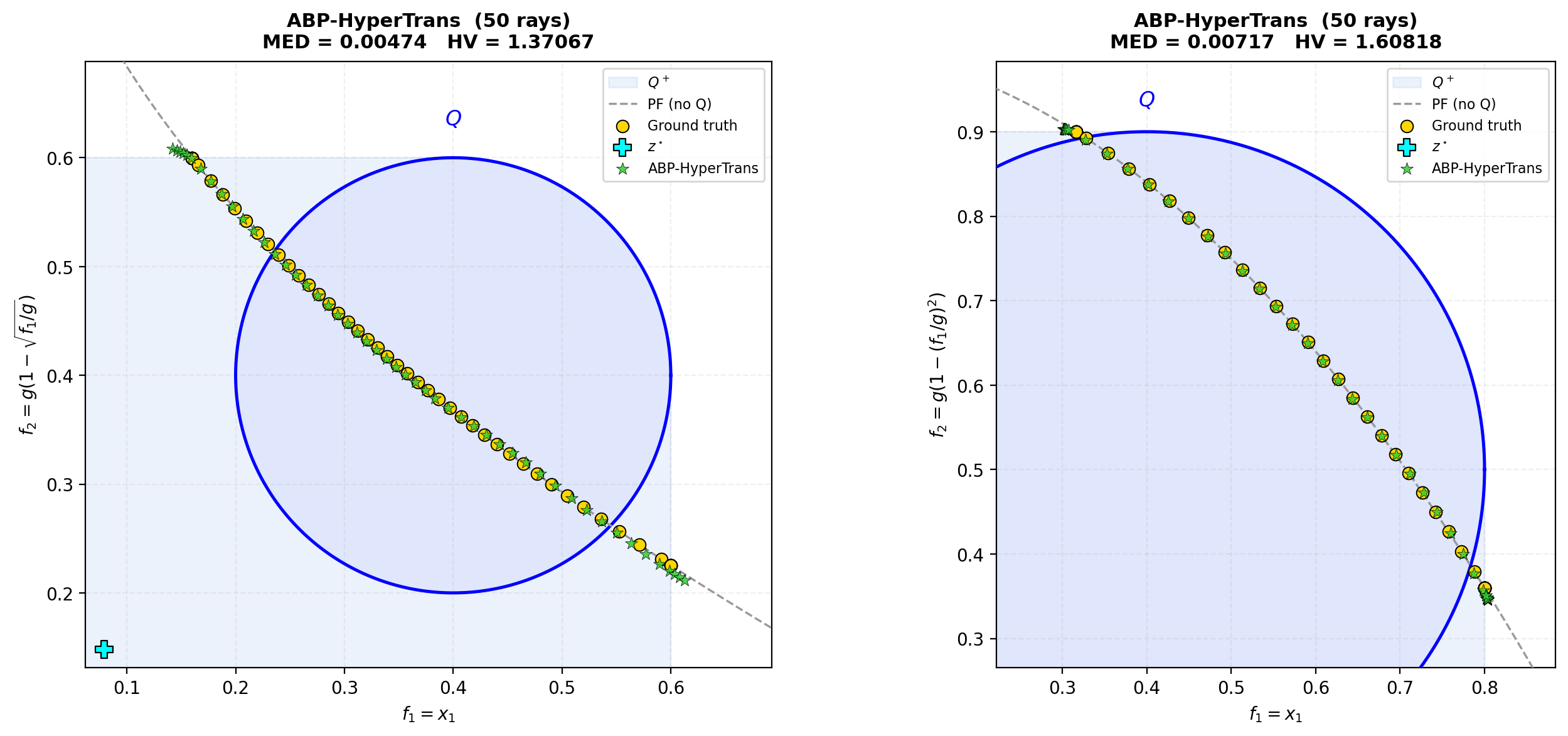}
  \caption{Constrained Pareto front approximation on two non-convex
  ZDT benchmarks with 50 rays.
  Left: ZDT1 ($\dim x=30$, convex front $f_2=1-\sqrt{f_1}$).
  Right: ZDT2 ($\dim x=30$, concave front $f_2=1-f_1^2$).
  Layout as in Fig.~\ref{fig:mop-cvx}.}
  \label{fig:mop-zdt}
\end{figure}

In Fig.~\ref{fig:mop-zdt}, the 30-dimensional ZDT benchmarks
present a substantially harder learning task.
On ZDT1 (left), the convex unconstrained front
($f_2=1-\sqrt{f_1}$) is clipped by~$Q$, and ABP-HyperTrans
successfully concentrates its solutions along the feasible arc.
On ZDT2 (right), the concave shape of the unconstrained front
($f_2=1-f_1^2$) interacts with the spherical constraint in a
non-trivial way; ABP-HyperTrans nonetheless tracks the ground-truth
constrained front with only minor deviations near the boundary
of~$Q$.

We use $K=50$ preference rays for all MOP experiments, following
the convention of~\citep{tuan2024}.
Appendix~\ref{app:ray-ablation} reports an ablation confirming that
approximation quality improves monotonically with ray count.

\section{Experiments: multi-task learning}\label{sec:exp-mtl}

\subsection{Experimental setup}
\label{subsec:mtl-setup}

We evaluate the two-phase training strategy
(Section~\ref{subsec:twophase}) on three image classification
benchmarks from multi-task learning.

\textbf{Datasets.}
These three benchmark datasets~\citep{sabour2017,xiao2017,lin2019}
are commonly used in multi-task learning.
Each image is $1\times 36\times 36$ (single-channel) and contains two
overlapping items (digits or fashion articles); the two tasks are to
classify the top-left item (Task~L) and the bottom-right item
(Task~R), each into 10 classes.
Each dataset has 120\,000 training and 20\,000 test examples;
10\% of the training set is held out for validation.

\textbf{Target network and hypernetwork.}
The target network is the Multi-LeNet CNN described in
Section~\ref{subsec:hypernetworks} (Fig.~\ref{fig:lenet-target}),
with $|\bm{\theta}|=31{,}910$ parameters per preference vector and
hypernetwork size $|\phi|\approx 8.66$\,M ($d=256$),
following~\citep{tuan2024}.

\textbf{Training details.}
We follow the ten-fold cross-validation protocol and general training
configurations of~\citep{tuan2024} for a fair comparison.
All experiments use the Adam optimizer with a learning rate of
$5\times 10^{-4}$ and batch size~256.
For evaluation, we employ 25 evenly spaced preference vectors on
the simplex with a hypervolume reference point of $(2,2)$.
Our proposed two-phase strategy transitions from Phase~1 to Phase~2 when the
validation feasibility rate exceeds 95\% (or after a maximum of 80 epochs),
at which point the adaptive penalty coefficient $\beta$ is introduced.

\textbf{Constraint regions.}
Table~\ref{tab:mtl-Q} specifies the constraint regions~$Q$ used for
each dataset.
The Box constraint enforces per-task loss upper bounds
$Q=\{\bm{\ell}\in\R^2:\bm{\ell}\le\bm{b}\}$;
the Sphere constraint enforces
$Q=\overline{B}(\bm{c},R)=\{\bm{\ell}:\|\bm{\ell}-\bm{c}\|\le R\}$.
These regions were selected so that the constraint boundary intersects
the unconstrained Pareto front near the centre of the achievable
loss range on each dataset.

\begin{table}[t]
\centering
\small
\caption{Constraint regions~$Q$ for the MTL experiments.
  Box: $Q=\{\bm{\ell}:\bm{\ell}\le\bm{b}\}$;
  Sphere: $Q=\overline{B}(\bm{c},R)$.}
\label{tab:mtl-Q}
\setlength{\tabcolsep}{4pt}
\begin{tabular}{@{}lll@{}}
\toprule
Dataset & Box $\bm{b}$ & Sphere $(\bm{c},\,R)$ \\
\midrule
Multi-MNIST   & $(0.30,\,0.40)$ & $\bm{c}=(0.25,\,0.35),\;R=0.05$ \\
Multi-Fashion & $(0.75,\,0.75)$ & $\bm{c}=(0.65,\,0.65),\;R=0.10$ \\
Fashion+MNIST & $(0.60,\,0.60)$ & $\bm{c}=(0.50,\,0.50),\;R=0.10$ \\
\bottomrule
\end{tabular}
\end{table}

\subsection{Results}\label{subsec:mtl-results}

Tables~\ref{tab:mtl-mnist}--\ref{tab:mtl-fashmnist} report
HV (computed on the feasible subset $\mathcal{P}_{\mathrm{feas}}$ only),
$\pi$ (feasibility rate, as defined in Section~\ref{subsec:metrics}),
and EFHV under Box ($Q=[\bm{a},\bm{b}]$) and
Sphere ($Q=\overline{B}(\bm{c},R)$) constraints, averaged over
10 folds.
Since the baselines of \citet{tuan2023,tuan2024} do not enforce
constraints during training, we evaluate their test solutions
post hoc against each constraint region~$Q$ to compute $\pi$
and EFHV under equal conditions.

\begin{table}[t]
\centering
\small
\caption{Multi-MNIST results (mean $\pm$ std, 10-fold CV, HV ref.\ point $(2,2)$).
  HV is computed on the feasible subset $\mathcal{P}_{\mathrm{feas}}$ only.
  \textbf{Bold}: best $\pi$ and EFHV per constraint.}
\label{tab:mtl-mnist}
\setlength{\tabcolsep}{3.5pt}
\begin{tabular}{@{}llccc@{}}
\toprule
Method & Constr.\ & HV ($\uparrow$) & $\pi$\,(\%) ($\uparrow$) & EFHV ($\uparrow$) \\
\midrule
Hyper-MLP~\citep{tuan2023}
  & Box    & $2.974 \pm 0.014$ & $43.0 \pm 10.3$ & $1.282 \pm 0.312$ \\
HyperTrans~\citep{tuan2024}
  & Box    & $2.955 \pm 0.015$ & $36.5 \pm 3.7$  & $1.078 \pm 0.108$ \\
ABP-HyperMLP (ours)
  & Box    & $2.958 \pm 0.013$ & $\mathbf{100.0 \pm 0.0}$  & $\mathbf{2.958 \pm 0.013}$ \\
ABP-HyperTrans (ours)
  & Box    & $2.959 \pm 0.007$ & $92.5 \pm 11.0$ & $2.737 \pm 0.327$ \\
\midrule
Hyper-MLP~\citep{tuan2023}
  & Sphere & $2.974 \pm 0.014$ & $43.0 \pm 10.3$ & $1.282 \pm 0.312$ \\
HyperTrans~\citep{tuan2024}
  & Sphere & $2.955 \pm 0.015$ & $35.5 \pm 2.4$  & $1.049 \pm 0.070$ \\
ABP-HyperMLP (ours)
  & Sphere & $2.968 \pm 0.010$ & $\mathbf{100.0 \pm 0.0}$ & $\mathbf{2.968 \pm 0.010}$ \\
ABP-HyperTrans (ours)
  & Sphere & $2.955 \pm 0.008$ & $99.5 \pm 1.3$  & $2.940 \pm 0.036$ \\
\bottomrule
\end{tabular}
\end{table}

\begin{table}[t]
\centering
\small
\caption{Multi-Fashion results (same protocol as Table~\ref{tab:mtl-mnist}).}
\label{tab:mtl-fashion}
\setlength{\tabcolsep}{3.5pt}
\begin{tabular}{@{}llccc@{}}
\toprule
Method & Constr.\ & HV ($\uparrow$) & $\pi$\,(\%) ($\uparrow$) & EFHV ($\uparrow$) \\
\midrule
Hyper-MLP~\citep{tuan2023}
  & Box    & $2.337 \pm 0.016$ & $42.9 \pm 2.8$  & $1.002 \pm 0.069$ \\
HyperTrans~\citep{tuan2024}
  & Box    & $2.201 \pm 0.028$ & $36.6 \pm 4.5$  & $0.805 \pm 0.110$ \\
ABP-HyperMLP (ours)
  & Box    & $2.316 \pm 0.015$ & $\mathbf{94.9 \pm 8.2}$  & $\mathbf{2.191 \pm 0.203}$ \\
ABP-HyperTrans (ours)
  & Box    & $2.196 \pm 0.016$ & $90.3 \pm 12.3$ & $1.984 \pm 0.280$ \\
\midrule
Hyper-MLP~\citep{tuan2023}
  & Sphere & $2.337 \pm 0.016$ & $42.9 \pm 2.8$  & $1.002 \pm 0.069$ \\
HyperTrans~\citep{tuan2024}
  & Sphere & $2.201 \pm 0.028$ & $36.6 \pm 4.5$  & $0.805 \pm 0.110$ \\
ABP-HyperMLP (ours)
  & Sphere & $2.322 \pm 0.012$ & $\mathbf{100.0 \pm 0.0}$ & $\mathbf{2.322 \pm 0.012}$ \\
ABP-HyperTrans (ours)
  & Sphere & $2.195 \pm 0.020$ & $89.1 \pm 10.4$ & $1.957 \pm 0.230$ \\
\bottomrule
\end{tabular}
\end{table}

\begin{table}[t]
\centering
\small
\caption{Fashion+MNIST results (same protocol as Table~\ref{tab:mtl-mnist}).}
\label{tab:mtl-fashmnist}
\setlength{\tabcolsep}{3.5pt}
\begin{tabular}{@{}llccc@{}}
\toprule
Method & Constr.\ & HV ($\uparrow$) & $\pi$\,(\%) ($\uparrow$) & EFHV ($\uparrow$) \\
\midrule
Hyper-MLP~\citep{tuan2023}
  & Box    & $2.928 \pm 0.012$ & $49.1 \pm 1.8$  & $1.439 \pm 0.056$ \\
HyperTrans~\citep{tuan2024}
  & Box    & $2.847 \pm 0.021$ & $41.7 \pm 2.0$  & $1.188 \pm 0.062$ \\
ABP-HyperMLP (ours)
  & Box    & $2.928 \pm 0.020$ & $87.4 \pm 4.0$  & $2.561 \pm 0.132$ \\
ABP-HyperTrans (ours)
  & Box    & $2.844 \pm 0.033$ & $\mathbf{90.9 \pm 5.5}$  & $\mathbf{2.584 \pm 0.161}$ \\
\midrule
Hyper-MLP~\citep{tuan2023}
  & Sphere & $2.928 \pm 0.012$ & $49.1 \pm 1.8$  & $1.439 \pm 0.056$ \\
HyperTrans~\citep{tuan2024}
  & Sphere & $2.847 \pm 0.021$ & $41.7 \pm 2.0$  & $1.188 \pm 0.062$ \\
ABP-HyperMLP (ours)
  & Sphere & $2.935 \pm 0.021$ & $88.6 \pm 3.3$  & $\mathbf{2.606 \pm 0.111}$ \\
ABP-HyperTrans (ours)
  & Sphere & $2.867 \pm 0.030$ & $\mathbf{91.4 \pm 4.0}$  & $2.550 \pm 0.130$ \\
\bottomrule
\end{tabular}
\end{table}

Several consistent patterns emerge across all three datasets.

\textbf{Constrained vs.\ unconstrained.}
The EFHV metric reveals the decisive advantage of constrained
training.
Although the unconstrained baselines achieve comparable or slightly
higher HV on their feasible subsets, their $\pi$ is consistently
low (36--49\%), resulting in EFHV values roughly 50--60\% below
ours.
For instance, on Multi-MNIST (Box), the baseline Hyper-MLP achieves
HV~$=2.974$ but only $\pi=43.0\%$ (EFHV~$=1.282$), whereas
ABP-HyperMLP reaches $\pi=100.0\%$ with HV~$=2.958$
(EFHV~$=2.958$, a $2.3\times$ improvement).
This pattern is consistent across all 12 configurations
(2~architectures $\times$ 2~constraints $\times$ 3~datasets).

\textbf{Architecture and constraint geometry.}
ABP-HyperMLP consistently outperforms ABP-HyperTrans in HV,
suggesting that the lighter MLP hypernetwork is better suited
to the two-task MTL setting.
The Sphere constraint generally yields higher $\pi$ than Box,
likely due to its smoother projection geometry.

\textbf{Qualitative visualization.}
Figures~\ref{fig:mtl-mnist-comp}--\ref{fig:mtl-fashmnist-comp}
compare the loss-space Pareto fronts of all four methods for a
representative fold under Box (left) and Sphere (right) constraints.
In each panel, the shaded region indicates the feasible set~$Q$
(or~$Q^+$ for the Sphere case), and the legend reports the
ray-level feasibility rate for that fold.
The unconstrained baselines (grey circles, yellow diamonds) spread
well beyond~$Q$, achieving only 36--64\% feasibility.
In contrast, ABP-trained solutions (red squares, green triangles)
concentrate inside~$Q$, reaching 76--100\% feasibility while
maintaining comparable spread along the Pareto front.
Quantitative comparison across all folds should rely on the 10-fold
averages in Tables~\ref{tab:mtl-mnist}--\ref{tab:mtl-fashmnist};
additional per-architecture visualizations under individual constraint
types are provided in Appendix~\ref{app:constraint-viz}.

\begin{figure}
  \centering
  \includegraphics[width=\textwidth]{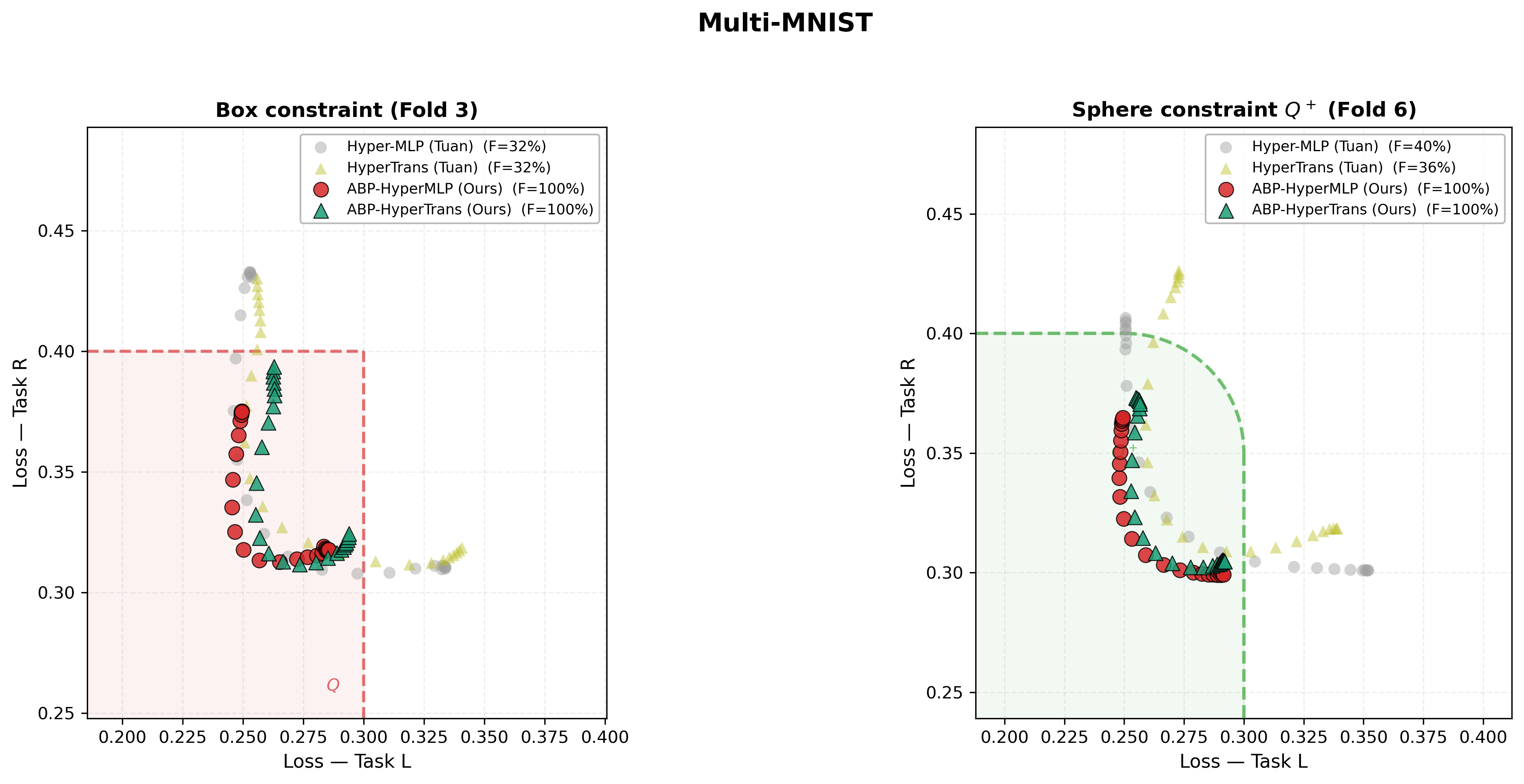}
  \caption{Multi-MNIST: Pareto fronts of all four methods under Box
  (left) and Sphere (right) constraints for a representative fold.
  Shaded regions: constraint set~$Q$ / $Q^+$;
  legend shows ray-level feasibility.}
  \label{fig:mtl-mnist-comp}
\end{figure}

\begin{figure}
  \centering
  \includegraphics[width=\textwidth]{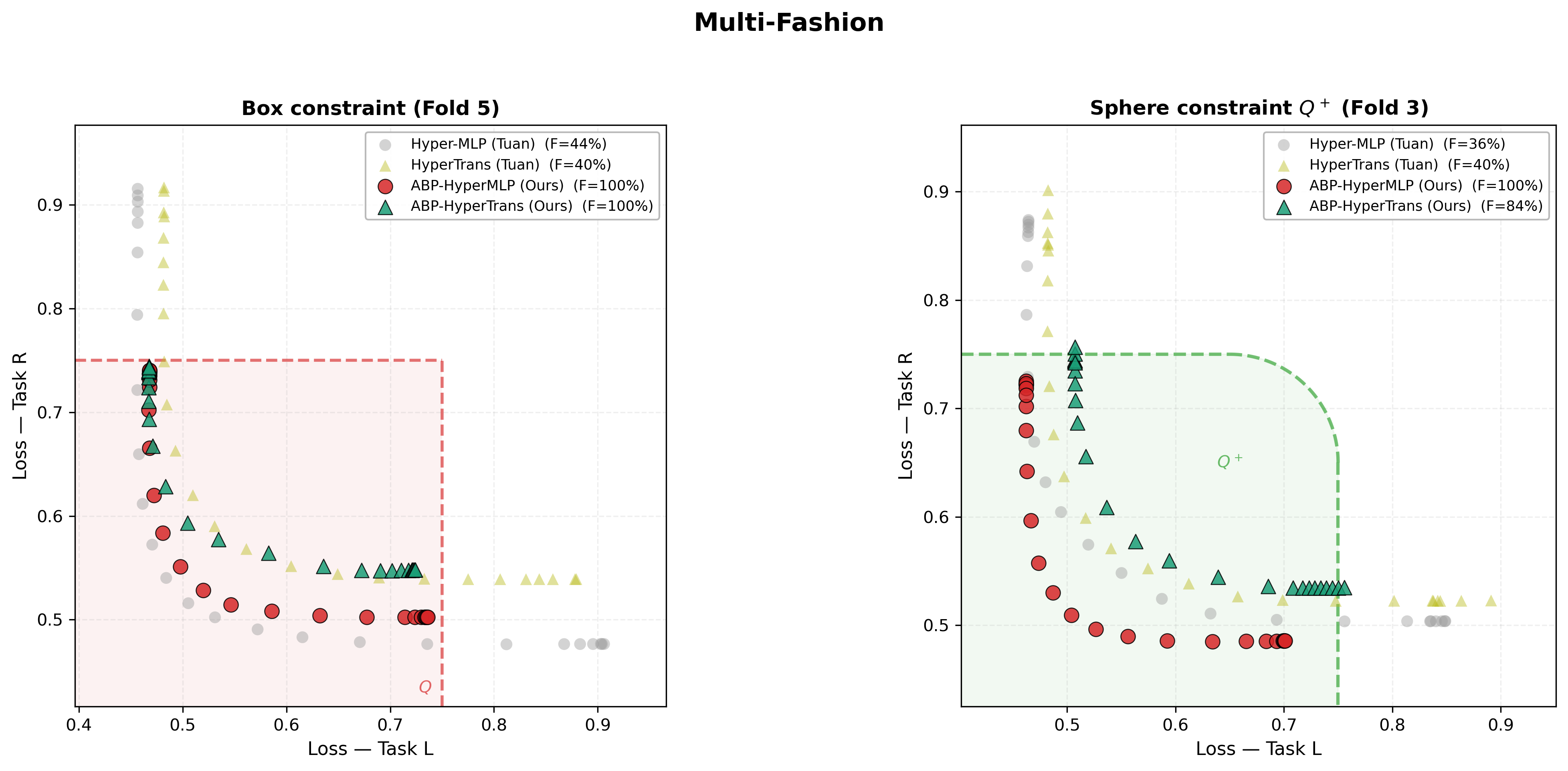}
  \caption{Multi-Fashion: Pareto fronts of all four methods.
  Layout as in Fig.~\ref{fig:mtl-mnist-comp}.}
  \label{fig:mtl-fashion-comp}
\end{figure}

\begin{figure}
  \centering
  \includegraphics[width=\textwidth]{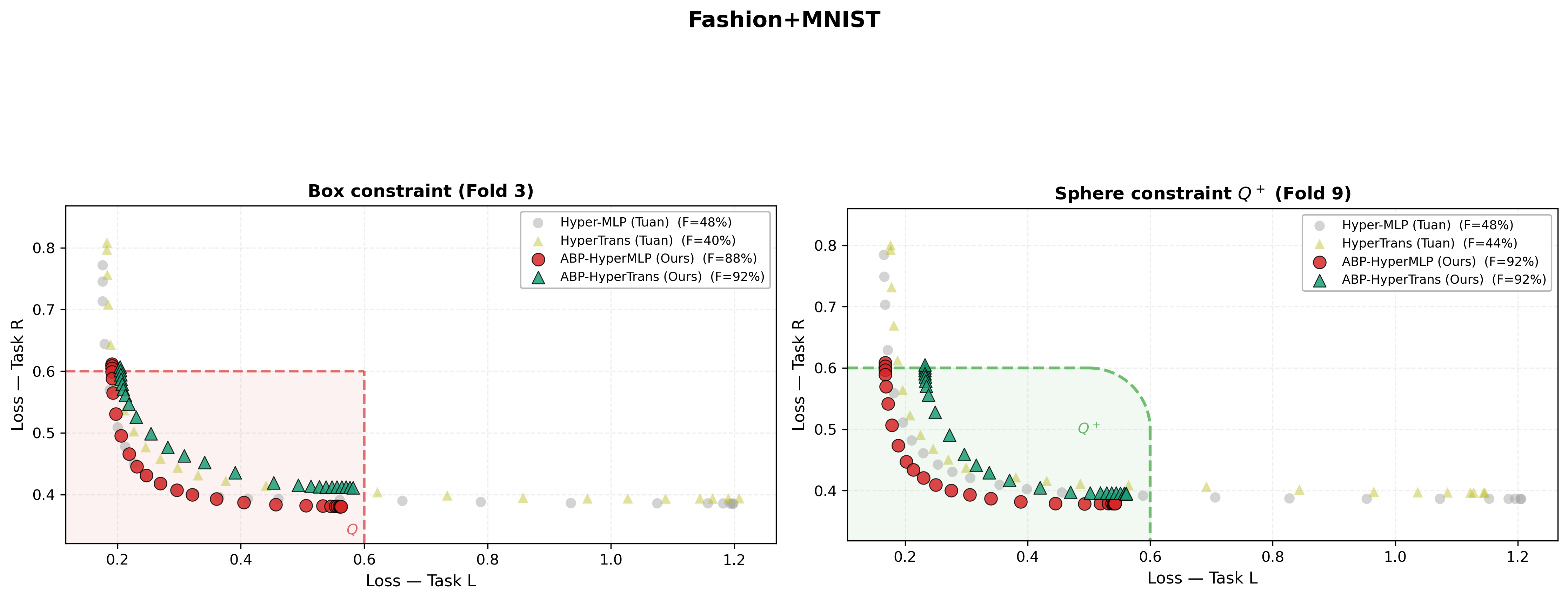}
  \caption{Fashion+MNIST: Pareto fronts of all four methods.
  Layout as in Fig.~\ref{fig:mtl-mnist-comp}.}
  \label{fig:mtl-fashmnist-comp}
\end{figure}

\subsection{Ablation study: training strategy}
\label{subsec:mtl-ablation}

To validate the necessity of each component in the two-phase
training strategy, we conduct an ablation study
(ABP-HyperMLP, Box constraint, 10-fold CV) on all three datasets
and report the mean across datasets in Table~\ref{tab:ablation}.
We compare the Full strategy against four degraded variants.
\textbf{No-Phase1} skips the Phase~1 feasibility warm-up and begins
directly with $\calL_{\mathrm{obj}} + \beta\,\calL_Q$ from epoch~1.
\textbf{No-Adaptive-$\beta$} retains both phases but fixes the
penalty coefficient ($\rho=1.0$, so $\beta$ never grows).
\textbf{Single-Phase} uses a fixed
$\calL_{\mathrm{obj}} + \beta_0\,\calL_Q$ throughout training with
no phase transition and no adaptive~$\beta$.
\textbf{Hard-Switch} uses pure $\calL_Q$ in Phase~1
($\varepsilon\!=\!0$) and pure $\calL_{\mathrm{obj}}$ in Phase~2
($\beta\!=\!0$), creating an abrupt loss discontinuity at the
transition.

\begin{table}[t]
\centering
\small
\caption{Ablation study: mean across three MTL datasets
  (ABP-HyperMLP, Box constraint, 10-fold CV).
  \textbf{Bold}: best per metric.}
\label{tab:ablation}
\setlength{\tabcolsep}{5pt}
\begin{tabular}{@{}lccc@{}}
\toprule
Variant & $\pi$\,(\%) ($\uparrow$) & HV ($\uparrow$) & EFHV ($\uparrow$) \\
\midrule
Full (proposed)
  & $\mathbf{93.2}$
  & $\mathbf{2.736}$
  & $\mathbf{2.550}$ \\
No-Phase1
  & $92.5$
  & $2.736$
  & $2.540$ \\
No-Adaptive-$\beta$
  & $90.5$
  & $2.693$
  & $2.424$ \\
Single-Phase
  & $82.9$
  & $2.684$
  & $2.205$ \\
Hard-Switch
  & $62.7$
  & $2.677$
  & $1.740$ \\
\bottomrule
\end{tabular}
\end{table}

A clear ranking emerges (Table~\ref{tab:ablation}):
Full $\approx$ No-Phase1 $\gg$ No-Adaptive-$\beta$ $>$
Single-Phase $\gg$ Hard-Switch.
The Full strategy achieves the highest EFHV (2.550) and
feasibility (93.2\%).
Hard-Switch suffers the most ($\pi=62.7\%$, EFHV~$=1.740$),
driven by catastrophic feasibility collapse on Multi-Fashion
($\pi=40\%$) and Fashion+MNIST ($\pi=48\%$), confirming that
an abrupt loss discontinuity prevents effective constraint enforcement.
No-Phase1 is close to Full in aggregate
(EFHV~$=2.540$, $-0.4\%$), but exhibits lower robustness:
on Multi-Fashion its feasibility drops to $88.8\%$ versus $94.0\%$
for Full, indicating that the warm-up phase provides a more
reliable starting point on harder constraints.
Removing adaptive penalty adjustment (No-Adaptive-$\beta$)
reduces EFHV by $4.9\%$, while the naive Single-Phase baseline
loses $13.5\%$, confirming the value of both the phase structure
and the dynamic penalty schedule.

\section{Conclusion}\label{sec:conclusion}

We have proposed the Adaptive Balanced Penalty (ABP) algorithm for the
Bi-Level Scalarized Split Problem~\eqref{eq:BSSP} and proved that
the full sequence of iterates converges to an optimal solution under
standard convexity, step-size, and zero-gap assumptions.
The ABP penalty structure is translated into a two-phase,
feasibility-first training strategy for hypernetworks
(ABP-HyperMLP, ABP-HyperTrans), and evaluated through a dedicated
metric---the Expected Feasible Hypervolume (EFHV)---that jointly
captures solution quality and constraint satisfaction.

\textit{Summary of results.}
On five multi-objective benchmarks (Section~\ref{sec:exp-mop}),
the ABP solver closely approximates the ground-truth constrained
Pareto front, and ABP-HyperTrans preserves this accuracy at
real-time inference cost.
On three multi-task learning datasets (Section~\ref{sec:exp-mtl}),
ABP-HyperNet raises feasibility from the baselines' 36--49\%
to 87--100\%, yielding up to $2.3\times$ higher EFHV while
incurring at most $\approx 0.02$ HV cost relative to the
unconstrained upper bound.

\textit{The convex surrogate technique.}
The linearized half-space surrogate~$\delta_k$
(Section~\ref{subsubsec:grad-ineq}) converts the non-convex
image-feasibility function~$G$ into a global subgradient inequality
at each iteration.
This technique is applicable to any split feasibility problem
involving a convex mapping and a convex target set, and may be of
independent interest.

\textit{Limitations and future work.}
The zero-gap Assumption~\eqref{A7} is a structural condition on the
interplay between the Pareto front and the target region~$Q^+$;
developing an algorithm that achieves exact convergence without this
condition---for instance through a dynamically updated lower bound
or a Lagrangian dual approach---is an important open problem.
Extending the convergence analysis to the stochastic, non-convex
regime and scaling to many-objective problems ($m\ge 4$) are
additional directions for future work.

\appendix
\section{Full proofs for Section~\ref{sec:convergence}}\label{app:proofs}

This appendix provides complete proofs for all results stated in
Section~\ref{sec:convergence}. Results are organized following the
logical dependency order of the proof chain. All lemma and theorem
numbers are identical to those in the main body.

\subsection{Technical lemmas (Lemmas~\ref{lem:H}--\ref{lem:ineq-G})}\label{app:tech-lemmas}

\noindent\textbf{Lemma~\ref{lem:H}} \textit{(restated).}
\textit{Under~\eqref{A2}, $H$ is convex and $C^1$ with $\nabla H(x)=x-P_C(x)$.}

\noindent\textit{Role: Used in Lemma~\ref{lem:ineq-H} and Lemma~\ref{lem:lyap}.}

\begin{proof}
$H$ is the infimal convolution of $\tfrac{1}{2}\|\cdot\|^2$ with
the indicator of~$C$, hence convex
\citep[Proposition~12.9]{bauschke2017}.
Differentiability and the gradient formula follow from Danskin's
theorem \citep[Proposition~B.22]{bertsekas1999}.
\end{proof}

\noindent\textbf{Lemma~\ref{lem:phi}} \textit{(restated).}
\textit{Under~\eqref{A1} and~\eqref{A3}, $\varphi$ is convex and $w^k\in\partial\varphi(x^k)$ for every choice of~$i^*$ in~\eqref{eq:wk}.}

\noindent\textit{Role: Used in Lemma~\ref{lem:ineq-phi}.}

\begin{proof}
$\varphi=\max_i r_i(f_i-z_i^*)$ is convex as a pointwise maximum
of convex functions.
Let $\phi_{i^*}:=r_{i^*}(f_{i^*}-z_{i^*}^*)$; then
$\phi_{i^*}(x^k)=\varphi(x^k)$ and $\phi_{i^*}\le\varphi$.
Convexity of $\phi_{i^*}$ gives
$\varphi(y)\ge\phi_{i^*}(y)\ge\varphi(x^k)+\ip{w^k}{y-x^k}$
for all $y$, confirming $w^k\in\partial\varphi(x^k)$.
\end{proof}

\noindent\textbf{Lemma~\ref{lem:Qsub}} \textit{(restated).}
\textit{$Q^+\subseteq Q^+_k$ for every~$k$.}

\noindent\textit{Role: Used in Lemma~\ref{lem:ineq-G} via $\delta_k \leq G$ pointwise.}

\begin{proof}
The projection characterization \citep[Theorem~3.14]{bauschke2017}
gives $\ip{\rho^k}{w-p^k}\le 0$ for every $w\in Q^+$.
\end{proof}

\noindent\textbf{Lemma~\ref{lem:delta-cvx}} \textit{(restated).}
\textit{Under~\eqref{A1} and~\eqref{A3}, $\delta_k$ is convex.}

\noindent\textit{Role: Used in Lemma~\ref{lem:ineq-G} to apply the gradient inequality for the convex function~$\delta_k$.}

\begin{proof}
Trivial when $\rho^k=0$.  For $\rho^k\neq 0$, set
$h_k(x):=\sum_i\rho^k_i f_i(x)$, which is convex since each
$f_i$ is convex and $\rho^k_i\ge 0$ (Corollary~\ref{cor:res}).
Then $\delta_k(x)=\psi(h_k(x)-\ip{\rho^k}{p^k})/(2\|\rho^k\|^2)$
where $\psi(t):=[t]_+^2$ is convex and nondecreasing.
Convexity of $\delta_k$ follows from the composition rule
\citep[Proposition~1.2.4]{hiriart2004}.
\end{proof}

\noindent\textbf{Lemma~\ref{lem:delta-agree}} \textit{(restated).}
\textit{For every~$k$: $\delta_k(x^k)=G_k$, and
$\nabla\delta_k(x^k)=v^k$ (with $v^k=0$ when $\rho^k=0$).}

\noindent\textit{Role: Used in Lemma~\ref{lem:ineq-G} to identify $\nabla\delta_k(x^k) = v^k$.}

\begin{proof}
When $\rho^k=0$: $G_k=\delta_k(x^k)=0$ and $v^k=0$.
When $\rho^k\neq 0$: since
$\ip{\rho^k}{\calF(x^k)-p^k}=\|\rho^k\|^2>0$, the $[\cdot]_+$
in~\eqref{eq:delta-explicit} is active and
$\delta_k(x^k)=\|\rho^k\|^2/2=G_k$.
Setting $\Psi(x):=\ip{\rho^k}{\calF(x)-p^k}$, we have
$\Psi(x^k)>0$ so $\delta_k=\Psi^2/(2\|\rho^k\|^2)$ is $C^1$
near~$x^k$, and the chain rule gives
$\nabla\delta_k(x^k)=\Psi(x^k)\,J_\calF(x^k)^T\rho^k/\|\rho^k\|^2
=J_\calF(x^k)^T\rho^k=v^k$.
\end{proof}

\noindent\textbf{Lemma~\ref{lem:ineq-phi}} \textit{(restated).}
\textit{Under~\eqref{A1}, \eqref{A3}, \eqref{A7}: $\langle w^k, x^k-x^*\rangle\ge\Delta_k$ for all~$k$ and $x^*\in\Omega$.}

\noindent\textit{Role: Used in Lemma~\ref{lem:lyap} (case $\Delta_k \geq 0$).}

\begin{proof}
From $w^k\in\partial\varphi(x^k)$:
$\varphi(x^*)\ge\varphi(x^k)+\ip{w^k}{x^*-x^k}$.
Using $\varphi(x^*)=\varphi^*=\varphi_{\lb}$ (by~\eqref{A7}) and
rearranging gives $\ip{w^k}{x^k-x^*}\ge\Delta_k$.
\end{proof}

\noindent\textbf{Lemma~\ref{lem:ineq-H}} \textit{(restated).}
\textit{Under~\eqref{A2}: $\langle z^k, x^k-x^*\rangle\ge H_k\ge 0$ for all~$k$ and $x^*\in\Omega$.}

\noindent\textit{Role: Used in Lemma~\ref{lem:lyap} (both cases).}

\begin{proof}
Convexity of $H$ and $H(x^*)=0$ give
$0\ge H_k+\ip{z^k}{x^*-x^k}$; rearranging yields the result.
\end{proof}

\noindent\textbf{Lemma~\ref{lem:ineq-G}} \textit{(restated).}
\textit{Under~\eqref{A1}--\eqref{A3}, $v^k$ satisfies a global subgradient inequality for~$G$:
$G(y)\ge G_k+\langle v^k, y-x^k\rangle$ for all $y\in\mathbb{R}^n$, $k\ge 0$.}

\noindent\textit{Role: Used in Lemma~\ref{lem:lyap} (both cases).}

\begin{proof}
When $\rho^k=0$: $G_k=v^k=0$ and~\eqref{eq:ineq-G} holds trivially.
When $\rho^k\neq 0$: Lemma~\ref{lem:Qsub} gives $\delta_k\le G$
pointwise. Using the gradient inequality for the convex function
$\delta_k$ (Lemmas~\ref{lem:delta-cvx} and~\ref{lem:delta-agree}):
$G(y)\ge\delta_k(y)\ge G_k+\ip{v^k}{y-x^k}$.
Setting $y=x^*\in \mathcal{S}$ gives $\ip{v^k}{x^k-x^*}\ge G_k$.
\end{proof}

\subsection{Lyapunov analysis (Lemma~\ref{lem:lyap}, Proposition~\ref{prop:fejer}, Lemma~\ref{lem:dir})}\label{app:lyapunov}

\noindent\textbf{Lemma~\ref{lem:lyap}} \textit{(Lyapunov inequality, restated).}
\textit{Under Assumption~\ref{ass:all}, for every $x^*\in\Omega$ and $k\ge 0$,
$\|x^{k+1}-x^*\|^2 \le \|x^k-x^*\|^2 - (2\lambda_k/\eta_k)\Phi_k + \lambda_k^2$.}

\noindent\textit{Role: Central recursive inequality; used in Lemmas~\ref{lem:sum} and~\ref{lem:lyap-approx}.}

\begin{proof}
Expanding $\|x^{k+1}-x^*\|^2$ using
$x^{k+1}=x^k-(\lambda_k/\eta_k)d^k$ and bounding
$\lambda_k^2\|d^k\|^2/\eta_k^2\le\lambda_k^2$
(since $\eta_k\ge\|d^k\|$), it suffices to show
$\ip{d^k}{x^k-x^*}\ge\Phi_k$.
By~\eqref{eq:dk} and Lemmas~\ref{lem:ineq-H}--\ref{lem:ineq-G}:
\[
  \ip{d^k}{x^k-x^*}
  \ge\alpha_k\,\mathbf{1}_{\{\Delta_k\ge 0\}}\ip{w^k}{x^k-x^*}
  +\beta_k H_k+\gamma_k G_k.
\]
If $\Delta_k\ge 0$: Lemma~\ref{lem:ineq-phi} gives the first term
$\ge\alpha_k\Delta_k=\alpha_k\Delta_k^+$, so
$\ip{d^k}{x^k-x^*}\ge\Phi_k$.
If $\Delta_k<0$: $\Delta_k^+=0$ and the indicator vanishes, so
$\ip{d^k}{x^k-x^*}\ge\beta_k H_k+\gamma_k G_k=\Phi_k$.
\end{proof}

\noindent\textbf{Proposition~\ref{prop:fejer}} \textit{(Quasi-Fej\'er property, restated).}
\textit{Under Assumption~\ref{ass:all}, the sequence
$\{\|x^k-x^*\|^2\}$ converges for every $x^*\in\Omega$;
in particular, $\{x^k\}$ is bounded.}

\noindent\textit{Role: Establishes boundedness of $\{x^k\}$; used in Lemma~\ref{lem:dir} and Proposition~\ref{prop:subseq}.}

\begin{proof}
Drop the nonpositive term in~\eqref{eq:lyap} to get
$\norm{x^{k+1}-x^*}^2\le\norm{x^k-x^*}^2+\lambda_k^2$.
Apply Lemma~\ref{lem:RS} with $b_k:=\lambda_k^2$; summability
holds by~\eqref{A5}.
\end{proof}

\noindent\textbf{Lemma~\ref{lem:dir}} \textit{(restated).}
\textit{Under Assumption~\ref{ass:all}, $\|d^k\|\le\bar M<\infty$ for
all~$k$, and $\mu\le\eta_k\le\bar\eta:=\max(\mu,\bar M)$.}

\noindent\textit{Role: Bounds $\|d^k\|$ and $\eta_k$; used in Lemmas~\ref{lem:sum} and~\ref{lem:ergodic}.}

\begin{proof}
By Proposition~\ref{prop:fejer}, $\{x^k\}$ lies in a compact ball
$K=\bar B(x^*,R)$.
Nonexpansiveness of $P_C$ gives $\|z^k\|\le 2R$.
Continuity of each $\nabla f_i$ on $K$ bounds $\|w^k\|\le M_w$.
Continuity of $\calF$ and $J_\calF$ on $K$ gives
$\|\rho^k\|\le M_\rho$ and $\|J_\calF(x^k)\|_{\mathrm{op}}\le M_J$,
hence $\|v^k\|\le M_J M_\rho$.
By~\eqref{A6}:
$\|d^k\|\le\bar\alpha M_w+2\bar\beta R+\bar\gamma M_J M_\rho
=:\bar M$.
\end{proof}

\subsection{Global convergence: proof of Theorem~\ref{thm:main}}\label{app:global-conv}

\noindent\textbf{Lemma~\ref{lem:sum}} \textit{(restated).}
\textit{$\sum_{k=0}^\infty\lambda_k\Phi_k<\infty$.}

\begin{proof}
Sum~\eqref{eq:lyap} from $0$ to $N-1$ and drop the
nonnegative term $\|x^N-x^*\|^2$:
$\sum_{k=0}^{N-1}(2\lambda_k/\eta_k)\Phi_k
\le\|x^0-x^*\|^2+\sum_{k=0}^{N-1}\lambda_k^2$.
Since $\eta_k\le\bar\eta$ (Lemma~\ref{lem:dir}), letting
$N\to\infty$ gives the result using~\eqref{A5}.
\end{proof}

\noindent\textbf{Lemma~\ref{lem:liminf}} \textit{(restated).}
\textit{$\liminf_{k\to\infty}\Phi_k=0$.}

\begin{proof}
If $\Phi_k\ge\delta>0$ for all $k\ge k_0$, then
$\sum_{k\ge k_0}\lambda_k\Phi_k\ge\delta\sum_{k\ge k_0}\lambda_k=\infty$
by~\eqref{A5}, contradicting Lemma~\ref{lem:sum}.
\end{proof}

\noindent\textbf{Proposition~\ref{prop:simult}} \textit{(restated).}
\textit{There exists a subsequence $\{k_j\}$ with
$\Delta_{k_j}^+\to 0$, $H_{k_j}\to 0$, $G_{k_j}\to 0$, and
$\Delta_{k_j}\to 0$.}

\begin{proof}
Extract $\{k_j\}$ with $\Phi_{k_j}\to 0$ (Lemma~\ref{lem:liminf}).

\textit{Step~1.}
$\Phi_{k_j}\ge\underline\beta H_{k_j}+\underline\gamma G_{k_j}$,
so $H_{k_j}\to 0$ and $G_{k_j}\to 0$.
Then $\alpha_{k_j}\Delta_{k_j}^+=\Phi_{k_j}-\beta_{k_j}H_{k_j}
-\gamma_{k_j}G_{k_j}\to 0$, and $\alpha_{k_j}\ge\underline\alpha>0$
gives $\Delta_{k_j}^+\to 0$.

\textit{Step~2.}
We show $\liminf_j\Delta_{k_j}\ge 0$.  Suppose not; extract a
further subsequence with $x^{k_j}\to\hat x$ and
$\Delta_{k_j}\to\ell_0<0$.
From $H_{k_j}\to 0$ and closedness of $C$: $\hat x\in C$.
From $G_{k_j}\to 0$, continuity of $\calF$, and closedness of
$Q^+$: $\calF(\hat x)\in Q^+$, so $\hat x\in \mathcal{S}$.
By continuity of $\varphi$ and assumption~\eqref{A7}:
$\varphi(\hat x)=\ell_0+\varphi_{\lb}=\ell_0+\varphi^*<\varphi^*$,
contradicting $\hat x\in \mathcal{S}$.
Hence $\Delta_{k_j}\to 0$.
\end{proof}

\noindent\textbf{Proposition~\ref{prop:subseq}} \textit{(restated).}
\textit{There exist a subsequence $\{k_j\}$ and $x^\infty\in\Omega$ with
$x^{k_j}\to x^\infty$.}

\begin{proof}
Take $\{k_j\}$ from Proposition~\ref{prop:simult}; extract a
convergent sub-subsequence $x^{k_j}\to x^\infty$ using boundedness
(Proposition~\ref{prop:fejer}).
Step~2 of Proposition~\ref{prop:simult} (with $\ell_0=0$) shows
$x^\infty\in \mathcal{S}$, and $\Delta_{k_j}\to 0$ with continuity of
$\varphi$ gives $\varphi(x^\infty)=\varphi_{\lb}=\varphi^*$, so
$x^\infty\in\Omega$.
\end{proof}

\noindent\textbf{Theorem~\ref{thm:main}} \textit{(Global convergence, restated).}
\textit{Under Assumption~\ref{ass:all}, $\{x^k\}$ generated by
the ABP algorithm (Algorithm~\ref{alg:main}) converges to some $x^\infty\in\Omega$.}

\begin{proof}
Let $x^\infty$ be as in Proposition~\ref{prop:subseq}.
By Proposition~\ref{prop:fejer} (applied with $x^*:=x^\infty$,
valid for every $x^*\in\Omega$), the sequence
$a_k:=\|x^k-x^\infty\|^2$ converges to some $a^*\ge 0$.
Along $\{k_j\}$: $a_{k_j}\to 0$, so $a^*=0$ and
$x^k\to x^\infty$.
\end{proof}

\subsection{Approximate convergence: proof of Theorem~\ref{thm:approx}}\label{app:approx-conv}

\noindent\textbf{Lemma~\ref{lem:ineq-mod}} \textit{(Modified optimality inner product, restated).}
\textit{Let $x^*\in\Omega$.  Under \textup{\eqref{A1}, \eqref{A3}, (A7')}:
$\langle w^k, x^k-x^*\rangle \ge \Delta_k - \sigma$ for all $k\ge 0$.}

\begin{proof}
The subgradient inequality $w^k\in\partial\varphi(x^k)$ gives
$\varphi(x^*)\ge\varphi(x^k)+\ip{w^k}{x^*-x^k}$, so
\[
  \ip{w^k}{x^k-x^*}
  \;\ge\;
  \varphi(x^k)-\varphi(x^*)
  =\Delta_k+\varphi_{\lb}-\varphi^*
  =\Delta_k-\sigma.
\]
\end{proof}

\noindent\textbf{Lemma~\ref{lem:lyap-approx}} \textit{(Modified Lyapunov inequality, restated).}
\textit{Under \textup{\eqref{A1}--\eqref{A6}, (A7'), (A8)}, for every $x^*\in\Omega$,
$\|x^{k+1}-x^*\|^2 \le \|x^k-x^*\|^2 - (2\lambda_k/\eta_k)\Phi_k + (2\bar\alpha\sigma/\mu)\lambda_k + \lambda_k^2$.}

\begin{proof}
Expanding $x^{k+1}=x^k-(\lambda_k/\eta_k)d^k$:
\[
  \norm{x^{k+1}-x^*}^2
  = \norm{x^k-x^*}^2
    - \frac{2\lambda_k}{\eta_k}\ip{d^k}{x^k-x^*}
    + \frac{\lambda_k^2}{\eta_k^2}\norm{d^k}^2.
\]
Since $\eta_k\ge\norm{d^k}$, the last term satisfies
$\lambda_k^2\norm{d^k}^2/\eta_k^2\le\lambda_k^2$.  It remains to
bound $\ip{d^k}{x^k-x^*}$ from below.

\textit{Case $\Delta_k<0$ (indicator $=0$):}
$d^k=\beta_k z^k+\gamma_k v^k$, so by Lemmas~\ref{lem:ineq-H}
and~\ref{lem:ineq-G},
\[
  \ip{d^k}{x^k-x^*}\ge\beta_k H_k+\gamma_k G_k=\Phi_k
  \ge\Phi_k-\bar\alpha\sigma.
\]

\textit{Case $\Delta_k\ge 0$ (indicator $=1$):}
Using Lemma~\ref{lem:ineq-mod} and Lemmas~\ref{lem:ineq-H}
and~\ref{lem:ineq-G}:
\begin{align*}
  \ip{d^k}{x^k-x^*}
  &\ge \alpha_k(\Delta_k-\sigma)+\beta_k H_k+\gamma_k G_k \\
  &= \alpha_k\Delta_k^+ + \beta_k H_k + \gamma_k G_k - \alpha_k\sigma
  \;=\; \Phi_k - \alpha_k\sigma
  \;\ge\; \Phi_k - \bar\alpha\sigma.
\end{align*}

In both cases $\ip{d^k}{x^k-x^*}\ge\Phi_k-\bar\alpha\sigma$.
Using $\eta_k\ge\mu$:
\[
  -\frac{2\lambda_k}{\eta_k}\ip{d^k}{x^k-x^*}
  \;\le\;
  -\frac{2\lambda_k}{\eta_k}\Phi_k
  + \frac{2\bar\alpha\sigma\,\lambda_k}{\eta_k}
  \;\le\;
  -\frac{2\lambda_k}{\eta_k}\Phi_k
  + \frac{2\bar\alpha\sigma}{\mu}\,\lambda_k.
\]
\end{proof}

\paragraph{The drift term and its consequences.}
Compared with the Lyapunov inequality~\eqref{eq:lyap}, the only new
term is the drift $(2\bar\alpha\sigma/\mu)\lambda_k$.  When $\sigma=0$
it vanishes and the original recursion is recovered.  When $\sigma>0$,
the partial sum $(2\bar\alpha\sigma/\mu)\Lambda_N\to\infty$ as
$N\to\infty$ (since $\sum\lambda_k=\infty$).  Consequently the
Robbins--Siegmund lemma cannot be applied to $\norm{x^k-x^*}^2$, and
the sequence may fail to be bounded; hence Assumption~(A8) is
necessary for the analysis that follows.

\subsubsection{Ergodic and simultaneous approximate bounds}

\noindent\textbf{Lemma~\ref{lem:ergodic}} \textit{(Ergodic bound, restated).}
\textit{Under \textup{\eqref{A1}--\eqref{A6}, (A7'), (A8)}, define
$C_* := \bar\alpha\bar\eta/\mu$ and $\Lambda_N:=\sum_{k=0}^N\lambda_k$. Then
$\limsup_{N\to\infty} \Lambda_N^{-1}\sum_{k=0}^N\lambda_k\Phi_k \le C_*\varepsilon_0$ and
$\liminf_{k\to\infty}\Phi_k \le C_*\varepsilon_0$.}

\begin{proof}
Fix any $x^*\in\Omega$ and sum~\eqref{eq:lyap-mod} from $k=0$ to
$k=N$:
\begin{equation}\label{eq:sum-lyap}
  2\sum_{k=0}^N\frac{\lambda_k}{\eta_k}\,\Phi_k
  \;\le\;
  \norm{x^0-x^*}^2 - \norm{x^{N+1}-x^*}^2
  + \frac{2\bar\alpha\sigma}{\mu}\,\Lambda_N
  + \sum_{k=0}^N\lambda_k^2.
\end{equation}
Since $\eta_k\le\bar\eta$ for all~$k$,
we have $\lambda_k/\eta_k\ge\lambda_k/\bar\eta$, so
\[
  \frac{2}{\bar\eta}\sum_{k=0}^N\lambda_k\Phi_k
  \;\le\;
  2\sum_{k=0}^N\frac{\lambda_k}{\eta_k}\Phi_k.
\]
Combining with~\eqref{eq:sum-lyap} and dropping
$-\norm{x^{N+1}-x^*}^2\le 0$:
\[
  \frac{2}{\bar\eta}\sum_{k=0}^N\lambda_k\Phi_k
  \;\le\;
  \norm{x^0-x^*}^2
  + \frac{2\bar\alpha\sigma}{\mu}\,\Lambda_N
  + \sum_{k=0}^N\lambda_k^2.
\]
Dividing by $\Lambda_N>0$ and taking $N\to\infty$: the term
$(\norm{x^0-x^*}^2+\sum_{k=0}^\infty\lambda_k^2)/\Lambda_N\to 0$
(finite numerator, by~\eqref{A5}; $\Lambda_N\to\infty$, by~\eqref{A5}).
Using $\sigma\le\varepsilon_0$ yields~\eqref{eq:cesaro} after
multiplying by $\bar\eta/2$.

For~\eqref{eq:ergodic}: if $\Phi_k>C_*\varepsilon_0+\delta$ for all
$k\ge k_0$ and some $\delta>0$, then
$\Lambda_N^{-1}\sum_{k=0}^N\lambda_k\Phi_k
>C_*\varepsilon_0+\delta-O(1/\Lambda_N)$,
and taking $N\to\infty$ gives $\limsup>C_*\varepsilon_0$,
contradicting~\eqref{eq:cesaro}.
\end{proof}

\noindent\textbf{Proposition~\ref{prop:simult-approx}} \textit{(Approximate simultaneous bounds, restated).}
\textit{Under \textup{\eqref{A1}--\eqref{A6}, (A7'), (A8)}, there exist a subsequence
$\{k_j\}$ and $\ell_*\in[0,C_*\varepsilon_0]$ with $\Phi_{k_j}\to\ell_*$ and
$H_{k_j} \le \Phi_{k_j}/\underline\beta$, $G_{k_j} \le \Phi_{k_j}/\underline\gamma$,
$\Delta_{k_j}^+ \le \Phi_{k_j}/\underline\alpha$.}

\begin{proof}
By Lemma~\ref{lem:ergodic}, extract a subsequence $\{k_j\}$ with
$\Phi_{k_j}\to\ell_*\le C_*\varepsilon_0$.
Since each summand of $\Phi_k$ is nonneg\-ative:
\[
  \underline\alpha\,\Delta_{k_j}^+\le\Phi_{k_j},\quad
  \underline\beta\,H_{k_j}\le\Phi_{k_j},\quad
  \underline\gamma\,G_{k_j}\le\Phi_{k_j}.
\]
Dividing by the respective lower bounds and letting $j\to\infty$
yields~\eqref{eq:Hbound}--\eqref{eq:Dbound}.  When $\varepsilon_0=0$,
$\ell_*=0$ and each upper bound is~$0$; since each term is also
nonneg\-ative, all three sequences converge to~$0$.
\end{proof}

\paragraph{Remark.}
Equations~\eqref{eq:Hbound}--\eqref{eq:Gbound} provide upper bounds
on $H_{k_j}$ and $G_{k_j}$ whose limits equal $\ell_*/\underline\beta$
and $\ell_*/\underline\gamma$ respectively; these limits are zero only
when $\varepsilon_0=0$.  The cluster-point bounds in
Theorem~\ref{thm:approx} are consequently stated as inequalities~$\le$,
not as equalities, reflecting the fact that continuity of $H$ and $G$
transfers the upper bound---not an exact limit---to any cluster
point~$\hat x$.

\subsubsection{Main approximate convergence theorem}

\noindent\textbf{Theorem~\ref{thm:approx}} \textit{(Approximate convergence, restated).}
\textit{Under \textup{\eqref{A1}--\eqref{A6}, (A7'), (A8)}, let $C_*=\bar\alpha\bar\eta/\mu$.
Every cluster point~$\hat x$ satisfies approximate feasibility in~$C$,
approximate split feasibility, and approximate optimality bounds that
vanish when $\varepsilon_0=0$.}

\begin{proof}
By~(A8) and the Bolzano--Weierstrass theorem, extract a further
subsequence (still denoted $\{k_j\}$) with $x^{k_j}\to\hat x$.

\medskip\noindent
\textit{Proofs of~\ref{thm:feas} and~\ref{thm:sfeas}.}
The function $H(x)=\tfrac12\dist^2(x,C)$ is continuous
(Lemma~\ref{lem:H}).  The function $G(x)=\tfrac12\dist^2(\calF(x),Q^+)$
is continuous by continuity of~$\calF$~\eqref{A1} and continuity of the
projection onto the closed convex set~$Q^+$.  Using continuity of~$H$
and the bound~\eqref{eq:Hbound}:
\[
  H(\hat x)
  =\lim_{j\to\infty}H(x^{k_j})
  =\lim_{j\to\infty}H_{k_j}
  \;\le\;
  \lim_{j\to\infty}\frac{\Phi_{k_j}}{\underline\beta}
  =\frac{\ell_*}{\underline\beta}
  \;\le\;\frac{C_*\varepsilon_0}{\underline\beta}.
\]
The distance bound $\dist(\hat x,C)\le\sqrt{2C_*\varepsilon_0/\underline\beta}$
follows by taking square roots and multiplying by~$\sqrt{2}$.
The argument for $G(\hat x)$ is identical, using~\eqref{eq:Gbound}.

\medskip\noindent
\textit{Proof of~\ref{thm:opt}.}
Since $\varphi$ is continuous~\eqref{A1}, \eqref{A3}:
\[
  \varphi(\hat x)
  =\lim_{j\to\infty}\varphi(x^{k_j})
  =\lim_{j\to\infty}\bigl(\Delta_{k_j}+\varphi_{\lb}\bigr).
\]
From~\eqref{eq:Dbound}, $\Delta_{k_j}\le\Delta_{k_j}^+\le\ell_*/\underline\alpha+o(1)$,
so taking the limit:
\[
  \varphi(\hat x)
  \;\le\;\varphi_{\lb}+\frac{\ell_*}{\underline\alpha}
  \;\le\;\varphi^*+\frac{\ell_*}{\underline\alpha}
  \;\le\;\varphi^*+\frac{C_*\varepsilon_0}{\underline\alpha}.
\]

\medskip\noindent
\textit{Exact convergence when $\varepsilon_0=0$.}
If $\varepsilon_0=0$ then $\ell_*=0$ (Proposition~\ref{prop:simult-approx}).
From~\ref{thm:feas}: $H(\hat x)=0$, so $\dist(\hat x,C)=0$ and
closedness of~$C$ gives $\hat x\in C$.
From~\ref{thm:sfeas}: $G(\hat x)=0$, so $\dist(\calF(\hat x),Q^+)=0$
and closedness of~$Q^+$ gives
$\calF(\hat x)\in Q^+$.
Hence $\hat x\in \mathcal{S}$.
From~\ref{thm:opt}: $\varphi(\hat x)\le\varphi^*$.  Combined with
$\hat x\in \mathcal{S}$ and the definition $\varphi^*=\inf_{x\in \mathcal{S}}\varphi(x)$:
$\varphi(\hat x)=\varphi^*$, so $\hat x\in\Omega$.
\end{proof}


\section{Constraint geometry visualization}\label{app:constraint-viz}

This appendix complements the main-body results
(Tables~\ref{tab:mtl-mnist}--\ref{tab:mtl-fashmnist}) with two
additions:
(i)~the unconstrained ABP baseline
($Q = \mathbb{R}^m$, denoted ``None''), which was not reported in
Section~\ref{subsec:mtl-results}, and
(ii)~per-architecture Pareto front visualizations under all three
constraint settings.
Rather than repeating the absolute HV, $\pi$, and EFHV values already
given in Tables~\ref{tab:mtl-mnist}--\ref{tab:mtl-fashmnist},
Table~\ref{tab:app-cost} reports the \emph{cost of constraint
enforcement}: the unconstrained HV serves as the upper bound, and
$\Delta\text{HV} = \text{HV}_{\text{None}} - \text{HV}_{Q}$ quantifies
how much front coverage is sacrificed to achieve feasibility rate~$\pi$.

\begin{table}[!htbp]
\centering
\small
\caption{Cost of constraint enforcement for ABP-HyperMLP and
  ABP-HyperTrans on each dataset.
  $\text{HV}_{\text{None}}$: unconstrained hypervolume (upper bound);
  $\Delta\text{HV}$: HV loss due to constraint ($>0$ means coverage
  reduced; $\leq 0$ means no loss);
  $\pi$: feasibility rate achieved by the constrained model (from
  Tables~\ref{tab:mtl-mnist}--\ref{tab:mtl-fashmnist}).
  All values are means over 10-fold CV with reference point $(2,2)$.}
\label{tab:app-cost}
\setlength{\tabcolsep}{4pt}
\begin{tabular}{@{}llcccc@{}}
\toprule
Dataset & Architecture & $\text{HV}_{\text{None}}$ &
  Constr.\ & $\Delta\text{HV}$ ($\downarrow$) &
  $\pi$\,(\%) ($\uparrow$) \\
\midrule
\multirow{4}{*}{Multi-MNIST}
  & \multirow{2}{*}{ABP-HyperMLP}
    & \multirow{2}{*}{$2.979$}
      & Box    & $0.021$ & $100.0$ \\
  & &  & Sphere & $0.011$ & $100.0$ \\
\cmidrule(l){2-6}
  & \multirow{2}{*}{ABP-HyperTrans}
    & \multirow{2}{*}{$2.970$}
      & Box    & $0.011$ & $\phantom{0}92.5$ \\
  & &  & Sphere & $0.015$ & $\phantom{0}99.5$ \\
\midrule
\multirow{4}{*}{Multi-Fashion}
  & \multirow{2}{*}{ABP-HyperMLP}
    & \multirow{2}{*}{$2.337$}
      & Box    & $0.021$ & $\phantom{0}94.9$ \\
  & &  & Sphere & $0.015$ & $100.0$ \\
\cmidrule(l){2-6}
  & \multirow{2}{*}{ABP-HyperTrans}
    & \multirow{2}{*}{$2.201$}
      & Box    & $0.005$ & $\phantom{0}90.3$ \\
  & &  & Sphere & $0.006$ & $\phantom{0}89.1$ \\
\midrule
\multirow{4}{*}{Fashion+MNIST}
  & \multirow{2}{*}{ABP-HyperMLP}
    & \multirow{2}{*}{$2.928$}
      & Box    & $0.000$ & $\phantom{0}87.4$ \\
  & &  & Sphere & $-0.007$ & $\phantom{0}88.6$ \\
\cmidrule(l){2-6}
  & \multirow{2}{*}{ABP-HyperTrans}
    & \multirow{2}{*}{$2.847$}
      & Box    & $0.003$ & $\phantom{0}90.9$ \\
  & &  & Sphere & $-0.020$ & $\phantom{0}91.4$ \\
\bottomrule
\end{tabular}
\end{table}

\paragraph{Analysis.}
Table~\ref{tab:app-cost} reveals that the HV cost of constraint
enforcement is consistently small.
Across all 12 (dataset, architecture, constraint) combinations, the
largest positive $\Delta\text{HV}$ is only~$0.021$, corresponding to a
relative loss of less than~$1\%$.
On Fashion+MNIST, $\Delta\text{HV}$ is zero or slightly negative,
indicating that the constraint penalty can act as an implicit
regularizer that marginally \emph{improves} front quality.
Comparing constraint geometries, Sphere constraints generally
achieve higher feasibility than Box for the same architecture
(e.g., ABP-HyperTrans on Multi-MNIST: $\pi = 99.5\%$ vs.\ $92.5\%$;
ABP-HyperMLP on Multi-Fashion: $\pi = 100.0\%$ vs.\ $94.9\%$),
which is consistent with the smoother gradient landscape of the
ball projection $P_{\overline{B}(\bm{c},R)}$ compared to the
axis-aligned Box projection.
The only exception is Fashion+MNIST, where both constraint types yield
similar $\pi$ values ($87$--$91\%$), suggesting that the asymmetric
task structure---rather than constraint geometry---is the dominant
factor affecting feasibility.

\paragraph{Qualitative visualization.}
Figures~\ref{fig:app-mn-mlp}--\ref{fig:app-fm-trans} show the
Pareto fronts produced by each ABP variant under all three constraint
settings for a representative fold of each dataset.
In each panel, the unconstrained solutions (None, blue) trace the
full Pareto curve, while the Box (red) and Sphere (green) solutions
retract into the feasible region---indicated by dashed boundaries
and shading---without collapsing diversity along the front.
The visual patterns are consistent with the quantitative findings
in Table~\ref{tab:app-cost}: front coverage is largely preserved
while solutions are redirected into the target region~$Q$.

\begin{figure}[!htbp]
  \centering
  \includegraphics[width=0.82\textwidth]{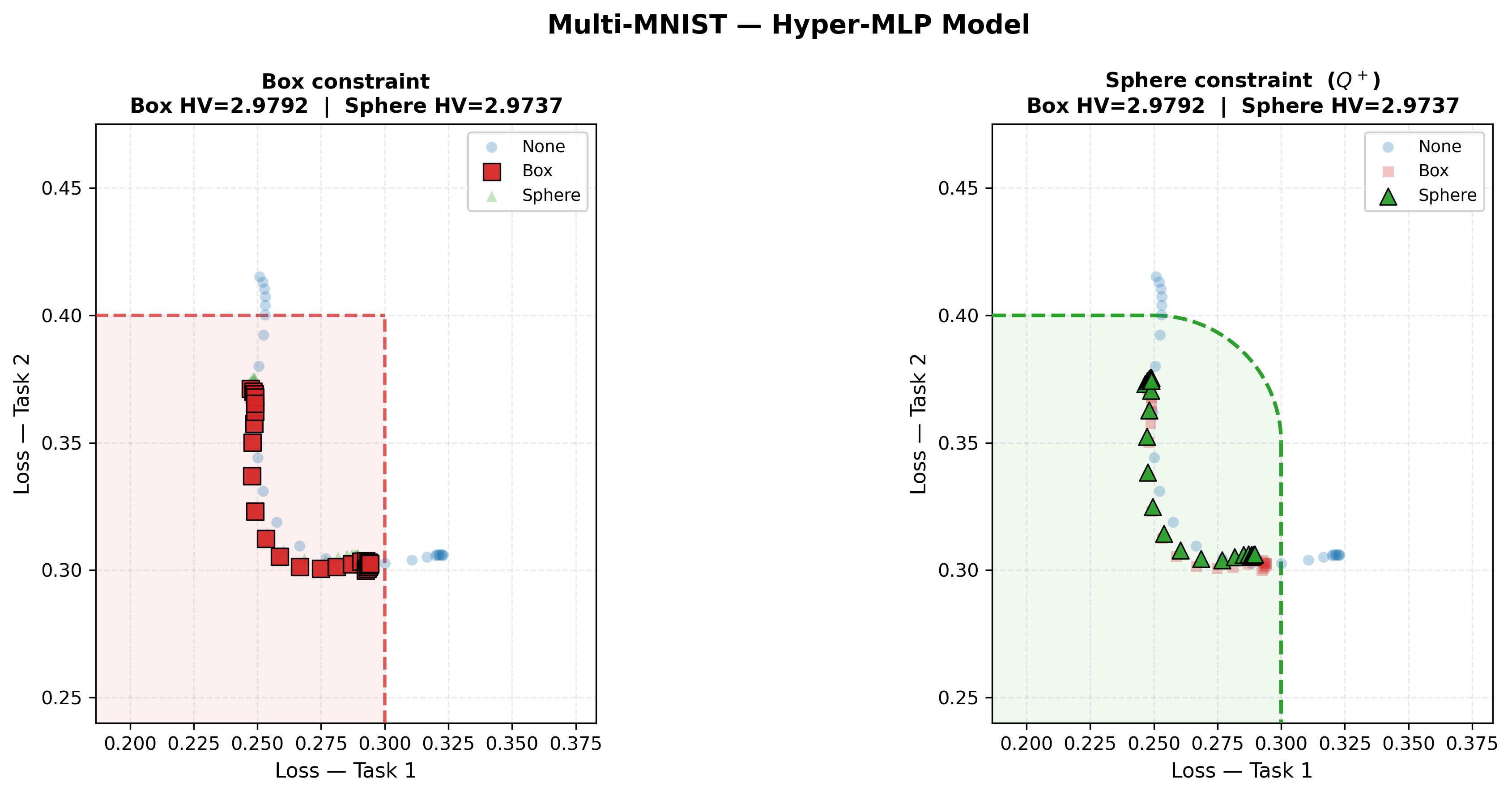}
  \caption{Multi-MNIST, ABP-HyperMLP: Pareto front under None (blue),
  Box (red), and Sphere (green) constraints.
  Dashed lines: constraint boundaries; shaded: $Q$ or $Q^+$.}
  \label{fig:app-mn-mlp}
\end{figure}

\begin{figure}[!htbp]
  \centering
  \includegraphics[width=0.82\textwidth]{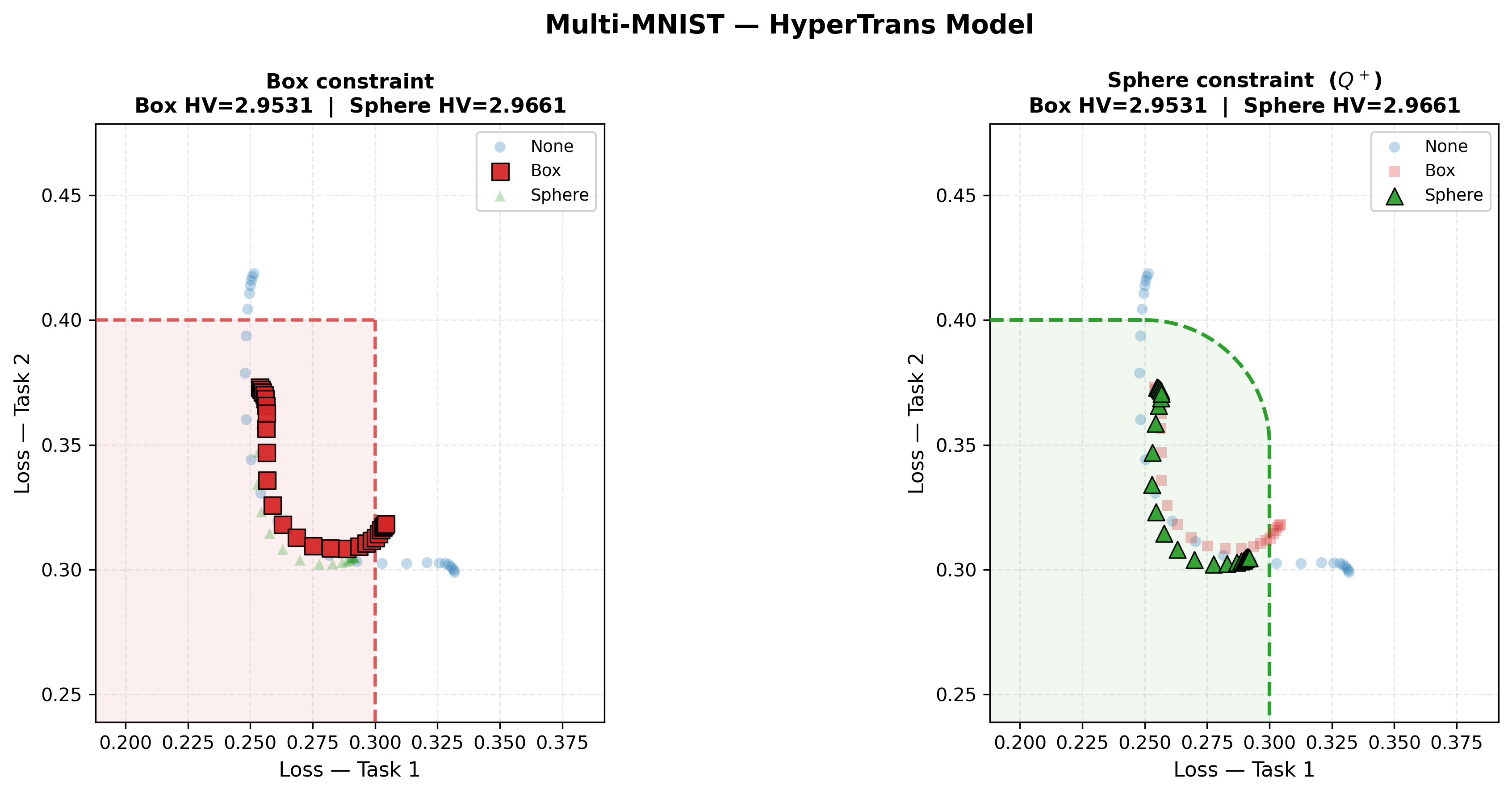}
  \caption{Multi-MNIST, ABP-HyperTrans: layout as in
  Fig.~\ref{fig:app-mn-mlp}.}
  \label{fig:app-mn-trans}
\end{figure}

\begin{figure}[!htbp]
  \centering
  \includegraphics[width=0.82\textwidth]{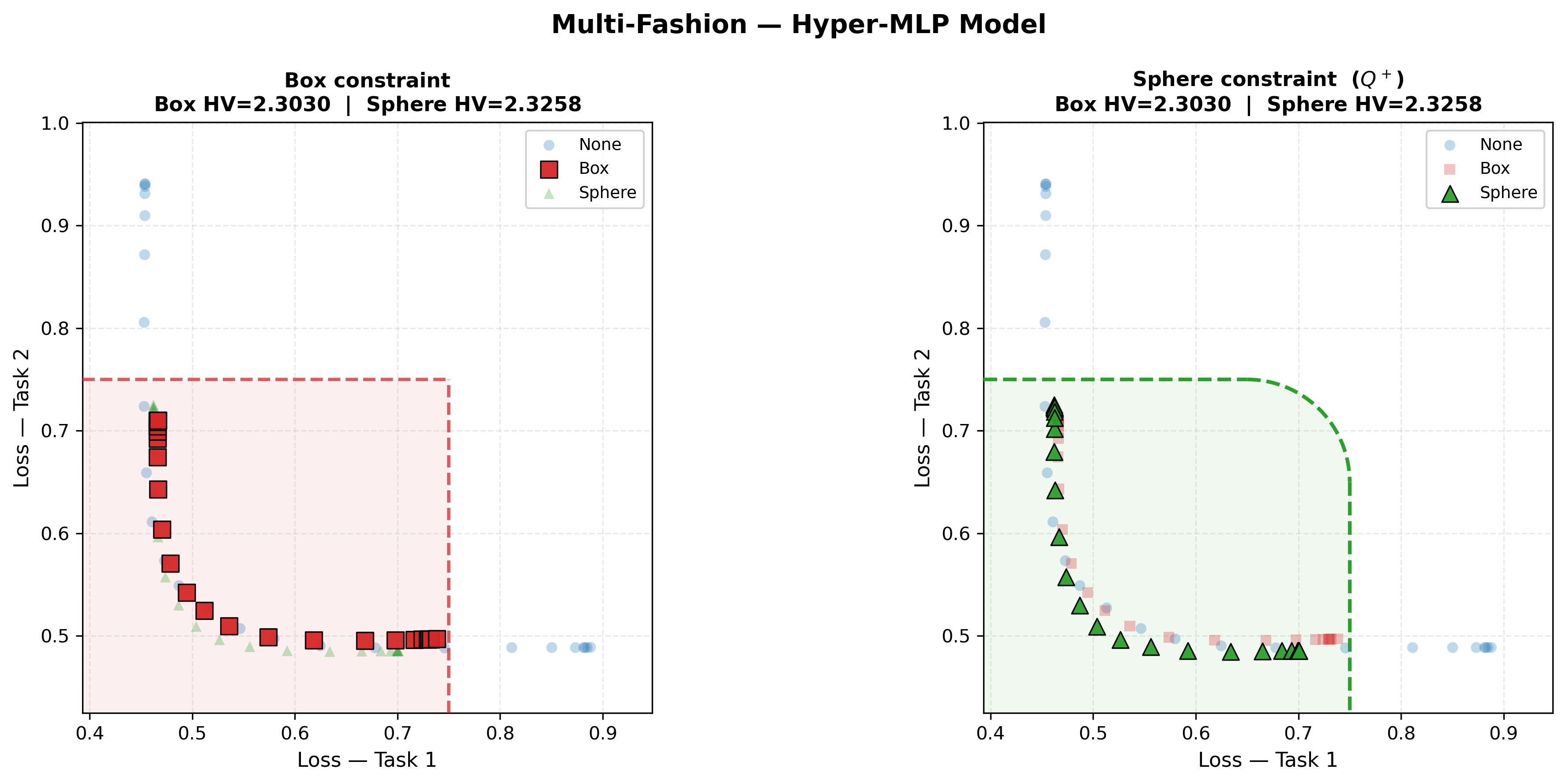}
  \caption{Multi-Fashion, ABP-HyperMLP: layout as in
  Fig.~\ref{fig:app-mn-mlp}.}
  \label{fig:app-mf-mlp}
\end{figure}

\begin{figure}[!htbp]
  \centering
  \includegraphics[width=0.82\textwidth]{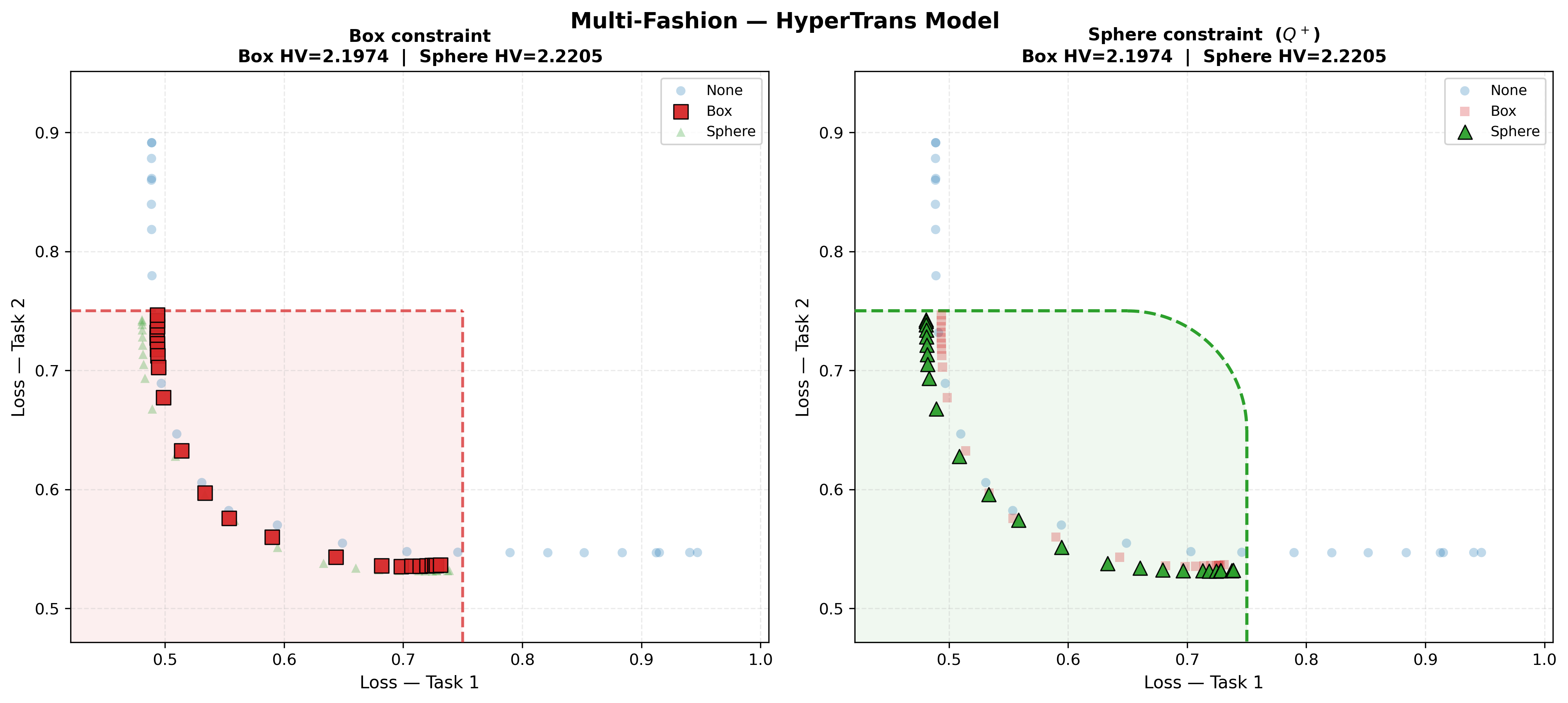}
  \caption{Multi-Fashion, ABP-HyperTrans: layout as in
  Fig.~\ref{fig:app-mn-mlp}.}
  \label{fig:app-mf-trans}
\end{figure}

\begin{figure}[!htbp]
  \centering
  \includegraphics[width=0.82\textwidth]{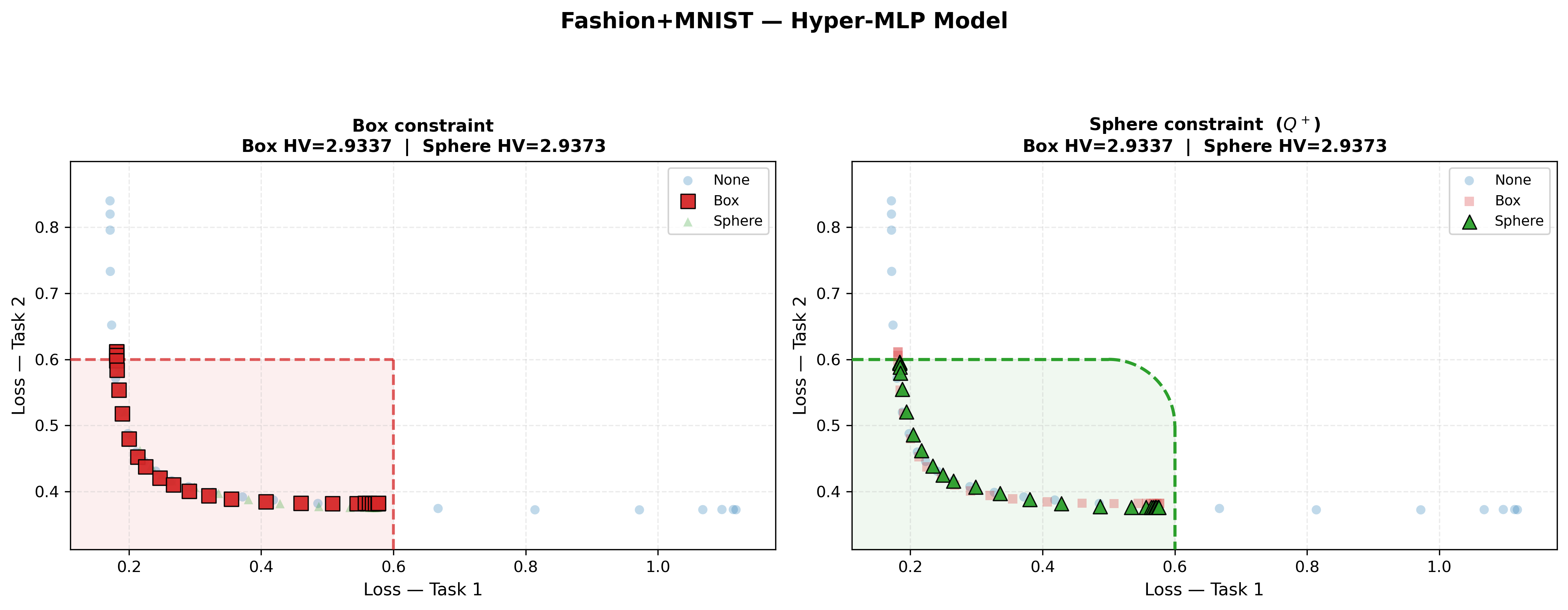}
  \caption{Fashion+MNIST, ABP-HyperMLP: layout as in
  Fig.~\ref{fig:app-mn-mlp}.}
  \label{fig:app-fm-mlp}
\end{figure}

\begin{figure}[!htbp]
  \centering
  \includegraphics[width=0.82\textwidth]{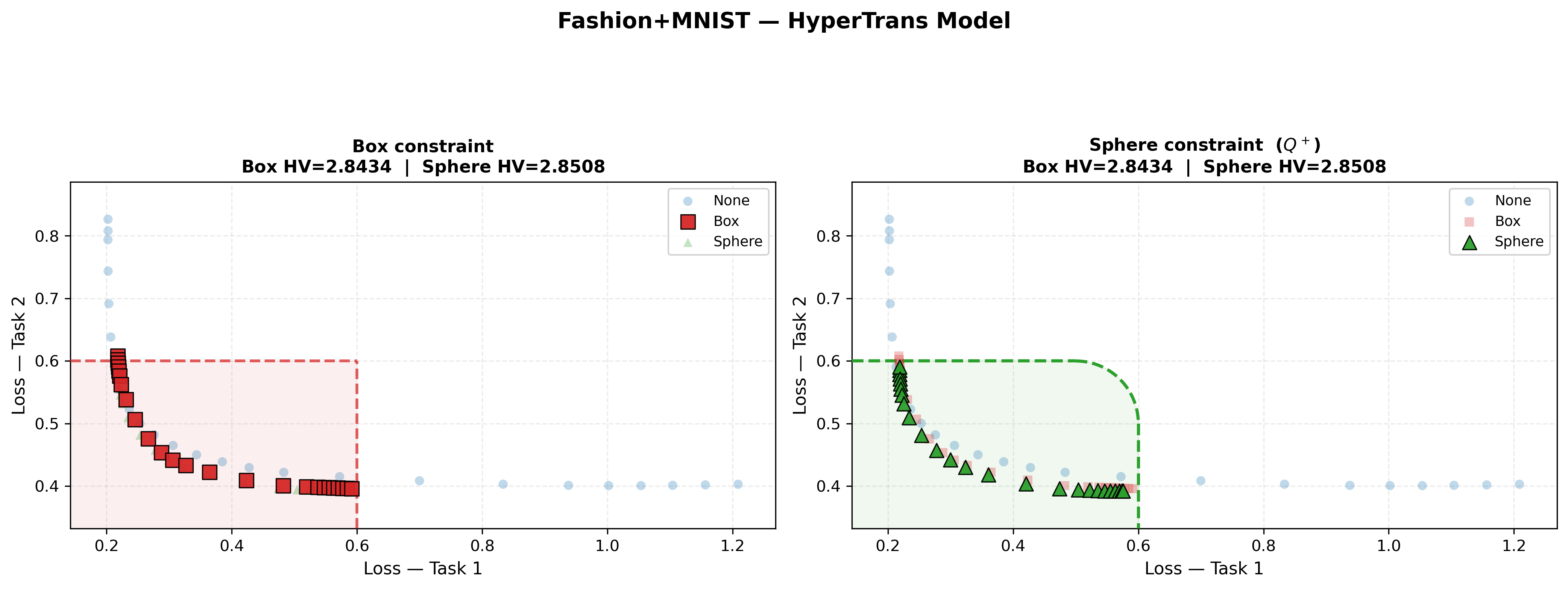}
  \caption{Fashion+MNIST, ABP-HyperTrans: layout as in
  Fig.~\ref{fig:app-mn-mlp}.}
  \label{fig:app-fm-trans}
\end{figure}


\section{Supplementary experimental details}\label{app:supp-exp}

\subsection{Verification of Assumptions~(A7)--(A8) for CVX1 and CVX2}
\label{app:sigma-gap}

For CVX1 and CVX2 the decision set $C$ is a compact box, so
Assumption~(A8) holds automatically for any initialization
$x^0 \in C$.  To assess the zero-gap condition, we compute
\[
  \varphi_{\lb} \;=\; \min_{x \in C}\,\varphi(x;\,\boldsymbol{r})
  \quad\text{and}\quad
  \varphi^* \;=\; \min_{x \in \mathcal{S}}\,\varphi(x;\,\boldsymbol{r})
\]
via SLSQP with analytical Jacobian for each of the 50 preference
rays~$\boldsymbol{r}$.  The resulting gap
$\sigma(\boldsymbol{r}) := \varphi^*(\boldsymbol{r}) -
\varphi_{\lb}(\boldsymbol{r})$ is reported in
Table~\ref{tab:sigma}.  For rays where $\sigma = 0$, the
zero-gap condition~\eqref{A7} holds exactly and Theorem~\ref{thm:main}
guarantees full-sequence convergence.  For rays where $\sigma > 0$,
Theorem~\ref{thm:approx} applies with $\varepsilon_0 = \sigma$.

\begin{table}[!htbp]
\centering
\caption{Optimality gap $\sigma = \varphi^* - \varphi_{\lb}$ over 50
  preference rays.  Full-sequence convergence (Theorem~\ref{thm:main})
  applies on rays with $\sigma = 0$; approximate convergence
  (Theorem~\ref{thm:approx}) applies otherwise.}
\label{tab:sigma}
\begin{tabular}{lcccc}
\hline
Problem & $\sigma_{\min}$ & $\sigma_{\max}$ & $\sigma_{\text{mean}}$
        & Rays with $\sigma = 0$ \\
\hline
CVX1 & 0.0000 & 0.0806 & 0.0090 & 39/50 \\
CVX2 & 0.0000 & 0.0134 & 0.0005 & 44/50 \\
\hline
\end{tabular}
\end{table}

\subsection{Effect of the number of preference rays}
\label{app:ray-ablation}

Table~\ref{tab:mop-rays} reports MED ($\downarrow$) and HV ($\uparrow$)
on CVX2 for ABP-HyperMLP and ABP-HyperTrans as the number of preference
rays varies from 10 to 100.

\begin{table}[!htbp]
\centering
\caption{Effect of the number of preference rays on MED ($\downarrow$)
  and HV ($\uparrow$) for CVX2, reported for a single representative run
  (not multi-seed averages; cf.\ Table~\ref{tab:mop-results} for
  mean\,$\pm$\,std over 10 seeds at 50 rays).
  Best result per metric and architecture is \textbf{bold}.}
\label{tab:mop-rays}
\begin{tabular*}{\tblwidth}{@{\extracolsep{\fill}}l cc cc cc cc@{}}
\toprule
& \multicolumn{2}{c}{10 rays}
& \multicolumn{2}{c}{20 rays}
& \multicolumn{2}{c}{50 rays}
& \multicolumn{2}{c}{100 rays} \\
\cmidrule(lr){2-3}\cmidrule(lr){4-5}\cmidrule(lr){6-7}\cmidrule(lr){8-9}
Method & MED & HV & MED & HV & MED & HV & MED & HV \\
\midrule
ABP-HyperMLP   & 0.00830 & 1.6298 & 0.00535 & 1.6690 & 0.00443 & 1.6886 & \textbf{0.00384} & \textbf{1.6910} \\
ABP-HyperTrans & 0.00461 & 1.6392 & 0.00322 & 1.6788 & 0.00242 & 1.6937 & \textbf{0.00212} & \textbf{1.6982} \\
\bottomrule
\end{tabular*}
\end{table}

\begin{figure}[!htbp]
  \centering
  \includegraphics[width=0.82\textwidth]{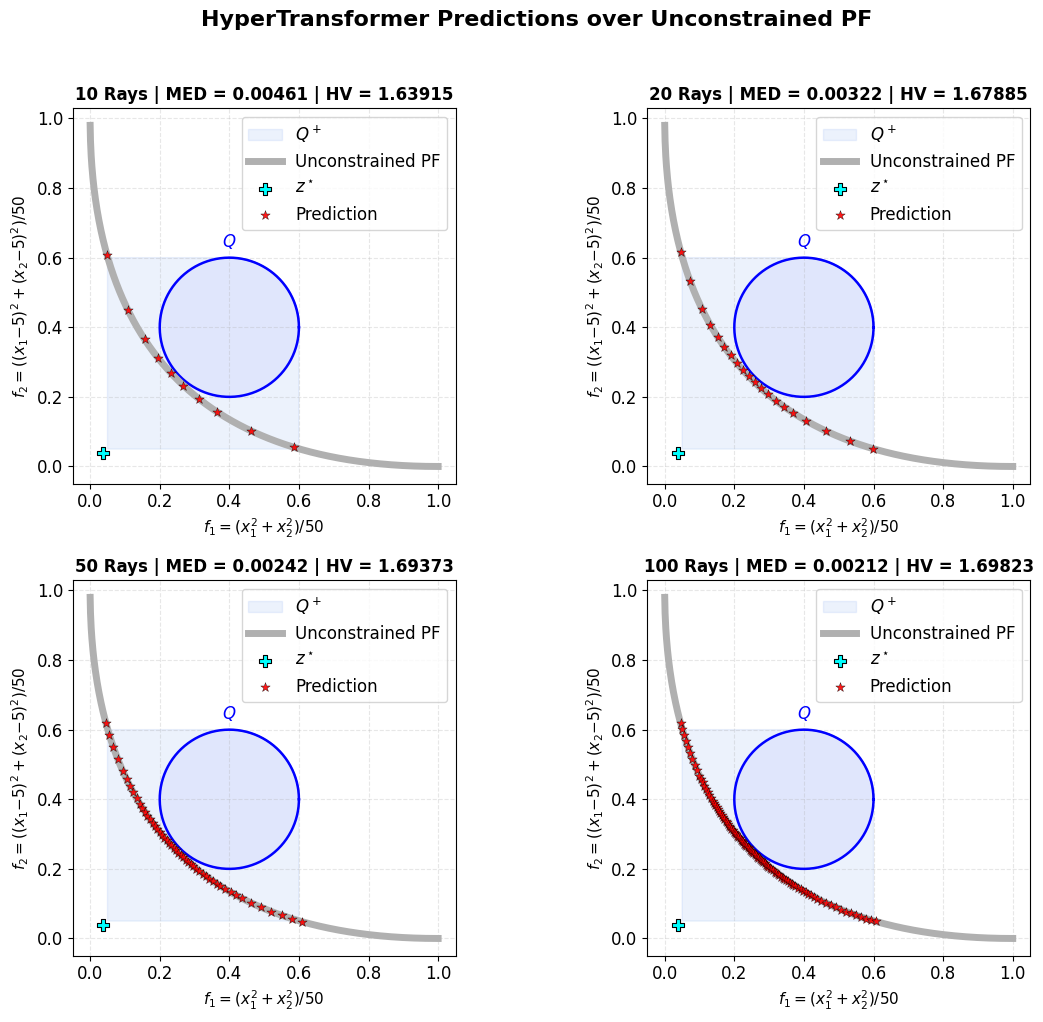}
  \caption{ABP-HyperTrans predictions on CVX2 at 10, 20, 50, and 100
  preference rays.
  Grey curve: unconstrained Pareto front; blue circle: boundary of~$Q$;
  shaded region:~$Q^+$; red stars: predicted solutions;
  cyan cross:~$z^*$.}
  \label{fig:mop-cvx2-rays}
\end{figure}

Both MED and HV improve monotonically with the number of rays:
ABP-HyperTrans's MED drops by 47\% from 10 to 50 rays
(0.00461 $\to$ 0.00242) and by a further 12\% from 50 to 100 rays
(0.00242 $\to$ 0.00212).
Visually, Fig.~\ref{fig:mop-cvx2-rays} confirms that denser ray
sampling yields a more complete coverage of the constrained Pareto
front inside~$Q^+$.
For consistency with the baseline of~\citep{tuan2024}, all MOP
experiments in this paper use 50 rays.

\subsection{Hyperparameter sensitivity}\label{app:sensitivity}

The two-phase training strategy introduces four key hyperparameters
that govern the balance between feasibility enforcement and
Chebyshev optimization.
We conduct a sensitivity study on Multi-MNIST
(Hyper-MLP, Box constraint $\bm{b}=(0.30,\,0.40)$, 300 epochs),
sweeping each hyperparameter across five values while keeping
the remaining ones at their defaults.
Each configuration is evaluated over five independent seeds;
Figure~\ref{fig:sensitivity} reports mean Hypervolume (HV)
and feasibility rate.

\begin{figure}[!htbp]
  \centering
  \includegraphics[width=\textwidth]{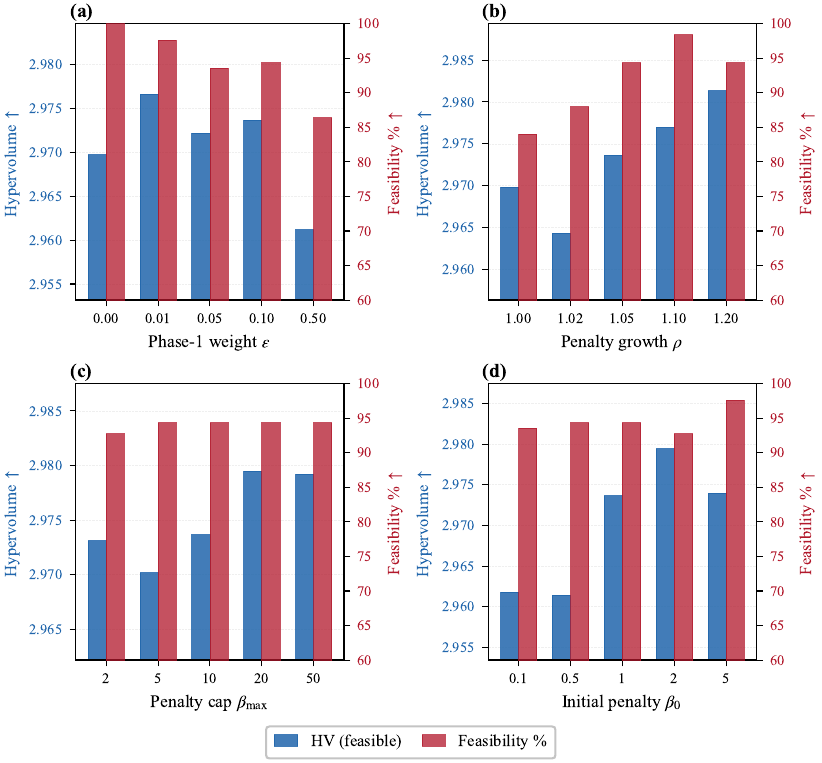}
  \caption{Hyperparameter sensitivity on Multi-MNIST
  (Hyper-MLP, Box constraint) across four key parameters.
  Bars represent the mean over 5 independent seeds for (a)~Phase-1 weight~$\varepsilon$, (b)~penalty growth~$\rho$, (c)~penalty cap~$\beta_{\max}$, and (d)~initial penalty~$\beta_{0}$.}
  \label{fig:sensitivity}
\end{figure}

\textit{Phase-1 objective weight~$\varepsilon$.}
During Phase~1 the loss is
$\mathcal{L}_{Q}+\varepsilon\,\mathcal{L}_{\mathrm{obj}}$.
Setting $\varepsilon=0$ yields 100\% feasibility but the highest
HV variance, because the hypernetwork does not learn the
preference-to-solution mapping during warm-up and must acquire it
from scratch in Phase~2.
Conversely, $\varepsilon=0.50$ overweights the Chebyshev term,
extending Phase~1 to 70 of 80 budgeted epochs and dropping
feasibility to 86\%.
The range $\varepsilon\in[0.01,0.10]$ balances diversity
preservation with rapid feasibility convergence;
we adopt $\varepsilon=0.10$ as the default.

\textit{Penalty growth factor~$\rho$.}
In Phase~2 the constraint coefficient~$\beta$ is multiplied by~$\rho$
(resp.\ $\rho^{-1}\approx0.98$) whenever feasibility drops below
(resp.\ exceeds) the target.
With $\rho=1.00$ (constant~$\beta$) feasibility falls to 84\%
because the restoring force cannot compensate for the
optimality-driven drift.
Increasing $\rho$ to 1.05--1.10 lifts feasibility above 94\%
while improving HV (from 2.970 to 2.977).
At $\rho=1.20$ HV reaches 2.981 but feasibility variance
increases, indicating occasional over-penalization that
flattens the Pareto front approximation in some seeds.
The default $\rho=1.05$ offers a robust trade-off.

\textit{Penalty cap~$\beta_{\max}$.}
The multiplicative cap limits how large~$\beta$ can grow.
For $\beta_{\max}=2$ the ceiling is reached too early and
feasibility drops to 93\%;
for $\beta_{\max}\ge10$ the cap is rarely activated and
performance stabilizes (HV~$\approx2.974$--$2.979$,
feasibility~$\approx94$\%).
We set $\beta_{\max}=10$ as a conservative default.

\textit{Initial penalty~$\beta_{0}$.}
When Phase~2 begins, the penalty coefficient is initialized
to~$\beta_{0}$.
Starting with a weak penalty ($\beta_{0}\le0.5$) causes a
transient feasibility dip as rays drift outside~$Q$ before the
adaptive schedule compensates, reducing HV to~$\approx2.962$.
Values $\beta_{0}\ge1.0$ prevent this drift;
$\beta_{0}=5.0$ yields the best feasibility (97.6\%) at the cost
of slightly reduced solution diversity.
The default $\beta_{0}=1.0$ balances these effects.

Overall, the method is robust across the tested ranges:
HV remains within $[2.961,\,2.985]$ and feasibility stays above
84\% in all 20 configurations.
The most sensitive parameters are the Phase-1 weight
$\varepsilon$ and the growth factor~$\rho$, both of which
directly govern the balance between optimality and constraint
enforcement in their respective phases.

\FloatBarrier

\section*{CRediT authorship contribution statement}
\textbf{Nguyen Viet Hoang:} Methodology, Software, Validation,
Writing -- Original Draft.
\textbf{Dung D. Le:} Supervision, Writing -- Review \& Editing.
\textbf{Tran Ngoc Thang:} Methodology, Conceptualization, Formal Analysis,
Supervision, Writing -- Review \& Editing.

\section*{Declaration of competing interest}
The authors declare that they have no known competing financial
interests or personal relationships that could have appeared to
influence the work reported in this paper.

\section*{Data availability}
The Multi-MNIST, Multi-Fashion, and Fashion+MNIST benchmark datasets
used in this study are publicly available at
\url{https://github.com/longhp1618/MultiSample-Hypernetworks/tree/main/Multi_MNIST/data}.
The MOP benchmark problems (CVX1--3, ZDT1--2) are fully specified by
the analytical definitions in Section~\ref{subsec:mop-benchmarks} and
require no additional data files.
Source code for reproducing the experiments will be made publicly
available in a dedicated repository upon acceptance of the manuscript.

\bibliographystyle{cas-model2-names}


\section*{Declaration of Generative AI and AI-assisted technologies in the writing process}
During the preparation of this work the authors used Claude
(Anthropic) for language editing and improving readability.
After using this tool, the authors reviewed and edited the content
as needed and take full responsibility for the content of the
publication.

\end{document}